\documentclass{article}

 \usepackage[preprint]{neurips_2026}

\def\bbE{\mathbb{E}}

\def\bbP{\mathbb{P}}

\def\bbR{\mathbb{R}}

\def\cD{\mathcal{D}}

\def\cF{\mathcal{F}}

\def\cH{\mathcal{H}}

\def\cL{\mathcal{L}}

\def\cO{\mathcal{O}}

\def\cR{\mathcal{R}}

\def\cT{\mathcal{T}}

\def\cX{\mathcal{X}}
\def\cY{\mathcal{Y}}

\def\Ex{{\bbE}}

\def\exp{{\sf exp}}

\def\ones{\mathbf{1}}

\usepackage[utf8]{inputenc} 
\usepackage[T1]{fontenc}    
\usepackage{hyperref}       
\usepackage{url}            
\usepackage{booktabs}       
\usepackage{amsfonts}       
\usepackage{nicefrac}       
\usepackage{microtype}      
\usepackage{xcolor}         
\usepackage{algorithm}
\usepackage{amsmath}
\usepackage{amssymb}
\usepackage{mathtools}
\usepackage{amsthm}
\usepackage{algpseudocode}
\usepackage{graphicx}
\usepackage{multirow}
\usepackage{subcaption}
\graphicspath{ {Figures/} }
\newtheorem{theorem}{Theorem}[section]
\newtheorem{proposition}[theorem]{Proposition}

\newtheorem{corollary}[theorem]{Corollary}
\newtheorem{definition}[theorem]{Definition}
\newtheorem{assumption}[theorem]{Assumption}
\newtheorem{remark}[theorem]{Remark}

\title{$\alpha$-Fair Insurance Pricing: A Fairness Continuum}

\author{
Tianhe Zhang\thanks{Corresponding author. Email: tzhang349@wisc.edu}
\thanks{Department of Risk and Insurance, Wisconsin School of Business, University of Wisconsin--Madison}
\And
Xiguang Liu
\thanks{Department of Information Systems and Operations Management, Warrington College of Business, University of Florida}
\And
Peng Shi
\footnotemark[2]
}

\begin{document}

\maketitle

\begin{abstract}
  Fairness in insurance pricing remains a long-standing and deeply debated puzzle. On one hand, insurers, driven by profitability considerations, set premiums that differentiate across individual risks to achieve actuarial fairness. On the other hand, insurance serves a critical societal function by pooling risks across a population, motivating cross-subsidization among groups to promote solidarity fairness. The tension between these two competing notions of fairness makes insurance pricing inherently complex, particularly in modern settings where granular data allow for increasingly fine risk differentiation and regulators face growing pressure to protect vulnerable groups. To address this challenge, we propose an 
  $\alpha$-\textbf{F}air \textbf{I}ndividual \textbf{S}olvent \textbf{P}remium ($\alpha$-FISP) framework for insurance pricing that explicitly captures the trade-off between actuarial and solidarity fairness while guaranteeing solvency, a fundamental requirement in insurance operations. We formulate the pricing problem as a constrained optimization task, where actuarially fair premiums are adjusted subject to budget constraints on cross-subsidization within each risk class. This formulation naturally yields a family of solutions parameterized by $\alpha$, tracing a continuum between purely actuarial and purely solidarity-based pricing and enabling decision-makers to select an operating point along this fairness spectrum. We derive theoretical guarantees for the proposed framework. Numerical experiments show that $\alpha$-FISP is computationally tractable and aligns well with the U.S. regulatory regimes featuring heterogeneous state-level fairness requirements.
\end{abstract}

\section{Introduction}
\label{section: Introduction}

Fairness has become a central theme in the era of data-driven decision-making, particularly as machine learning (ML) algorithms are increasingly deployed in high-stakes domains such as credit scoring, criminal justice, and healthcare. In this paper, we focus on fairness in insurance pricing, a setting where algorithmic decisions have widespread financial and societal impact. The insurance industry is inherently data-driven, relying heavily on statistical and predictive models to derive insights from data and set premiums \citep{frees2015analytics}. With the rapid growth of InsurTech and the availability of granular data, algorithmic pricing plays a central role in modern insurance operations \citep{sosa2022understanding}. Given the insurance market's size and economic criticality, ensuring fair pricing is a practical and societal imperative.

Insurance operates in a profit-driven environment while simultaneously providing financial protection as a societal good. This dual role makes fairness in insurance pricing a long-standing and deeply debated puzzle. On one hand, insurers—driven by profitability, competition, and solvency considerations—set premiums that differentiate across individual risks to achieve actuarial fairness. A premium is actuarially fair if it equals or closely reflects the insured’s expected loss \citep{landes2015fair}. Traditionally, insurers have implemented actuarial fairness through risk classification, grouping individuals with similar attributes into homogeneous risk classes. Refining these classes improves alignment between premiums and underlying risk, which helps mitigate adverse selection and encourages loss-control behavior by making policyholders financially accountable for their risk exposure \citep{michael1976imperfect, shavell1979moral}. On the other hand, insurance fulfills a critical social function by pooling risks across the population, which motivates cross-subsidization between groups to promote solidarity fairness. A solidarity-oriented view of fairness, often grounded in principles of social justice \citep{lehtonen2011forms}, emphasizes equality in sharing the burden of losses. Under this view, fairness may require some policyholders to pay more than their pure expected loss—or others to pay less—to maintain access and affordability for those at higher risk through no fault of their own.

The tension between these two competing notions of fairness—solidarity and actuarial—makes insurance pricing inherently complex. This debate is often conceptualized along a continuum, with a pure “moral” or solidarity-driven perspective at one end and a pure “technical” or actuarially driven perspective at the other \citep{thiery2006fairness}. While these notions are incompatible when taken to their extremes, empirical studies suggest that public preferences tend to occupy a nuanced middle ground, accepting some degree of cross-subsidization as long as pricing remains broadly risk-reflective \citep{dixon2024public}. This balancing act has become even more challenging in modern settings. Advances in data availability and predictive modeling have enabled insurers to differentiate risks with unprecedented granularity. While such refinement improves actuarial precision, it can also magnify premium differences between groups, raising concerns about affordability and social exclusion. In addition, because many rating variables are correlated with sensitive attributes such as gender, race, or socioeconomic status, data-driven pricing can unintentionally reproduce or reinforce existing inequities \citep{FreesHuang2023}. Consequently, regulators face growing pressure to protect vulnerable populations and restrict the extent of permissible risk differentiation. Recent regulatory interventions highlight this trend, including the EU Gender Directive (2004/113/EC), which prohibits the use of gender in life insurance pricing \citep{EU2004Directive113}, and Colorado’s SB 21-169 statute, which bars unfair discrimination across multiple protected characteristics \citep{ColoradoSB21-169}. Taken together, the tension between market-driven risk differentiation and regulation-mandated solidarity highlights the need for frameworks that balance these objectives, creating pricing systems that are both economically viable and socially acceptable.

Motivated by these observations, we introduce a novel insurance pricing framework, termed $\alpha$-fairness, to explicitly capture the trade-off between actuarial and solidarity fairness. This framework yields premiums parameterized by $\alpha \in [0,1]$, defining a continuum from purely actuarial to purely solidarity-based pricing and enabling decision-makers to select an operating point aligned with their strategic or regulatory priorities. First, we formally define the notion of $\alpha$-fairness. A premium structure is said to be $\alpha$-fair if, within any group of policyholders sharing the same non-sensitive attributes (a risk class), the ratio of the lowest to the highest premium is at least $\alpha$. At the extremes, $\alpha = 0$ corresponds to fully differentiated premiums based on sensitive attributes (pure actuarial fairness), while $\alpha = 1$ enforces uniform premiums within the group (pure solidarity fairness). Second, we formulate the computation of $\alpha$-fair premiums as a constrained optimization problem that adjusts actuarially fair premiums subject to budget constraints controlling the degree of cross-subsidization within each risk class. The resulting framework is computationally tractable, statistically well-grounded, and robust. Third, the proposed $\alpha$-fair premiums satisfy individual-level solvency requirements, a critical property for insurance pricing that is often overlooked in existing fair-pricing studies. More broadly, the framework is flexible and accommodates practical considerations such as portfolio-level profit targets and regulatory caps on rate changes.


We emphasize that, while our definition of $\alpha$-fairness is framed in terms of sensitive and non-sensitive attributes—reflecting one of its most pressing applications—the framework is broadly applicable. In this broader view, non-sensitive attributes correspond to rating variables used to define homogeneous risk classes, whereas sensitive attributes refer to variables for which regulators restrict or prohibit risk-based differentiation. These may include direct measures of socioeconomic status (e.g., gender, race) or proxy variables correlated with such characteristics.

\section{Related Work}
\label{section: Related Work}

\subsection{Fairness in Machine Learning}
\label{subsection: Fairness in Machine Learning}




Fairness in ML has become a rapidly growing area of research, motivated by concerns that algorithmic decisions may replicate or amplify existing social inequities \citep{mehrabi2021survey, barocas2023fairness}. A large body of work formalizes different definitions of fairness, typically falling into two broad categories: individual fairness and group fairness. Individual fairness requires that similar individuals receive similar outcomes \citep{dwork2012fairness}, while group fairness ensures parity of outcomes across predefined subpopulations (e.g., demographic groups) according to metrics such as equalized odds \citep{hardt2016equality}, demographic parity \citep{feldman2015certifying}, or calibration \citep{kleinberg2017inherent}. These notions often conflict, and a central challenge in algorithmic fairness lies in balancing competing fairness criteria under statistical, operational, and contextual constraints \citep{chouldechova2017fair, mitchell2021algorithmic, corbett2023measure}.

Beyond definitions, the machine learning literature has proposed numerous algorithmic interventions to promote fairness. Pre-processing approaches seek to mitigate bias in the data, for example through reweighting or learning fair representations \citep{kamiran2012data, zemel2013learning, calmon2017optimized}. In-processing approaches embed fairness constraints or regularization directly into model training, modifying the optimization objective to control disparate impact or mistreatment \citep{zafar2017fairness, donini2018empirical, agarwal2018reductions}. Post-processing methods adjust model outputs or decision thresholds to satisfy fairness criteria while preserving predictive accuracy \citep{hardt2016equality, pleiss2017fairness}. Recent work has increasingly focused on fairness–efficiency trade-offs through explicit multi-objective or constrained optimization formulations, enabling decision-makers to tune the balance between accuracy and equity \citep{menon2018cost, williamson2019fairness, martinez2020minimax}.

Our work is conceptually aligned with this emerging line of research. We adopt an optimization-based approach that explicitly trades off two competing fairness notions—actuarial and solidarity fairness—in the insurance context. However, unlike most ML settings, where fairness is defined in terms of statistical parity or predictive error, fairness in insurance pricing is inherently economic: it concerns how expected losses and premiums are distributed across individuals. Consequently, our framework extends the fairness discourse from prediction quality to the allocation of financial burden, embedding ethical and regulatory considerations directly into the pricing mechanism.

\subsection{Fairness in Insurance Pricing}
\label{subsection: Fairness in Insurance Pricing}

The discussion of fairness in insurance pricing has a long history in actuarial science, well before the emergence of algorithmic fairness in machine learning. Classical insurance theory distinguishes between actuarial fairness and solidarity fairness (or social fairness) \citep{thiery2006fairness, lehtonen2011forms}. Much of the early actuarial literature treated this as a normative tension rather than a quantitative trade-off, with fairness operationalized through the design of rating structures that balance heterogeneity and pooling \citep{charpentier2024insurance}.

Recent research in actuarial science has revived this debate in light of advances in data availability and machine learning. A growing body of work has sought to formalize fairness criteria and clarify how regulatory objectives translate into modeling constraints \citep{FreesHuang2023, xin2024antidiscrimination, fahrenwaldt2024fairness, cote2025fair}. Collectively, these studies emphasize that fairness in insurance cannot be reduced to a single metric, but must account for causal structure, legal standards, and market functioning. 

A complementary line of research focuses on computing discrimination-free premiums. For example, \citet{lindholm2022dfip} define a framework for removing discriminatory dependence between sensitive attributes and premiums, and \citet{lindholm2024multi} extend this approach using multi-task networks. \citet{araiza2024discrimination} propose a causal formulation that achieves fairness through counterfactual adjustments, ensuring pricing decisions are invariant to changes in protected attributes. More recently, \citet{zhang2025dfipLDP} incorporate local differential privacy to handle scenarios where sensitive attributes are unobserved or privatized. These approaches primarily operationalize fairness at the prediction level, using constraints or transformations applied to the outputs of pricing models.

In contrast, our work adopts a broader and more structural view of fairness in insurance pricing. Rather than constraining the predictive model directly, we model fairness as an explicit allocation mechanism that governs how actuarially fair premiums are adjusted under cross-subsidization constraints. This perspective aligns more closely with the long-standing actuarial debate between solidarity and actuarial fairness, while introducing a formal and tunable parameter, $\alpha$, that captures the trade-off between the two. Our $\alpha$-fairness framework thus complements existing literature by connecting traditional actuarial fairness principles with mordern fairness optimization, providing both a normative interpretation and a practical computational procedure for balancing social and actuarial objectives.


\subsection{Our Contribution}
\label{subsection: Our Contribution}

This paper advances the literature on fairness in both machine learning and insurance pricing in several important ways. First, beyond existing fairness notions and frameworks, we introduce a new paradigm, referred to as $\alpha$-fairness framework, that formalizes a continuum between actuarial and solidarity fairness. Unlike conventional fairness metrics that focus on prediction errors or demographic parity, the $\alpha$-fairness directly regulates premium differentials within risk classes, reflecting the inherently economic nature of fairness in insurance. Second, we formulate the computation of $\alpha$-fair premiums as a constrained optimization problem that adjusts actuarially fair premiums subject to group-specific cross-subsidization constraints. This formulation further guarantees individual-level solvency, a critical property that existing fairness-based pricing approaches often struggle to achieve. The resulting $\alpha$-Fair Individual Solvent Premium ($\alpha$-FISP) is supported by rigorous theoretical guarantees, including existence, boundedness, sample complexity preservation, and order preservation, providing a solid foundation for practical implementation. Third, the proposed framework yields a computationally tractable premium ($\alpha$-FISP) and is regulatorily adaptable: varying state-level fairness requirements can be achieved by choosing the interpretable fairness budget parameter $\alpha$, enabling a unified pricing pipeline that aligns with the state-governed structure of U.S. insurance regulation. Simulation studies and empirical case analyses demonstrate that $\alpha$-fair pricing aligns naturally with regulatory objectives. Finally, our work conceptually bridges machine learning and actuarial perspectives on fairness. By framing fairness as an allocation problem rather than solely a predictive constraint, the $\alpha$-fairness framework integrates ethical and regulatory considerations directly into insurance pricing, yielding a unified, interpretable, and practically implementable framework.

\section{Problem Formulation}
\label{section: Problem Formulation}

\subsection{Preliminaries}
\label{subsection: Preliminaries}

Let $(X,D,Y)$ be a random triplet with unknown joint distribution $\mathbb{P}_{X,D,Y}$. We observe $n$ samples:
\[
(x_i,d_i,y_i) \overset{\text{i.i.d.}}{\sim} \mathbb{P}_{X,D,Y}, \quad i=1,\ldots,n.
\]
In the insurance pricing context, $x_i \in \mathcal{X} \subseteq \mathbb{R}^p$
denotes a $p$-dimensional vector of non-sensitive rating variables used for
risk classification, $d_i \in \mathcal{D}$ denotes sensitive attributes, and
$y_i \in \mathcal{Y}$ represents the realized loss (or
risk outcome) of a policyholder. We assume that $Y$ is a non-negative random variable with finite second moment, i.e.,
$\mathbb{E}[Y^2] < \infty$, and that $D$ takes values in a finite discrete set, i.e., $|\mathcal{D}| < \infty$. 


\begin{definition}[Actuarial Fair]
\label{def:actuarial-fair-premium}
    The actuarial fair premium for $Y$ w.r.t. $(X,D)$ is defined as: 
    \[
    \mu(X, D):= \mathbb{E}[Y | X,D].
    \]
\end{definition}
This premium fully reflects differences in expected losses attributable to both non-sensitive and sensitive attributes. For notational convenience, we denote $\mu(X,D)$ by $\mu_D(X)$ throughout the paper.

\begin{definition}[Solidarity Fair]
\label{def:solidarity-premium}
    The solidarity fair premium for $Y$ w.r.t. $X$ is defined as:
    \[
    \mu(X):= \mathbb{E}[Y | X].
    \]
\end{definition}
Under this pricing rule, individuals sharing the same non-sensitive attributes are charged a common premium, thereby pooling risks across sensitive attributes.

\begin{definition}[Individual Fairness]
\label{def:individual-fairness}
Let $f \in \mathcal{F}$ be a pricing function $f:\mathcal{Z} \to \mathbb{R}$ defined on an input space $\mathcal{Z}$. Equip $\mathcal{Z}$ with a metric $d_{\mathcal{Z}}$ and the output space with metric $d_{\mathbb{R}}$. For $\epsilon \ge 0$ and $\delta \ge 0$ given, $f$ is \emph{individually fair} if for any $z,z' \in \mathcal{Z}$ satisfying $d_{\mathcal{Z}}(z,z') \le \epsilon$, we have $d_{\mathbb{R}}(f(z),f(z')) \le \delta$.

\end{definition}
This principle requires that individuals who are similar according to the chosen input-space metric receive similar outcomes. The thresholds $(\epsilon, \delta)$ allow relaxation of strict fairness requirements when exact parity is not feasible. 



\subsection{$\alpha$-Fairness and Individual Solvency}
For fair insurance pricing, we introduce two purpose-built distance functions.

\begin{definition}[Metric on Feature Space]
\label{def:metric-on-feature-space}
Let $dist_{\mathcal{X}\times\mathcal{D}} : (\mathcal{X}\times\mathcal{D}) \times (\mathcal{X}\times\mathcal{D}) \to [0,\infty]$
be defined for $(x,d),(x',d') \in \mathcal{X}\times\mathcal{D}$ as
\[
dist_{\mathcal{X}\times\mathcal{D}}\bigl((x,d),(x',d')\bigr) =
\begin{cases}
0, & \text{if } x=x', \\
+\infty, & \text{if } x \neq x'.
\end{cases}
\]
\end{definition}

This extended pseudometric is characterized exclusively by non-sensitive attributes. It treats policyholders with identical non-sensitive attributes ($x = x'$) as maximally similar regardless of sensitive attributes, while deeming those with differing attributes ($x \neq x'$) incomparable.


\begin{definition}[Metric on Outcome Space]
\label{def:metric-on-outcome-space}
For strictly positive premiums $y, y' > 0$ assigned by a pricing function $f$, define
\[
dist_{\mathcal{Y}}(y,y') := 1 - \min\left(\frac{y}{y'},\frac{y'}{y}\right).
\]
\end{definition}
This scale-invariant metric satisfies all metric axioms for $y, y' > 0$. It quantifies proportional deviation between premiums, where $dist_{\cY}(y,y') = 0$ if and only if $y = y'$ (exact/strict fairness constraint), and approaches $1$ as the ratio between premiums becomes unbounded (no fairness constraint).





With the distance functions defined on the feature space and outcome space, we can now formalize a proportional notion of individual fairness for insurance premiums.
\begin{definition}[$\alpha$-Fairness]
\label{def:alpha-fairness}
Let $f \in \mathcal{F}$ be a pricing function $f: \cX \times \cD \to \mathbb{R}_{+}$, and let 
$dist_{\mathcal{X} \times \mathcal{D}}$ and $dist_{\mathcal{Y}}$ be the feature-space and outcome-space distance functions defined in Definitions~\ref{def:metric-on-feature-space} and~\ref{def:metric-on-outcome-space}, respectively. Given thresholds $\epsilon \ge 0$ and $\alpha \in [0,1]$, $f$ is said to be \emph{$\alpha$-fair} if, for any pair of individuals $(x,d), (x',d') \in \mathcal{X} \times \mathcal{D}$ satisfying $dist_{\mathcal{X} \times \mathcal{D}}\bigl((x,d),(x',d')\bigr) \le \epsilon$, we have $dist_{\mathcal{Y}}\bigl(f(x,d), f(x',d')\bigr) \le 1 - \alpha$.
\end{definition}


When $x = x'$, $\alpha$-fairness requires $
\min\left( \frac{f(x,d)}{f(x,d')}, \frac{f(x,d')}{f(x,d)} \right) \ge \alpha,
$
ensuring that individuals with identical non-sensitive attributes receive proportionally similar premiums, with $\alpha$ controlling the allowable deviation. When $x \neq x'$, no fairness constraint is imposed, allowing full use of non-sensitive attributes for risk classification. The scale-invariant design of $dist_{\mathcal{Y}}$ ensures this proportional constraint holds across all premium levels.


For notational clarity, let $\{d_1,\ldots,d_{|\mathcal{D}|}\}$ denote the support of $D$, and let $p_{k|x} := \mathbb{P}(D = d_k | X = x)$.

\begin{definition}[Individual Solvency]
\label{def:solvency}
Let $f \in \mathcal{F}$ be a pricing function $f:\mathcal{X} \times \mathcal{D} \to \mathbb{R}_{+}$. The function $f$ is \emph{individually solvent} if, for every fixed $x \in \mathcal{X}$,
\[
\sum_{k=1}^{|\mathcal{D}|} \left( f(x,d_k) - \mu_k(x) \right) p_{k|x} \ge 0,
\]
where $\mu_k(x) := \mathbb{E}[Y | X=x, D=d_k]$ denotes the actuarially fair premium.
\end{definition}

Individual solvency enforces self-sufficiency by requiring premiums for identical non-sensitive attributes to cover their group's expected loss. This precludes cross-subsidization across heterogeneous groups ($x \neq x'$), confining financial liability to actuarially comparable pools.

\subsection{Constrained Optimization}

Our objective is to compute insurance premiums that satisfy both $\alpha$-fairness and individual solvency. We formulate this task as a constrained risk minimization problem. Specifically, for a pricing function $f:\mathcal{X}\times\mathcal{D}\to\mathbb{R}_{+}$, we seek to minimize mean squared error for $Y$:


\begin{equation}
\label{eq:objective-function}
\min_{f \in \mathcal{F}}
\; \mathbb{E}_{(X,D,Y)\sim \mathbb{P}_{X,D,Y}}
\bigl[(f(X,D)-Y)^2\bigr],
\end{equation}


subject to the constraints:
\begin{align}
\label{eq:alpha-fairness}
    1 - dist_{\cY} \left( f(x,d_k),f(x,d_{k'}) \right) &\ge \alpha, \ \ \ \forall x \in \cX, \alpha \in (0,1],  ~k, k' \in [|\cD|]. \\
\label{eq:individual-solvency}
    \sum_{k=1}^{|\cD|}  \left( f(x,d_k) - \mu_k(x) \right) \cdot p_{k|x}  &\ge 0, \ \ \ \forall x \in \cX.
\end{align}
The solution to this optimization problem is denoted the \emph{$\alpha$-Fair Individual Solvent Premium} ($\alpha$-FISP). To simplify, we express the nonlinear constraint~\eqref{eq:alpha-fairness} in the following equivalent linear inequalities:
\begin{align}
\label{eq:alpha-fairness-c1}
f(x,d_k) - \alpha\, f(x,d_{k'}) &\ge 0, \\
\label{eq:alpha-fairness-c2}
f(x,d_{k'}) - \alpha\, f(x,d_k) &\ge 0,
\end{align}
for all $x\in\mathcal{X}$ and all pairs $k<k'$.



Taken together, individual solvency and $\alpha$-fairness yield a two-tier equity framework.
Individual solvency anchors premiums to legitimate risk differentiation based on non-sensitive
attributes, ensuring inter-group equity. Conditional on risk class, $\alpha$-fairness governs
intra-group treatment across sensitive attributes, enforcing proportional fairness. This
hierarchical separation resolves the tension between actuarial fairness and societal fairness
by decoupling risk classification from fairness enforcement.




\section{Properties of the Fair Pricing Optimization}
\label{section: Properties of the Fair Pricing Optimization}


In this section, we study key theoretical properties of the constrained optimization problem defined by the objective~\eqref{eq:objective-function} and the constraints \eqref{eq:individual-solvency}, \eqref{eq:alpha-fairness-c1}, and \eqref{eq:alpha-fairness-c2}.

We begin by establishing the existence and uniqueness of the optimal solution. This result is fundamental, as it guarantees that the proposed fairness and solvency requirements lead to a well-defined pricing rule rather than an ill-posed or degenerate optimization problem. We then argue that the problem is well-posed from an insurance operations perspective. In particular, we show that imposing $\alpha$-fairness and individual solvency does not induce an excessive loss of pricing accuracy. Under mild regularity conditions tailored to the insurance context, $\alpha$-FISP incurs only a controlled, additive increase in generalization error relative to the unconstrained risk minimizer. This property is essential in practice, as constraints that dramatically degrade model performance would undermine both competitiveness and portfolio stability. We also derive the out-of-sample constraint violation rate for completeness. Finally, we establish that for any fixed $\alpha \in (0,1]$, the optimal pricing function is order preserving in risk. This monotonicity property provides an additional fairness guarantee that aligns with policyholder expectations and actuarial principles, further reinforcing the economic and regulatory plausibility of our framework.



\begin{proposition}[Existence and Uniqueness]
\label{proposition:existence-of-solution}
Fix $\alpha \in (0,1]$. Let $\mathcal{F}$ be a hypothesis class of measurable functions restricted to feasible nonnegative pricing functions $f:\mathcal{X}\times\mathcal{D}\to\mathbb{R}_{+}$ satisfying
$\mathbb{E}[f(X,D)^2]<\infty$. If there exists a constant $\mu_{\min}>0$ s.t.
$\mu_k(x)\ge \mu_{\min}$ and $p_{k|x}>0$ for $\mathbb{P}_X$-a.e. $x\in\mathcal{X}$
and all $k \in [|\cD|]$. Let $\mathcal{F}_{\alpha,IS}\subset\mathcal{F}$ be the feasible set of functions
satisfying the individual solvency constraint~\eqref{eq:individual-solvency} and the
$\alpha$-fairness constraints~\eqref{eq:alpha-fairness-c1}--\eqref{eq:alpha-fairness-c2}.
Then the optimization problem~\eqref{eq:objective-function} admits a unique minimizer
$f^*\in\mathcal{F}_{\alpha,IS}$.
\end{proposition}
The assumption $\mu_{\min}>0$ is natural in insurance applications, as premiums and expected losses are strictly positive for any active risk class. Likewise, the boundedness and square-integrability of $\mathcal{F}$ are consistent with underwriting guidelines and regulatory constraints that cap admissible premiums. Together, these conditions ensure that the feasible set is nonempty, closed, and convex, supporting both the existence and uniqueness of the optimal solution.


To establish that the constrained optimization problem is well-posed, we introduce two assumptions. The first controls the complexity of the constrained hypothesis class through finite covering numbers (Assumption~\ref{assumption:covering-number}), while the second characterizes the statistical error induced by plug-in estimators for conditional means and group probabilities (Assumption~\ref{assumption:plug-in-error}). Together, these assumptions allow us to quantify the excess risk incurred by enforcing $\alpha$-fairness and individual solvency.

\begin{assumption}[Covering Numbers]
\label{assumption:covering-number}
Define the class of $\alpha$-fair pricing functions as
$\mathcal{F}_\alpha := \bigl\{ f \in \mathcal{F} \;\big|\; f \text{ satisfies }
\eqref{eq:alpha-fairness-c1} \text{ and } \eqref{eq:alpha-fairness-c2} \bigr\}$, and let $\mathcal{T} \subset \{ t : \mathcal{X} \to [0,1] \}$ be a class of nonnegative measurable test functions. Let $\|\cdot\|_{L^2(\mathbb{P}_{X,D})}$ and $\|\cdot\|_{L^2(\mathbb{P}_X)}$ denote the $L^2$ norms
w.r.t. $\mathbb{P}_{X,D}$ and $\mathbb{P}_X$, respectively. Assume there exist positive
constants $V_F, A_F, V_T,$ and $A_T$ such that for all $\epsilon \in (0,1]$,
\begin{align*}
\log N\bigl(\epsilon, \mathcal{F}_\alpha, \|\cdot\|_{L^2(\mathbb{P}_{X,D})}\bigr)
&\le V_F \log\!\left( A_F / \epsilon \right), \\
\log N\bigl(\epsilon, \mathcal{T}, \|\cdot\|_{L^2(\mathbb{P}_X)}\bigr)
&\le V_T \log\!\left( A_T / \epsilon\right).
\end{align*}
Here, $N(\epsilon,\mathcal{G},\|\cdot\|)$ is the $\epsilon$-covering number of a function class $\mathcal{G}$ under the metric induced by $\|\cdot\|$.
\end{assumption}

\begin{remark}
The assumption is adapted from Theorem 2.1 in \cite{focm2024}, which bounds the covering number for fully-connected ReLU networks with uniformly bounded weights. As a direct consequence, the models we used in the experiments (section \ref{section: Numerical Experiments}) naturally satisfy this assumption.
\end{remark}

\begin{assumption}[Plug-in Estimation Error]
\label{assumption:plug-in-error}
Assume that for some $\epsilon_\mu, \epsilon_p \ge 0$, the estimators $\hat{\mu}_k(x)$ and $\hat{p}_{k|x}$ satisfy
\[
\sup_{x,k} \bigl| \mu_k(x) - \hat{\mu}_k(x) \bigr| \le \epsilon_\mu,
\qquad
\sup_{x} \sum_{k=1}^{|\mathcal{D}|}
\bigl| \hat{p}_{k|x} - p_{k|x} \bigr| \le \epsilon_p.
\]
\end{assumption}

Under these assumptions, we establish a finite-sample excess risk bound that quantifies the statistical and fairness-induced costs of enforcing $\alpha$-fairness and individual solvency.

\begin{theorem}
\label{theorem:sample-complexity-upper-bound-fairness-and-IS}
Let $f_{\alpha, IS} \in \cF_{\alpha,IS}$ be a hypothesis in the restricted hypothesis class $\cF_{\alpha,IS} \subset \mathcal{F}_{\alpha}$ satisfying constraints ~\eqref{eq:individual-solvency}--\eqref{eq:alpha-fairness-c2}. Let $\eta_n$ be a constraint slack parameter. Assume that $|Y| \le B$ a.s., $|f(X,D)| \le B_f$ a.s. for all $f \in \mathcal{F}$, and $p_{\min}:= \inf_{x,k} p_{k|x} > 0$ for all $k$ and $\mathbb{P}_X$-a.e. $x$. Let $M := B_f + B$. If there exist constants $0 < \mu_{\min} \le \mu_{\max} < \infty$ s.t. $\mu_{\min} \le \mu_D(X) \le \mu_{\max}$ a.s., and define the heterogeneity ratio $R_\mu := \mu_{\max}/\mu_{\min}$. Then there exists a universal constant $C > 0$ s.t., for any fixed
$\alpha \in (0,1]$ and $\delta \in (0,1)$, if
\[
\eta_n
=
M \epsilon_p + \epsilon_\mu
+ C M \left(
\sqrt{(V_F + V_T)/n} + \sqrt{\log(4/\delta)/n}
\right),
\]
then, the following holds with probability at least $1 - \delta$:
\begin{align*}
\mathcal{R}(\hat{f}_{\alpha,IS}) - \mathcal{R}(f^*)
&\le
8 M \,\mathfrak{R}_n(\mathcal{F}_\alpha)
+ 6 M^2 \sqrt{\log(4/\delta)/2n} + \mu_{\max}^2 \left[\left( R_\mu - 1/\alpha \right)_+\right]^2,
\end{align*}
where $\mathcal{R}(f)$ is the expected risk under
$\mathbb{P}_{X,D,Y}$ and $\Re_n(\cF_{\alpha})$ is the Rademacher complexity of $\cF_{\alpha}$.
\end{theorem}

\begin{remark}
\label{remark:theorem:sample-complexity-upper-bound-fairness-and-IS}
The excess risk bound decomposes into two components. The first captures the statistical error arising from estimation and function class complexity, and matches the $O(n^{-1/2})$ minimax rate for standard squared-loss regression. The second represents a fairness-induced bias that quantifies the intrinsic cost of enforcing $\alpha$-fairness under heterogeneous risk profiles. The effect of individual solvency enters through the slack term $\eta_n$, which is governed by the accuracy of the plug-in estimators and does not alter the leading convergence rate. As $\alpha \to 0$, the fairness constraint becomes inactive, and the bias term vanishes. Conversely, as $\alpha \to 1$, the fairness constraint is most restrictive, and the bias term attains its maximum. Importantly, this bias represents a worst-case upper bound and is typically modest in practical insurance settings. In personal lines such as auto and homeowners insurance, where fairness regulation is most salient, the heterogeneity ratio $R_\mu$ is often on the order of $10$, reflecting annual premiums ranging from hundreds to thousands of dollars. Consequently, enforcing $\alpha$-fairness and individual solvency does not materially increase the statistical difficulty of learning, and the overall convergence rate remains minimax optimal.
\end{remark}

\begin{corollary}
\label{corollary:constraint-violation-rate}
Following the setup in Theorem \ref{theorem:sample-complexity-upper-bound-fairness-and-IS}. If $\hat{f}_{\alpha,IS} \in \cF_{\alpha}$, then $\hat{f}_{\alpha,IS}$ admits no out-of-sample violation by the definition of $\cF_{\alpha}$. For any fixed tolerance level $\tau>0$, the probability of an out-of-sample solvency violation larger than $\tau$ decays at rate $\cO(n^{-1/2})$.
    
\end{corollary}

We have established that the proposed constrained optimization problem is well-posed (Theorem \ref{theorem:sample-complexity-upper-bound-fairness-and-IS}) and admits a consistent $\cO(n^{-1/2})$ constraint violation rate (\ref{corollary:constraint-violation-rate}). We now study the structural properties of $\alpha$-FISP that align naturally with insurance operations and economic intuition.

\begin{proposition}[Boundedness of $\alpha$-FISP]
\label{proposition:bounded-solution}
Let $f^* \in \cF$ be the optimal solution to the constrained optimization problem ((\ref{eq:objective-function}), (\ref{eq:individual-solvency}), (\ref{eq:alpha-fairness-c1}), and (\ref{eq:alpha-fairness-c2})) for any fixed $\alpha \in (0,1]$. If there exists $\mu_{\min}>0$ s.t. $\mu_k(x) \ge \mu_{\min}$ and $p_{k|x}>0$ for $\bbP_X$-a.e.\ $x \in \cX$ and all $k \in [|\cD|]$.
Then, for $\bbP_X$-a.e.\ $x$, we have
\[
\min_k \mu_k(x) \le f^*(x, \cdot) \le \max_k \mu_k(x).
\]
\end{proposition}

\begin{remark}
\label{remark:proposition:bounded-solution}
Proposition~\ref{proposition:bounded-solution} ensures that $\alpha$-FISP lies within the convex hull of the actuarially fair premiums for each risk class $x$. This property is critical both operationally and economically. Operationally, it guarantees that fairness and solvency constraints do not induce extreme pricing behavior: no premium is pushed below the minimum expected loss or inflated beyond the highest actuarially fair benchmark within the same risk class. Such bounds are consistent with underwriting practices and regulatory expectations, which prohibit loss-leading or excessively punitive pricing. Economically, the result shows that fairness-induced adjustments operate as a reallocation of premiums within an actuarially meaningful range, rather than as a distortion that breaks the link between price and risk. This property also ensures numerical stability of the optimization problem and precludes pathological solutions that could otherwise arise under $\alpha$-fairness constraints.
\end{remark}

\begin{proposition}[Order Preservation]
\label{proposition:order-preserving}
Let $f^* \in \cF$ be the optimal solution to the constrained optimization problem ((\ref{eq:objective-function}), (\ref{eq:individual-solvency}), (\ref{eq:alpha-fairness-c1}), and (\ref{eq:alpha-fairness-c2})). Suppose there exists $\mu_{\min}>0$ s.t. $\mu_k(x) \ge \mu_{\min}$ and $p_{k|x}>0$ for $\bbP_X$-a.e.\ $x \in \cX$ and all $k \in [|\cD|]$.
Then, for $\bbP_X$-a.e.\ $x$, the following hold:

\label{proposition:order-preserving-lower-bound}
(i) If $\mu_k(x) \ge \mu(x)$, then $f^*(x,d_k) \ge \mu(x), \quad \forall \alpha \in (0,1]$.

\label{proposition:order-preserving-upper-bound}
(ii) Suppose there exists a partition of $\cD$ into a low-risk cluster $C_L$
and a high-risk cluster $C_H = \cD \setminus C_L$, such that for some
$M(\alpha,x) > 0$,
\[
\begin{cases}
f^*(x,d_k) = \alpha M(\alpha,x), & \forall k \in [|C_L|] \\
f^*(x,d_k) = M(\alpha,x), & \forall k \in [|C_H|].
\end{cases}
\]
If $\sum_{k \in [|C_L|]} p_{k|x} \le \frac{1}{1+\alpha}$, then for every $k \in [|C_L|]$, $f^*(x,d_k) \le \mu(x), \quad \forall \alpha \in (0,1].$

\label{proposition:order-preserving-strict}
(iii) For any $k,j \in [|\cD|]$ and any fixed $\alpha \in (0,1]$,
\[
\mu_k(x) \ge \mu_j(x) \quad \implies \quad f^*(x,d_k) \ge f^*(x,d_j).
\]
\end{proposition}

\begin{remark}
\label{remark:proposition:order-preserving}

Part (i) shows that if a risk class $(x,d_k)$ is priced above the solidarity fair premium under actuarial fairness ($\alpha = 0$), then for any fairness budget $\alpha \in (0,1]$, the $\alpha$-FISP remains lower bounded by $\mu(x)$. Part (ii) provides a complementary upper-bound result under a structural assumption on the optimal solution and a sufficient condition on the mass of low-risk groups. This condition is economically intuitive: when a low-risk class has dominant probability mass, the solidarity premium $\mu(x)$ approaches its actuarial benchmark, causing the feasible $\alpha$-fair premium band to shift upward in order to satisfy individual solvency while minimizing prediction error. An illustrative example is provided in Appendix~\ref{example:proposition:order-preserving-upper-bound}. While the structural assumption in part (ii) may not always hold in practice,
part (iii) guarantees a robust form of fairness: the relative ordering of premiums across sensitive attributes within a risk class is preserved under $\alpha$-FISP. Thus, actuarially riskier groups are never priced below less risky ones, ensuring monotonicity and interpretability of the resulting premium structure.
\end{remark}

\section{Numerical Experiments}
\label{section: Numerical Experiments}

\subsection{Simulations}
\label{subsection: Simulations}

In this section, we conduct a series of controlled numerical experiments to illustrate the behavior of $\alpha$-FISP under varying fairness budgets and to empirically validate our theoretical results.

We generate two synthetic datasets, each containing $n=10,000$ observations. In Dataset A, the response variable $Y$ represents insurance claim costs. The non-sensitive attributes $X$ consist of age and smoking status. The age variable $X_{\text{A}}$ takes integer values between $18$ and $80$, and the smoking status variable $X_{\text{S}}$ is binary, denoting smoker (S) or non-smoker (NS). The sensitive attribute $D$ represents gender and is binary with levels male (M) and female (F). The data-generating process for Dataset A is:
\begin{align*}
    Y =& 100 + 4 X_{\text{A}} + 100 \cdot \boldsymbol{1}\{X_{\text{S}} = \text{S}\} + 120 \cdot \boldsymbol{1}\{D = \text{F}\} \\ 
    &+ 200 \cdot \boldsymbol{1}\{X_{\text{A}} \in [20, 40], D = \text{F}\} + \epsilon, \epsilon \sim N(0,40^2)
\end{align*}
Dataset B retains the same response $Y$ and non-sensitive attributes $X$, but extends the sensitive attributes to three levels (A, B, and C). The data-generating process is:
\begin{align*}
    Y =& 100 + 4 X_{\text{A}} + 100 \cdot \boldsymbol{1}\{X_{\text{S}} = \text{S}\} + 100 \cdot \boldsymbol{1}\{D = \text{B}\} \\ 
    &+ 200 \cdot \boldsymbol{1} \{D = C\}+ 200 \cdot \boldsymbol{1}\{X_{\text{A}} \in [20, 40], D = \text{B}\} \\
    &+ 180 \cdot \boldsymbol{1}\{X_{\text{A}} \in [50, 60], D = \text{C}\} + \epsilon, \epsilon \sim N(0,40^2)
\end{align*}

To examine the order-preserving property in Proposition \ref{proposition:order-preserving}, we further construct variants A1/B1 and A2/B2 by assigning different group probabilities. Specifically, A1 and B1 satisfy the condition $\sum_{k \in [|C_L|]} p_{k|x} \le \frac{1}{1 + \alpha}$, whereas A2 and B2 violate this condition. In all simulations, we assume $X_{\text{A}} \perp D$ and $X_{\text{A}} \perp X_{\text{S}}$. Detailed probability assignments are provided in Appendix \ref{appendix:simulation-prob-assignment}.

All models are trained using two-layer feed-forward neural networks, with hidden layer sizes of $(5,5)$ for Dataset A and $(16,8)$ for Dataset B. Model training is performed by optimizing the Lagrangian associated with the constrained optimization problem defined by the objective (\ref{eq:objective-function}) and the constraints (\ref{eq:individual-solvency}), (\ref{eq:alpha-fairness-c1}), and (\ref{eq:alpha-fairness-c2}). The final model parameters are obtained by averaging across five random splits. 



Figure \ref{fig:toy_gt_equivalence_alpha} illustrates the behavior of $\alpha$-FISP using Dataset A1. Panel (a) shows that as $\alpha \to 0$, the premiums converge to the actuarially fair premiums defined in Definition~\ref{def:actuarial-fair-premium}, whereas panel (b) shows that as $\alpha \to 1$, they approach solidarity-fair premiums as defined in Definition~\ref{def:solidarity-premium}.

\begin{figure}[h]
    \begin{subfigure}{0.5\textwidth}
        \centering
        \includegraphics[width=\textwidth]{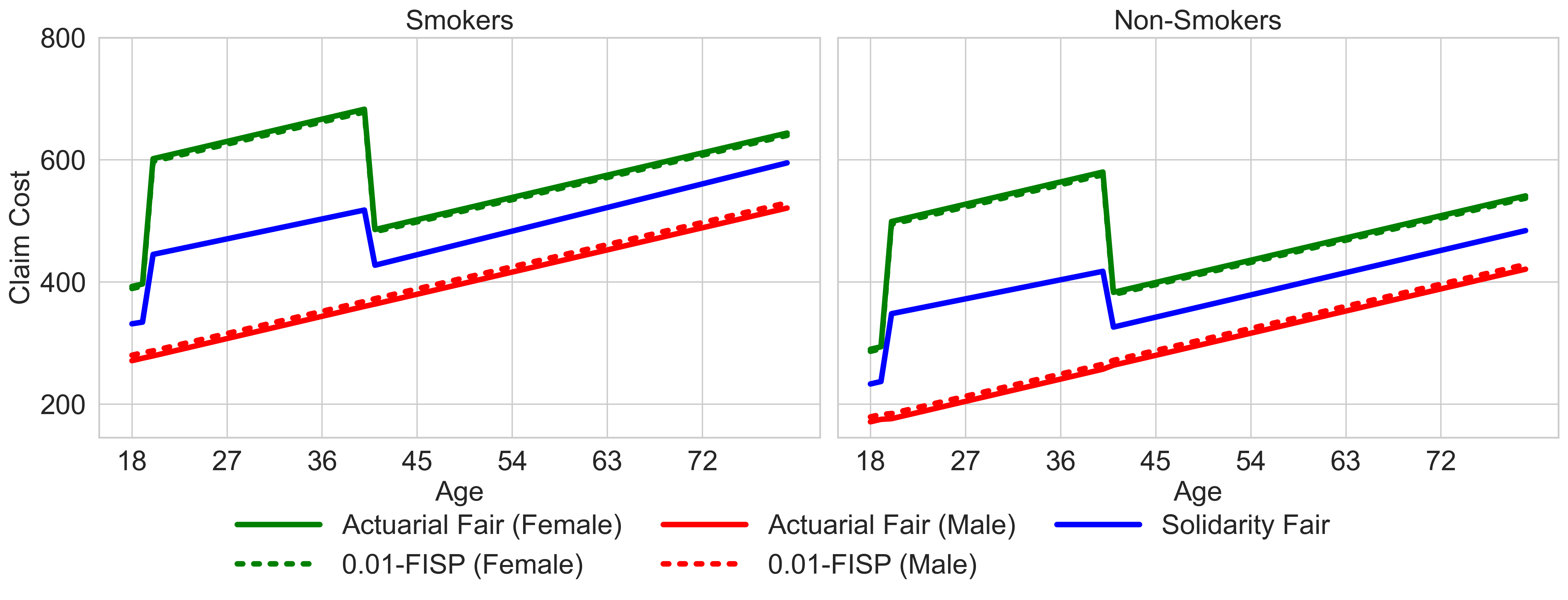}
        \caption{Convergence to actuarially fair premiums as $\alpha\to 0$}
        \label{subfig:toy_alpha0.01_gt_s}
    \end{subfigure}
    \begin{subfigure}{0.5\textwidth}
        \centering
        \includegraphics[width=\textwidth]{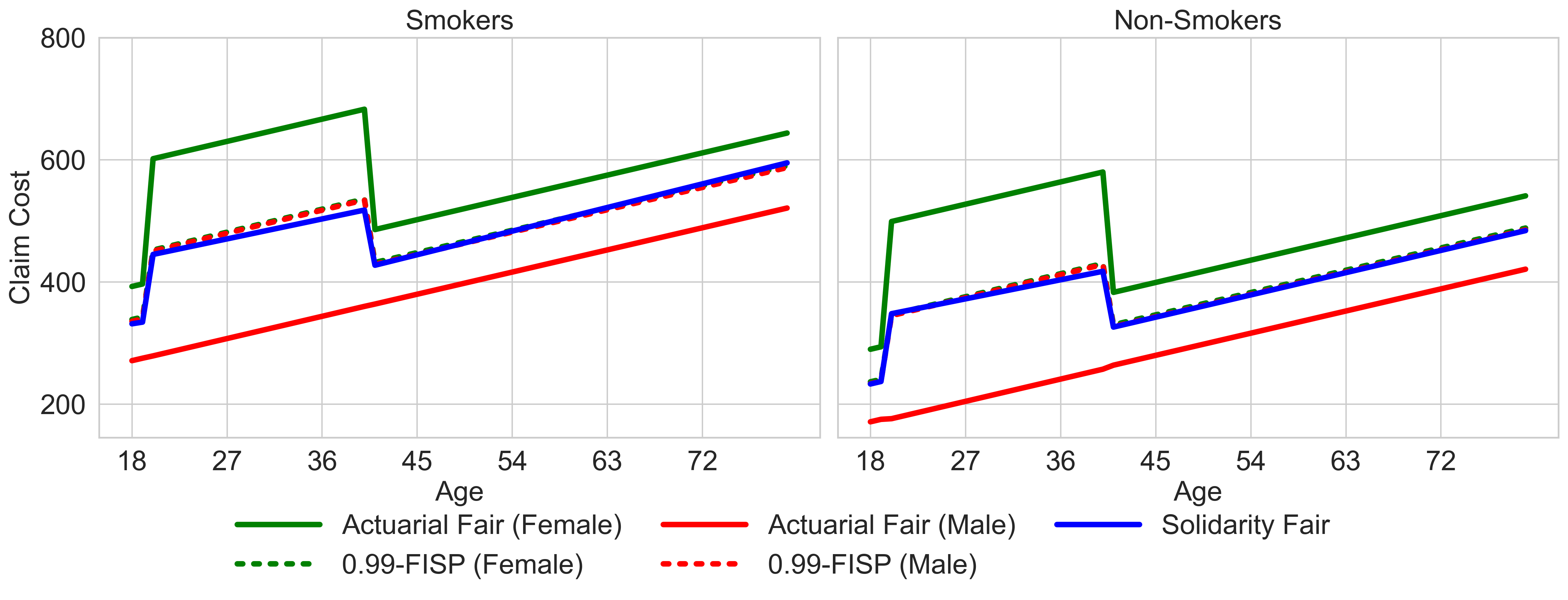}
        \caption{Convergence to solidarity-fair premiums as $\alpha\to 1$}
        \label{subfig:toy_alpha0.99_gt_s}
    \end{subfigure}
    \caption{$\alpha$-fair premiums at the two fairness extremes}
    \label{fig:toy_gt_equivalence_alpha}
\end{figure} 

Figure \ref{fig:toy_gt_alpha} illustrates the role of the low-risk mass condition $\sum_{k \in [|C_L|]} p_{k|x} \le \frac{1}{1 + \alpha}$ in ensuring order preservation for $\alpha$-fair premiums. Using a representative value of $\alpha=0.8$, panel (a), constructed using Dataset A1, shows that low-risk $\alpha$-FISP is bounded above by $\mu(x)$ and high-risk $\alpha$-FISP is bounded below by $\mu(x)$ when the condition is satisfied, while panel (b), based on Dataset A2, demonstrates that when the condition is violated, low-risk $\alpha$-FISP may exceed the solidarity fair premium $\mu(x)$. 

\begin{figure}[h]
    \begin{subfigure}{0.5\textwidth}
        \centering
        \includegraphics[width=\textwidth]{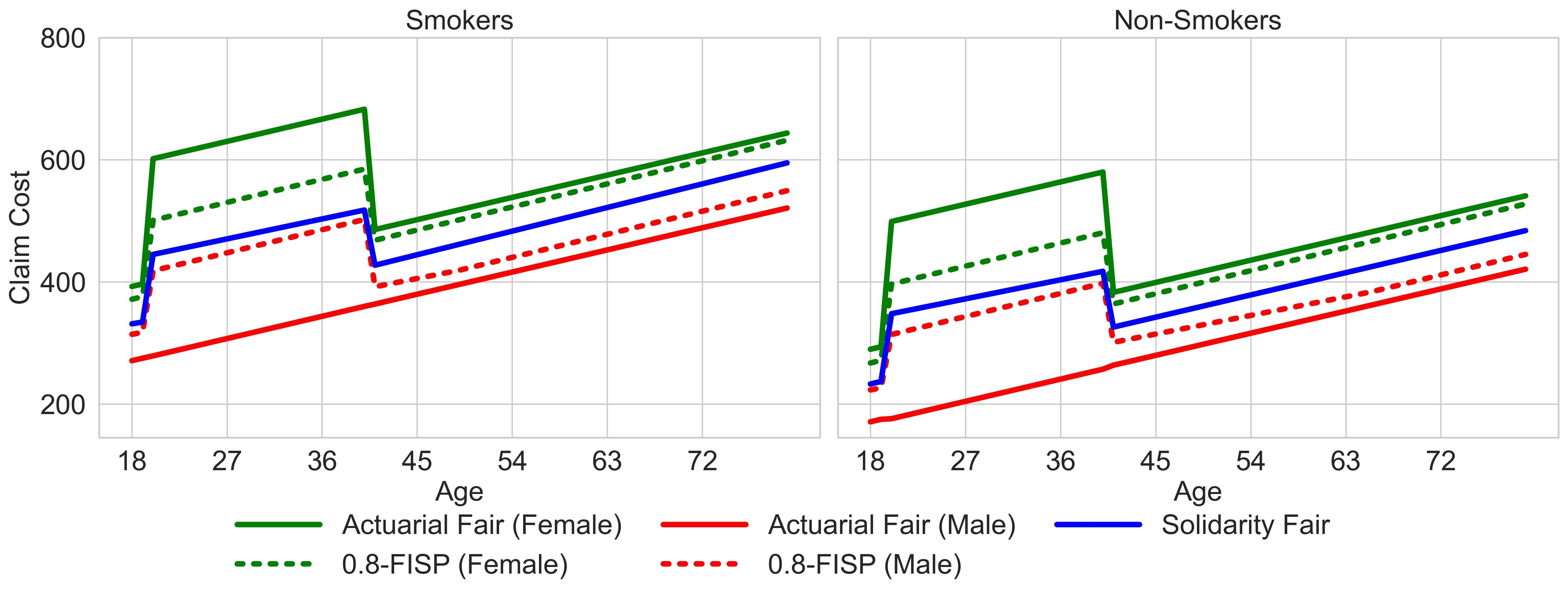}
        \caption{Condition satisfied ($P_L = 0.5 \le \frac{1}{1 + \alpha}$)}
        \label{subfig:toy_alpha0.8_gt_s}
    \end{subfigure}
    \begin{subfigure}{0.5\textwidth}
        \centering
        \includegraphics[width=\textwidth]{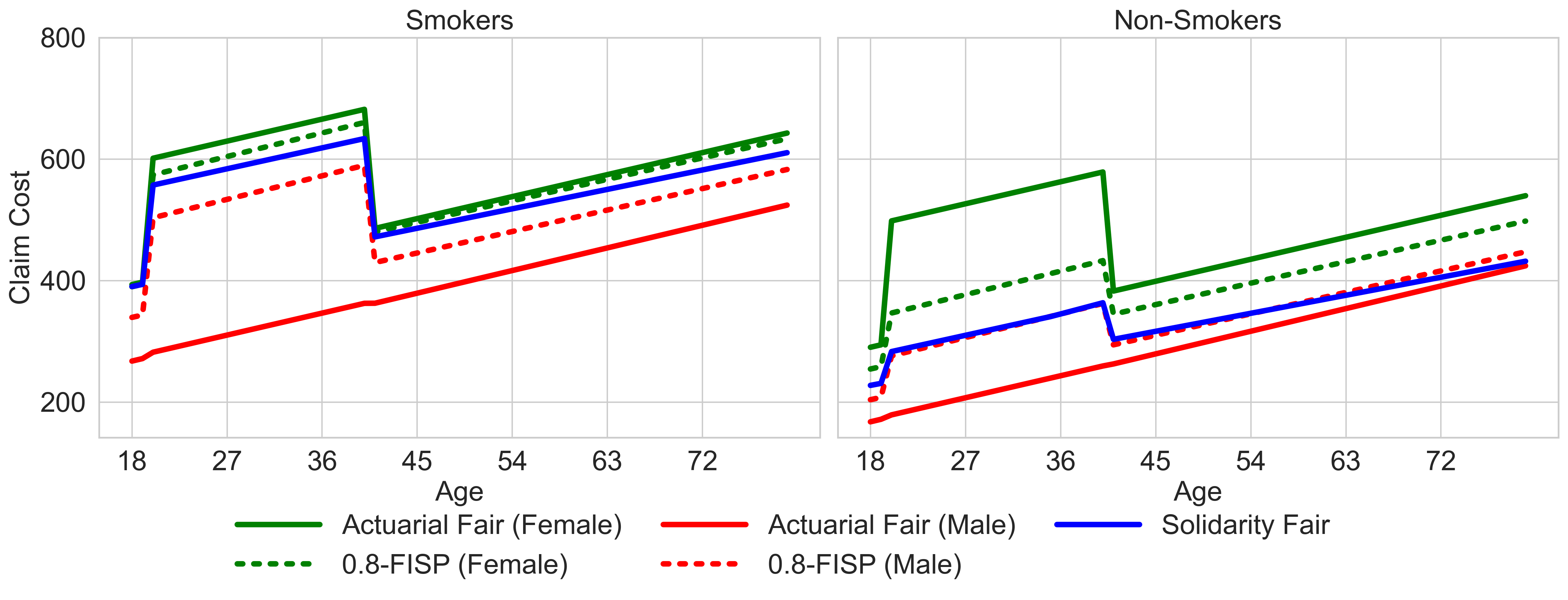}
        \caption{Condition violated ($P_L = 0.7 > \frac{1}{1 + \alpha}$)}
        \label{subfig:toy_alpha0.8_gt_v}
    \end{subfigure}
     \caption{Order preservation of $\alpha$-fair premiums for $|\cD| = 2$ ($\alpha=0.8$)}
    \label{fig:toy_gt_alpha}
\end{figure} 

Figure \ref{fig:toy3_gt_alpha} is analogous to Figure \ref{fig:toy_gt_alpha} with $\alpha = 0.8$, but is obtained using Datasets B1 and B2 for panels (a) and (b), respectively, and illustrates the general case with $|\cD| > 2$. The contrast is more pronounced in this multi-group setting: when the low-risk mass condition is violated, the low-risk $\alpha$-FISP exceeds the solidarity fair premium $\mu(x)$, underscoring the necessity of the condition stated in Proposition \ref{proposition:order-preserving} (ii). Nevertheless, the relative premium ordering is maintained in all cases, as proved in Proposition \ref{proposition:order-preserving-strict} (iii).

\begin{figure}[H]
    \begin{subfigure}{0.49\textwidth}
        \centering
        \includegraphics[width=\textwidth]{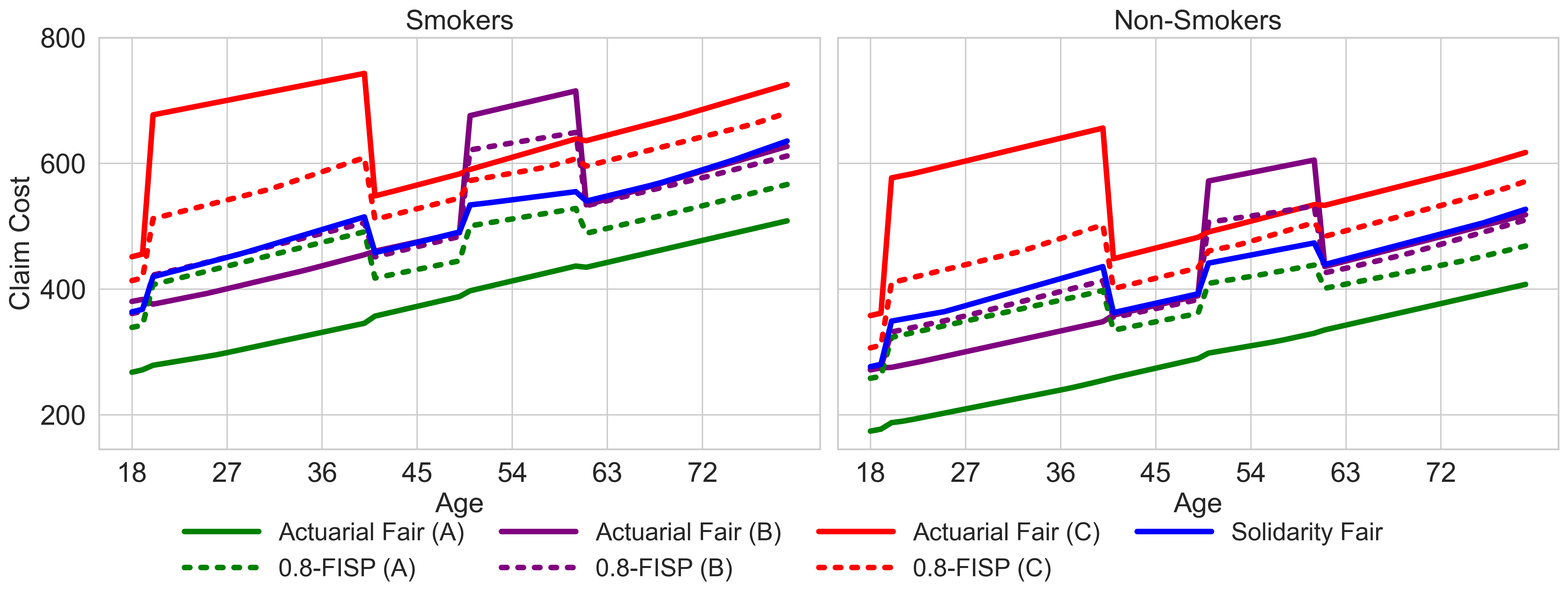}
        \caption{Condition satisfied ($P_L = \frac{1}{3} \le \frac{1}{1 + \alpha}$)}
        \label{subfig:toy3_alpha0.8_gt_s}
    \end{subfigure}
    \begin{subfigure}{0.49\textwidth}
        \centering
        \includegraphics[width=\textwidth]{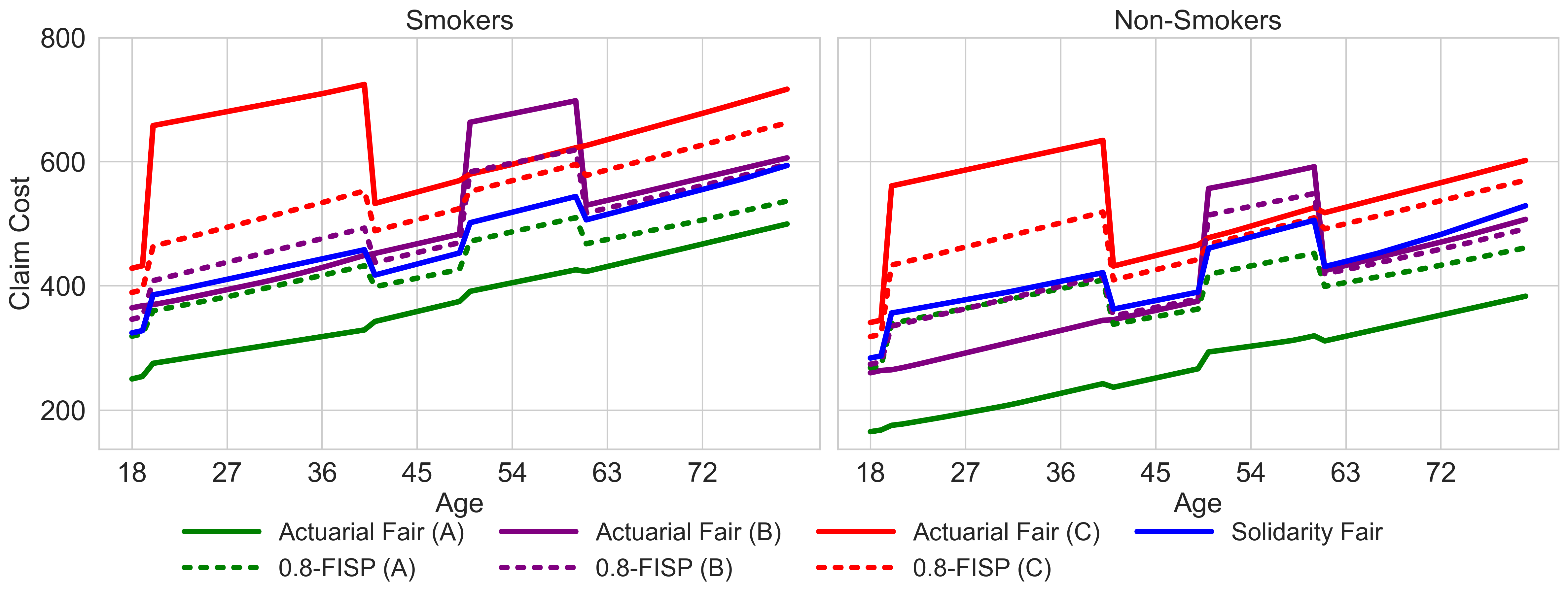}
        \caption{Condition violated ($P_L = 0.6 > \frac{1}{1 + \alpha}$)}
        \label{subfig:toy3_alpha0.8_gt_v}
    \end{subfigure}
    \caption{Order preservation of $\alpha$-fair premiums for $|\cD| > 2$}
    \label{fig:toy3_gt_alpha}
\end{figure} 

\subsection{Health Insurance}
\label{subsection: Health Insurance}

We analyze the U.S. Health Insurance dataset \citep{10.5555/2588158}, treating sex as the sensitive attribute. Reflecting common industry practices where tail risks are ceded to reinsurance, we restrict our study to the dense region of frequent claims. Within this region, we examine observations at the $25^{\text{th}}, 50^{\text{th}}$, and $75^{\text{th}}$ percentile of health expenditures as representative risk classes. As in Section \ref{subsection: Simulations}, we employ a two-layer feed-forward NN with $16$ and $8$ neurons in the first and second hidden layers, respectively. Figure \ref{fig:Health_alpha_sensitivity_50th_segments} displays $\alpha$-FISP for $\alpha \in \{0.25,0.5,0.75\}$, with actuarially fair and solidarity-fair premiums shown as benchmarks, for the risk class corresponding to the 50th percentile of health expenditures.
\begin{figure}[H]
    \centering
    \includegraphics[width=0.7\textwidth]{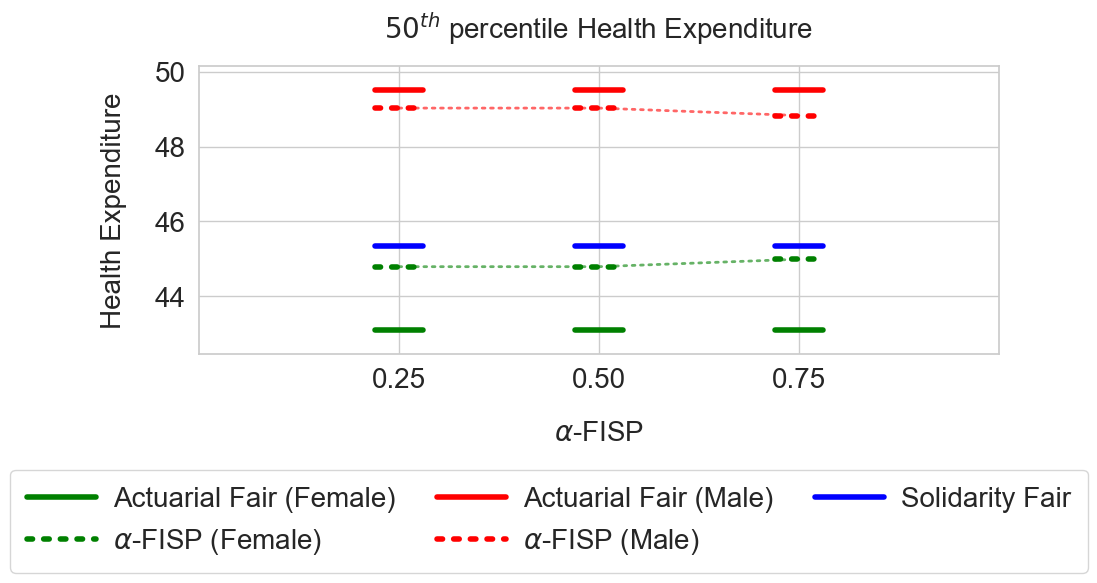}
    \caption{$\alpha$-fair premiums for the median health risk class}   
    \label{fig:Health_alpha_sensitivity_50th_segments}
\end{figure}
The results illustrate that increasing $\alpha$ tightens the fairness constraint, progressively shifting premiums from actuarial levels toward the solidarity benchmark. This visualization highlights the cross-subsidization mechanism embedded in the $\alpha$-fair framwork and empirically validates the order-preserving property established in Proposition \ref{proposition:order-preserving}. These findings are consistent with the simulation results in Section \ref{subsection: Simulations}, and analogous patterns for other expenditure percentiles are in Appendix \ref{appendix:deferred-health-insurance-figures}. 

\subsection{Solvency Margin and Choice of $\alpha$}
\label{subsection: Solvency Margin and Choice of alpha}

In this section, we first empirically compare $\alpha$-FISP with the widely cited fair insurance pricing method (DFIP) proposed by \cite{lindholm2022dfip}. Then, we provide some insight on how to choose $\alpha$ in practice.

\begin{table}[htbp]
    \centering
    \scriptsize
    \setlength{\tabcolsep}{3pt}
    \renewcommand{\arraystretch}{0.92}
    \resizebox{\textwidth}{!}{
    \begin{tabular}{llrrrrrr}
        \toprule
        \textbf{Condition} & \textbf{Metric}
        & \multicolumn{6}{c}{$\boldsymbol{\alpha}$} \\
        \cmidrule(lr){3-8}
        & & \textbf{0.01} & \textbf{0.20} & \textbf{0.40}
          & \textbf{0.60} & \textbf{0.80} & \textbf{0.99} \\
        \midrule
        \multirow{6}{*}{Satisfied}
        & $\alpha$-FISP Test MSE 
        & 1658.57 & 1659.62 & 1697.04 & 3349.90 & 7039.25 & 12798.40 \\
        & DFIP Test MSE 
        & 13081.87 & 13081.87 & 13081.87 & 13081.87 & 13081.87 & 13081.87 \\
        & $\alpha$-FISP Solvency Margin 
        & \textbf{2.16} & \textbf{2.20} & \textbf{3.33} & \textbf{9.35} & \textbf{10.72} & \textbf{7.09} \\
        & DFIP Solvency Margin 
        & -9.87 & -9.87 & -9.87 & -9.87 & -9.87 & -9.87 \\
        & $\alpha$-FISP Solvency Violation Rate (\%) 
        & 0.00 & 0.00 & 0.00 & 0.00 & 0.00 & 0.00 \\
        & $\alpha$-FISP Fairness Violation Rate (\%) 
        & 0.00 & 0.00 & 0.74 & 1.22 & 1.11 & 0.70 \\
        \midrule
        \multirow{6}{*}{Violated}
        & $\alpha$-FISP Test MSE 
        & 1674.11 & 1674.11 & 1687.53 & 3727.72 & 6555.59 & 11302.19 \\
        & DFIP Test MSE 
        & 12551.79 & 12551.79 & 12551.79 & 12551.79 & 12551.79 & 12551.79 \\
        & $\alpha$-FISP Solvency Margin 
        & \textbf{2.63} & \textbf{2.63} & \textbf{3.42} & \textbf{13.99} & \textbf{9.29} & \textbf{15.45} \\
        & DFIP Solvency Margin 
        & -7.17 & -7.17 & -7.17 & -7.17 & -7.17 & -7.17 \\
        & $\alpha$-FISP Solvency Violation Rate (\%) 
        & 0.00 & 0.00 & 0.00 & 0.00 & 0.00 & 0.00 \\
        & $\alpha$-FISP Fairness Violation Rate (\%) 
        & 0.00 & 0.00 & 0.73 & 0.88 & 0.57 & 0.00 \\
        \bottomrule
    \end{tabular}
    }
    \caption{$\alpha$-FISP and DFIP Comparison on Dataset A ($|\cD| = 2$)}
    \label{tab:toy_DFIP_comparison}
\end{table}

\begin{table}[htbp]
    \centering
    \scriptsize
    \setlength{\tabcolsep}{3pt}
    \renewcommand{\arraystretch}{0.92}
    \resizebox{0.75\textwidth}{!}{
    \begin{tabular}{lrrrrrr}
        \toprule
        \textbf{Metric}
        & \multicolumn{6}{c}{$\boldsymbol{\alpha}$} \\
        \cmidrule(lr){2-7}
        & \textbf{0.01} & \textbf{0.20} & \textbf{0.40}
        & \textbf{0.60} & \textbf{0.80} & \textbf{0.99} \\
        \midrule
        $\alpha$-FISP Test MSE
        & 10.41 & 10.41 & 10.41 & 10.41 & 10.70 & 17.03 \\
        DFIP Test MSE
        & 18.48 & 18.48 & 18.48 & 18.48 & 18.48 & 18.48 \\
        $\alpha$-FISP Solvency Margin
        & \textbf{0.68} & \textbf{0.68} & \textbf{0.68}
        & \textbf{0.68} & \textbf{0.62} & \textbf{1.32} \\
        DFIP Solvency Margin
        & -0.19 & -0.19 & -0.19 & -0.19 & -0.19 & -0.19 \\
        $\alpha$-FISP Solvency Violation Rate (\%)
        & 0.00 & 0.00 & 0.00 & 0.00 & 0.00 & 0.17 \\
        $\alpha$-FISP Fairness Violation Rate (\%)
        & 0.00 & 0.00 & 0.00 & 0.00 & 0.00 & 2.93 \\
        \bottomrule
    \end{tabular}
    }
    \caption{$\alpha$-FISP and DFIP Comparison on Healthcare Insurance Dataset}
    \label{tab:health_DFIP_comparison}
\end{table}

Table \ref{tab:toy_DFIP_comparison} and \ref{tab:health_DFIP_comparison} show that $\alpha$-FISP maintains a positive solvency margin across all fairness budgets, while DFIP produced insolvent premiums. This signifies one key merit of $\alpha$-FISP: explicitly incorporating the critical solvency requirement in risk pricing, which is overlooked in the current fair insurance pricing literature. While Corollary \ref{corollary:constraint-violation-rate} suggests no out-of-sample fairness violation assuming the model class fully contains $\cF_\alpha$, the positive fairness violation rates here are due to model mis-specification-the model class specified may not fully contain $\cF_{\alpha}$.

\begin{figure}[H]
    \begin{subfigure}{0.33\textwidth}
        \centering
        \includegraphics[width=\textwidth]{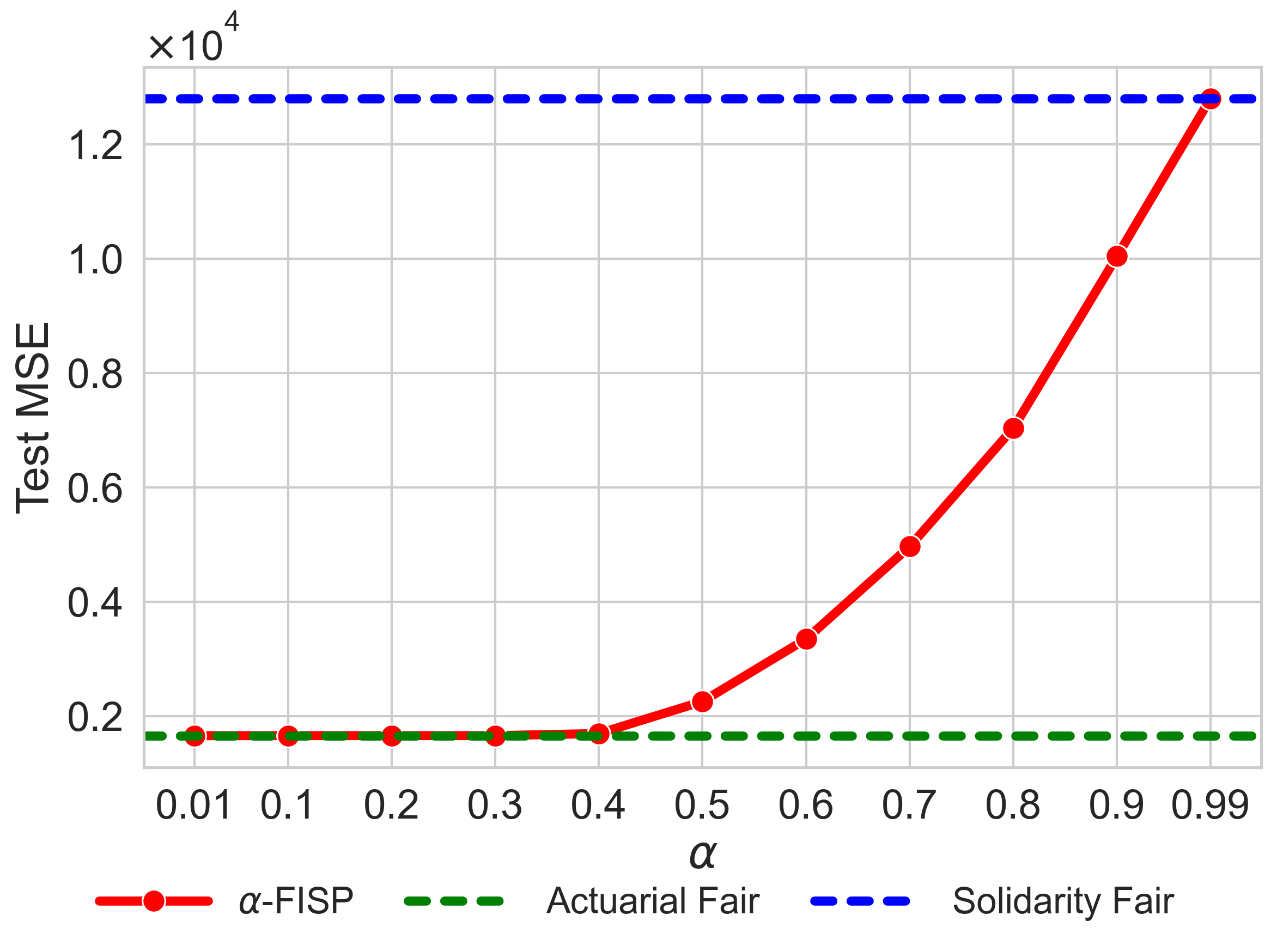}
        \caption{Dataset A1}
        \label{fig:toy_mse_benchmark_s}
    \end{subfigure}
    \begin{subfigure}{0.33\textwidth}
        \centering
        \includegraphics[width=\textwidth]{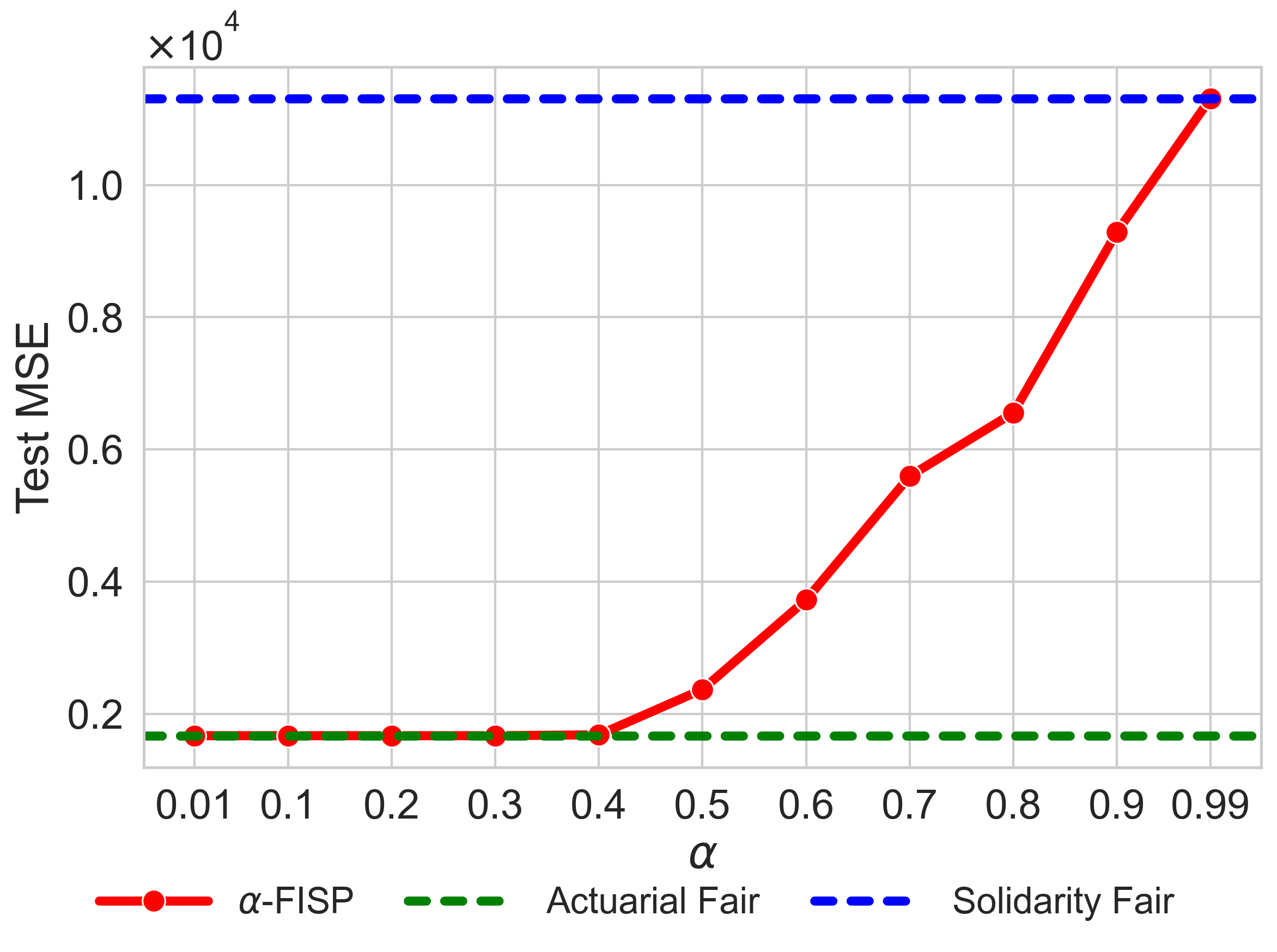}
        \caption{Dataset A2}
        \label{fig:toy_mse_benchmark_v}
    \end{subfigure}
    \begin{subfigure}{0.33\textwidth}
        \centering
        \includegraphics[width=\textwidth]{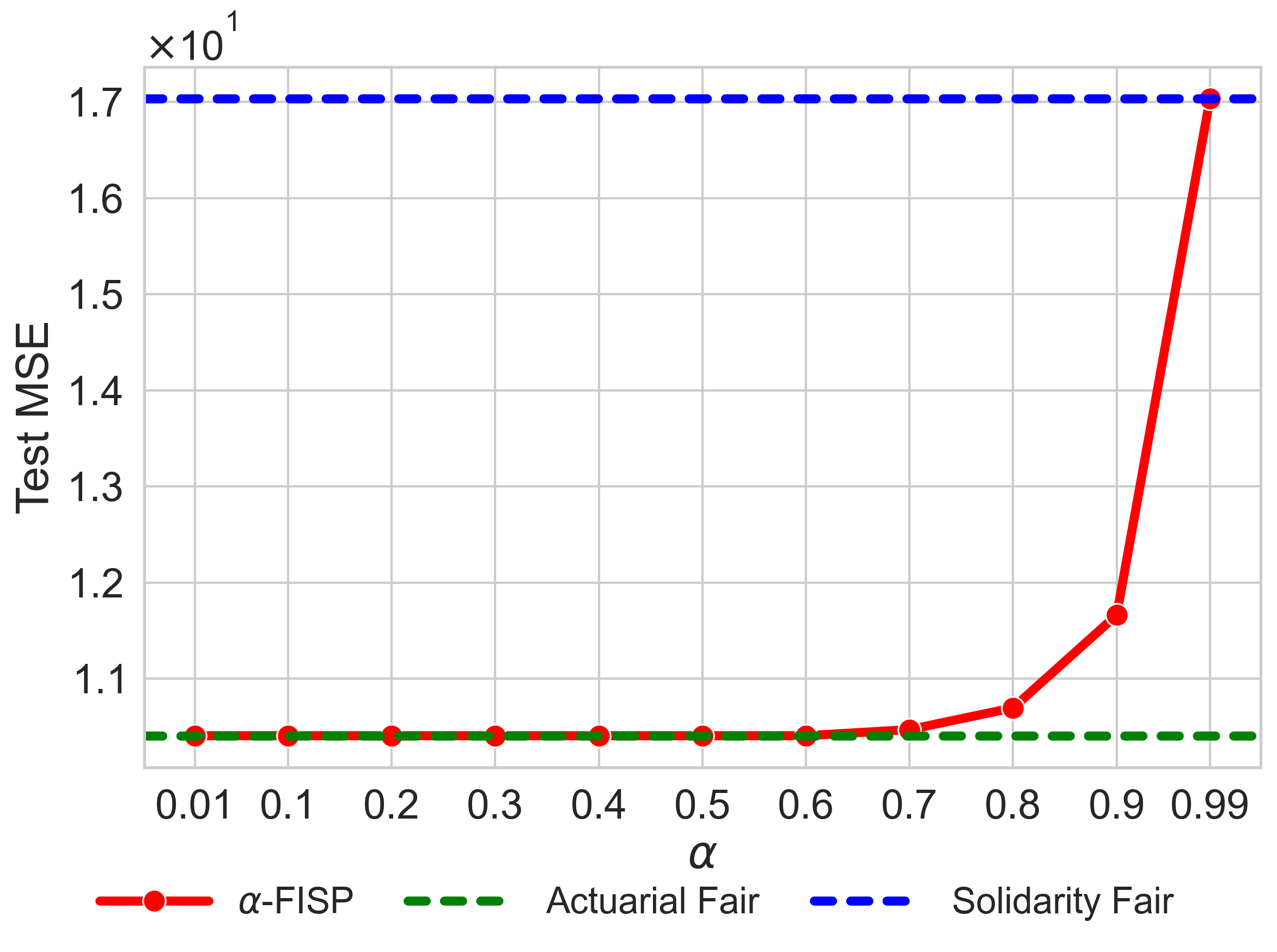}
        \caption{Healthcare Insurance Dataset}
        \label{fig:health_mse_benchmark}
    \end{subfigure}
    \caption{MSE vs. $\alpha$ Plots}
    \label{fig:toy_health_mse_benchmark}
\end{figure} 

Figure \ref{fig:toy_mse_benchmark_s}, \ref{fig:toy_mse_benchmark_v}, and \ref{fig:health_mse_benchmark} show different behaviors of MSE loss as $\alpha$ increases. This suggests that the sensitivity of MSE on $\alpha$ is highly dependent on the underlying data distribution. Therefore, creating such plots may be a necessary first step in selecting an optimal $\alpha$ that best balances two competing fairness notions. For example, comparing the optimal $\alpha$ for dataset A and the Healthcare insurance dataset, a smaller $\alpha$ may be more defensible for Dataset A, while a larger $\alpha$ may be more defensible since the MSE does not rise significantly even for $\alpha = 0.9$.

\section{Conclusion}
This paper introduces $\alpha$-fair insurance pricing, a principled framework that formalizes fairness in insurance as a continuum between actuarial and solidarity perspectives. By directly regulating premium dispersion within risk classes, $\alpha$-fairness captures the economic substance of fairness in pricing decisions, rather than relying on indirect predictive or demographic criteria. We formulate pricing as a constrained optimization problem, yielding the $\alpha$-FISP as a flexible and tractable solution. We establish theoretical guarantees and show through numerical experiments that the framework is computationally efficient and compatible with heterogeneous regulatory environments. More broadly, this work reframes fairness in insurance pricing as an allocation problem, offering a transparent and interpretable mechanism for navigating regulatory and ethical trade-offs.

\newpage
\bibliographystyle{chicago}
\bibliography{neurips_2026}


\newpage
\appendix

\section{Appendix A: Deferred Examples}
\label{appendix:deferred-examples}

\subsection{Proposition \ref{proposition:order-preserving} (ii) Example}
\label{example:proposition:order-preserving-upper-bound}
We present a simple numeric example in this section. Consider a simplified setting with a single risk class (ignoring non-sensitive risk class $x$ for brevity) and a binary sensitive attribute $D \in \{L,H\}$, representing Low-risk and High-risk groups. Assume actuarial fair premiums, $\mu_L = 10$ and $\mu_H = 100$ for each class, and a fairness budget $\alpha = 0.5$. According to Proposition \ref{proposition:order-preserving} (ii), we calculate the sufficient cutoff $p_L \le \frac{1}{1+\alpha} = \frac{1}{1+0.5}$, which is approximately 0.67. We compare two scenarios varying the population mix $p_L = \bbP(D=L)$.

\textbf{Case 1: Condition Satisfied ($p_L = 0.5$)}. We assume low-risk mass equals to 0.5, which is below the threshold 0.67. Solidarity premium, $\mu$, is calculated through a weighted average: $0.5 \cdot10+(1-0.5)\cdot100 = 55$. Recall optimal solution is defined as the minimizer of weighted loss function $L = p_L\cdot(f_L-\mu_L)^2+p_H\cdot(f_H-\mu_H)^2$. Under $\alpha$-FISP constraint, we have $f_L=\alpha \cdot f_H$. Then, re-formulate loss function to $L(f_H) = p_L\cdot(\alpha \cdot f_H-\mu_L)^2+p_H\cdot(f_H-\mu_H)^2$, and take derivative respect to $f_H$, we have $f_H^* = \cfrac{\alpha p_L\mu_L+p_H\mu_H}{\alpha^2 p_L + p_H}$. Plug-in numerical values, we have $f_H^* =84$, and $f_L^* =42$. We can see solvency constraint is naturally satisfied as $0.5\cdot84+0.5\cdot42 = 63 \ge 55$. We also observe that $f_L^* \le \mu$. The result is intuitive: the fair premium for the low-risk group is subsidized but remains below the population average.

\textbf{Case 2: Condition Violated ($p_L = 0.9$)}. Now the low-risk group dominates the population ($0.9 > 0.67$). Use similar formulation as in previous case, $\mu$ equals to 19; $f_H^*$ equals to approximately  44.62; $f_L^*$ equals to approximately 22.31. Individual solvency is satisfied ($0.9\cdot22.31+0.1\cdot 44.62 \approx 24.54 \ge 19$), but notice that $f_L^* > \mu$. $f_L^* < f_H^*$ confirms result Proposition \ref{proposition:order-preserving} (iii), but low-risk mass domination, under $\alpha$-fair constraint, leads to price paid by low-risk individuals exceed the solidarity benchmark.

Through this illustrative example, we demonstrate the practical implication of Proposition \ref{proposition:order-preserving}. When the low-risk mass condition is satisfied, the $\alpha$-FISP framework performs ideally, yielding a premium structure that naturally interpolates between actuarial and solidarity fairness. When the low-risk mass condition is violated, rigid adherence to a proportional fairness in skewed populaitons, especially when ensure myopic fair forcing low-risk premium equivalent to high-risk premium, may unintentionally over-penalize low-risk individuals, creating a scenario where the "fair" price is paradoxically more expensive than a fully solidified price. Our $\alpha$-FISP flexibility prevents the over-penalization of low-risk individuals that often arises from binary notion of fairness.

\newpage
\section{Appendix B: Deferred Figures}
\label{appendix:deferred-figures}

\subsection{Simulation Probability Assignment}
\label{appendix:simulation-prob-assignment}

In A1 (condition satisfied): $\bbP(X_{\text{S}} = \text{S}) = 0.5$ and $\bbP(D = \text{F} | X_{\text{S}} = \text{S}) = \bbP(D = \text{F} | X_{\text{S}} = \text{NS}) = 0.5$. 

In A2 (condition violated): $\bbP(X_{\text{S}} = \text{S}) = 0.3$, $\bbP(D = \text{F} | X_{\text{S}} = \text{S}) = 0.8$, and $\bbP(D = \text{F} | X_{\text{S}} = \text{NS}) = 0.3$. 

In B1 (condition satisfied): $\bbP(X_{\text{S}} = \text{S}) = 0.5$ and $\bbP(D = \text{A} | X_{\text{S}} = \text{S}) = \bbP(D = \text{B} | X_{\text{S}} = \text{S}) = \bbP(D = \text{C} | X_{\text{S}} = \text{S}) = \frac{1}{3}$. 

In B2 (condition violated): $\bbP(X_{\text{S}} = \text{S}) = 0.5$, $\bbP(D = \text{A} | X_{\text{S}} = \text{S}) = 0.6$, and $\bbP(D = \text{B} | X_{\text{S}} = \text{S}) = \bbP(D = \text{C} | X_{\text{S}} = \text{S}) = 0.2$

\newpage
\subsection{Simulation Figures (D = 2)}
\label{appendix:deferred-simulation-figures_D=2}

we present $\alpha$-FISP under varying fairness budget $\alpha$ for datasets A1 and A2 (i.e. $|\cD| = 2$) in the following:

\begin{figure}[H]
    \begin{subfigure}{0.5\textwidth}
        \centering
        \includegraphics[width=\textwidth]{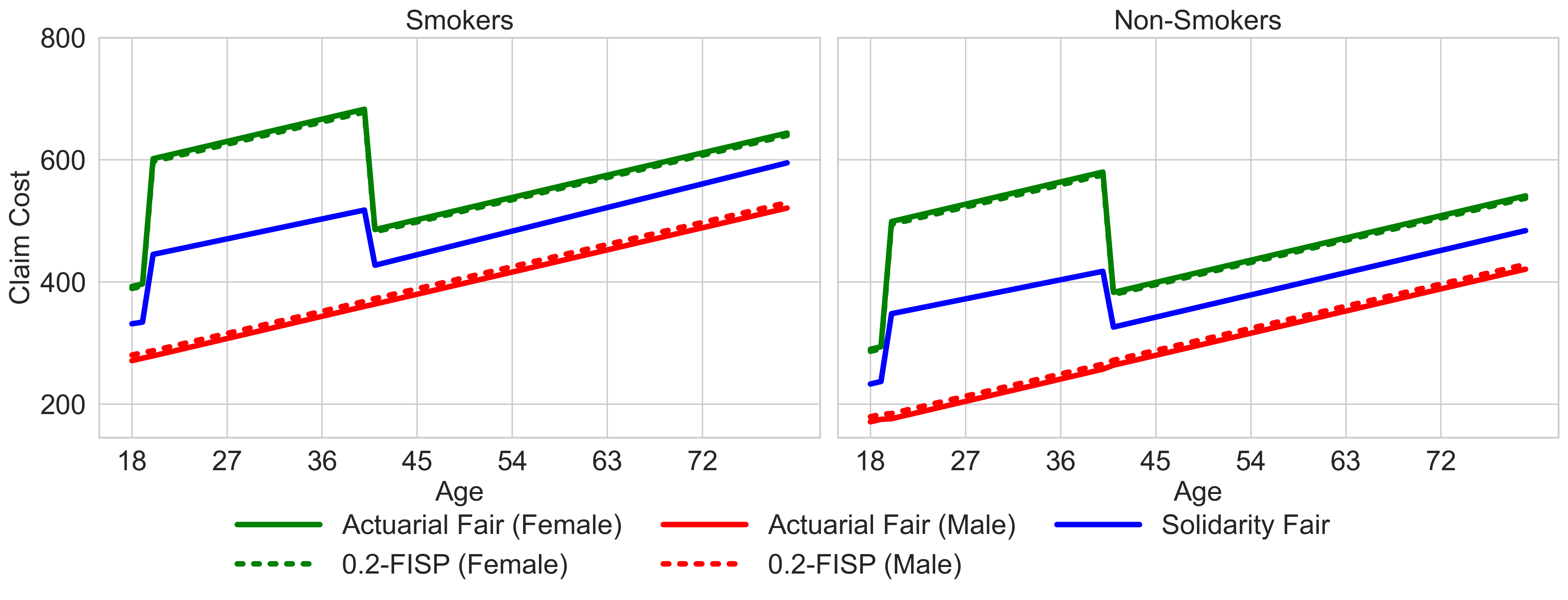}
        \caption{Condition satisfied ($P_L = 0.5 \le \frac{1}{1 + \alpha}$)}
        \label{subfig:toy_alpha0.2_gt_s}
    \end{subfigure}
    \begin{subfigure}{0.5\textwidth}
        \centering
        \includegraphics[width=\textwidth]{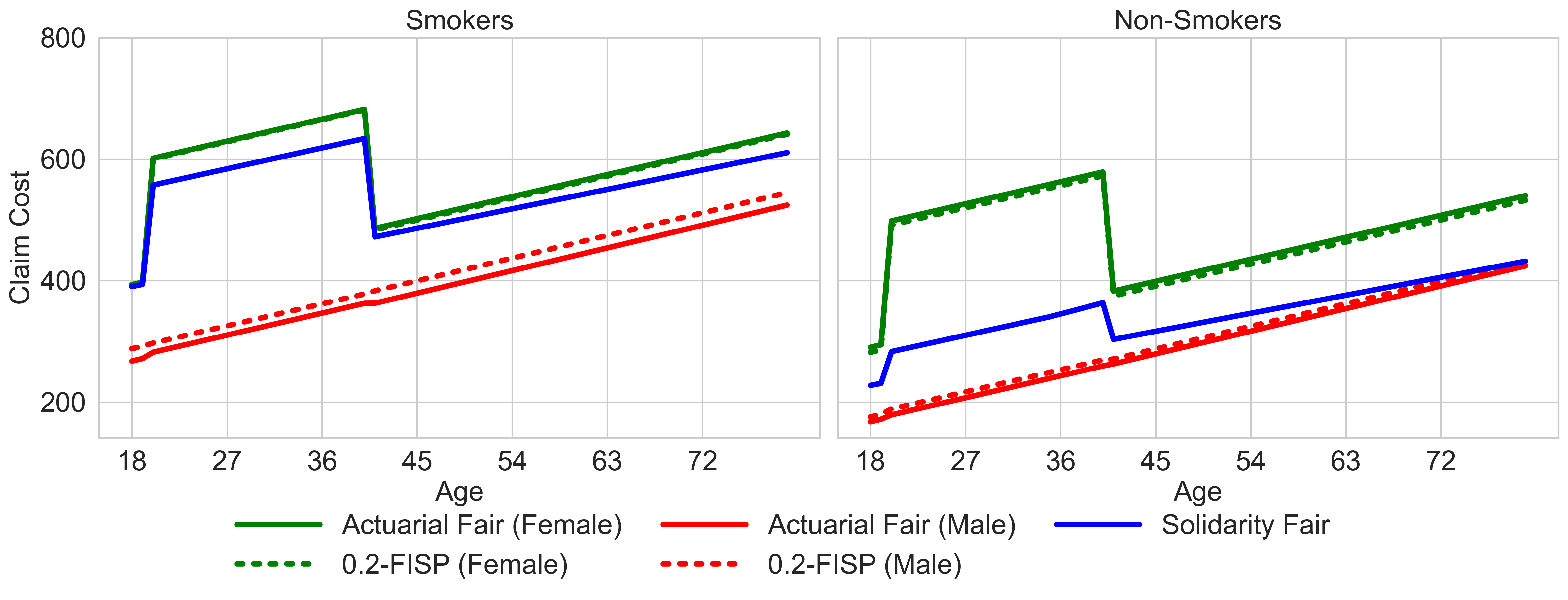}
        \caption{Condition satisfied ($P_L = 0.7 \le \frac{1}{1 + \alpha}$)}
        \label{subfig:toy_alpha0.2_gt_v}
    \end{subfigure}
     \caption{Order preservation of $\alpha$-fair premiums for $|\cD| = 2$ and $\alpha = 0.2$}
    \label{fig:toy_gt_alpha_0.2}
\end{figure} 

\begin{figure}[H]
    \begin{subfigure}{0.5\textwidth}
        \centering
        \includegraphics[width=\textwidth]{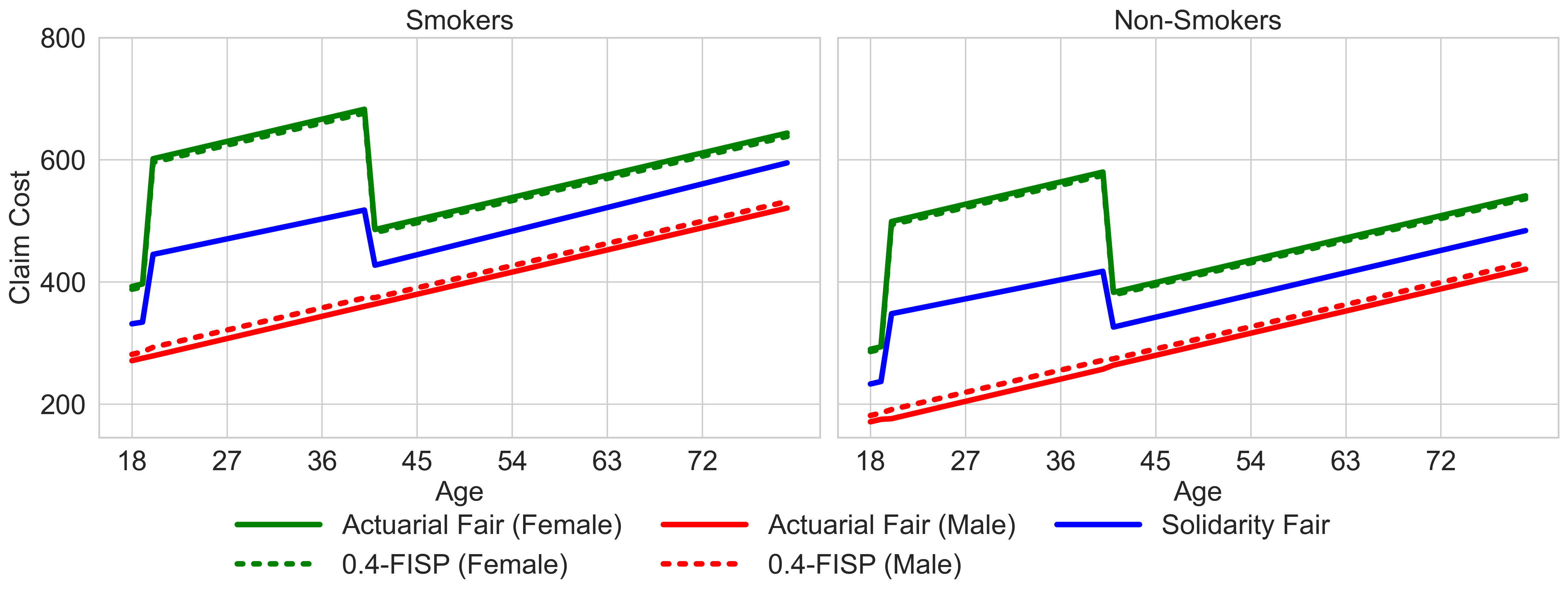}
        \caption{Condition satisfied ($P_L = 0.5 \le \frac{1}{1 + \alpha}$)}
        \label{subfig:toy_alpha0.4_gt_s}
    \end{subfigure}
    \begin{subfigure}{0.5\textwidth}
        \centering
        \includegraphics[width=\textwidth]{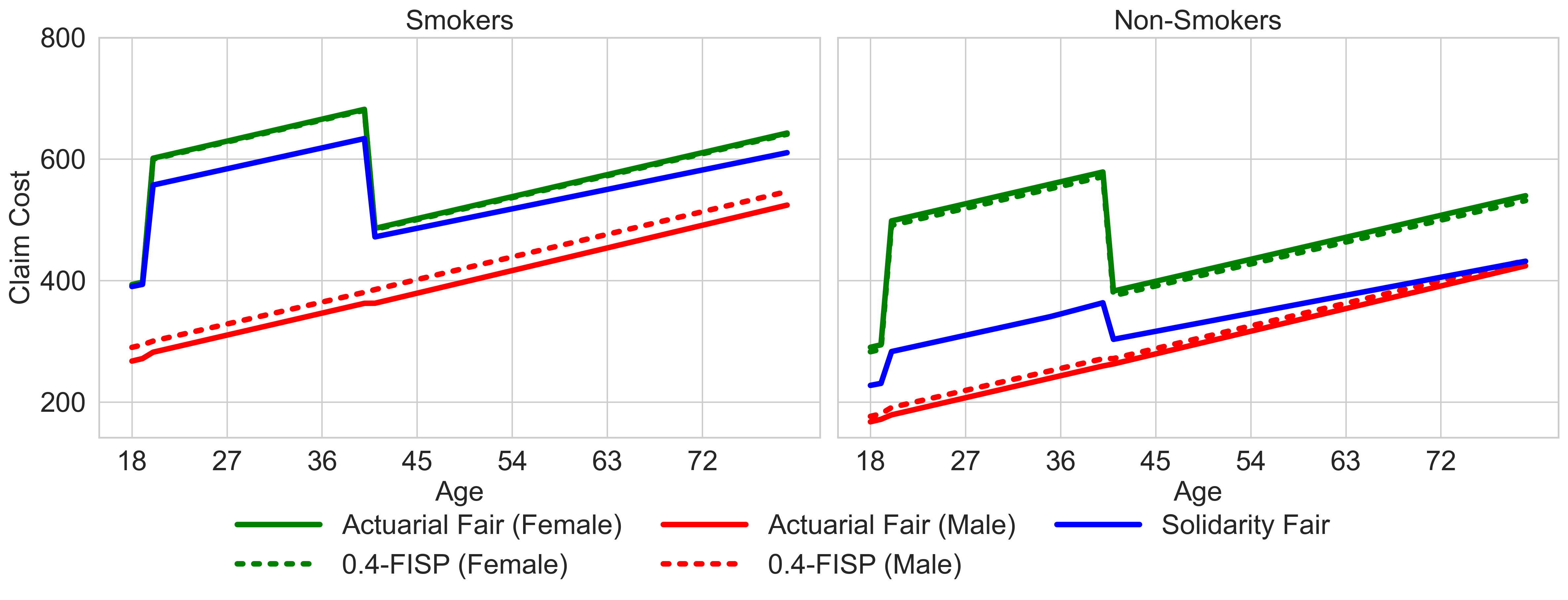}
        \caption{Condition satisfied ($P_L = 0.7 \le \frac{1}{1 + \alpha}$)}
        \label{subfig:toy_alpha0.4_gt_v}
    \end{subfigure}
     \caption{Order preservation of $\alpha$-fair premiums for $|\cD| = 2$ and $\alpha = 0.4$}
    \label{fig:toy_gt_alpha_0.4}
\end{figure} 

\begin{figure}[H]
    \begin{subfigure}{0.5\textwidth}
        \centering
        \includegraphics[width=\textwidth]{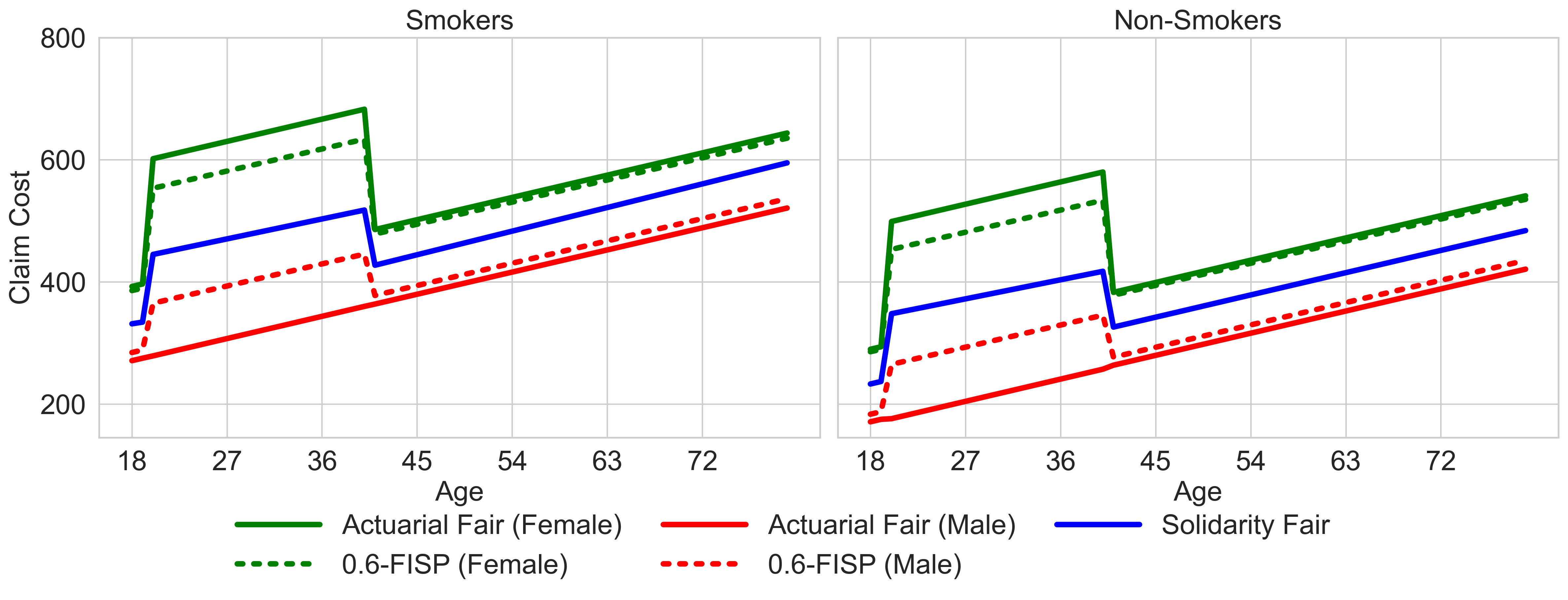}
        \caption{Condition satisfied ($P_L = 0.5 \le \frac{1}{1 + \alpha}$)}
        \label{subfig:toy_alpha0.6_gt_s}
    \end{subfigure}
    \begin{subfigure}{0.5\textwidth}
        \centering
        \includegraphics[width=\textwidth]{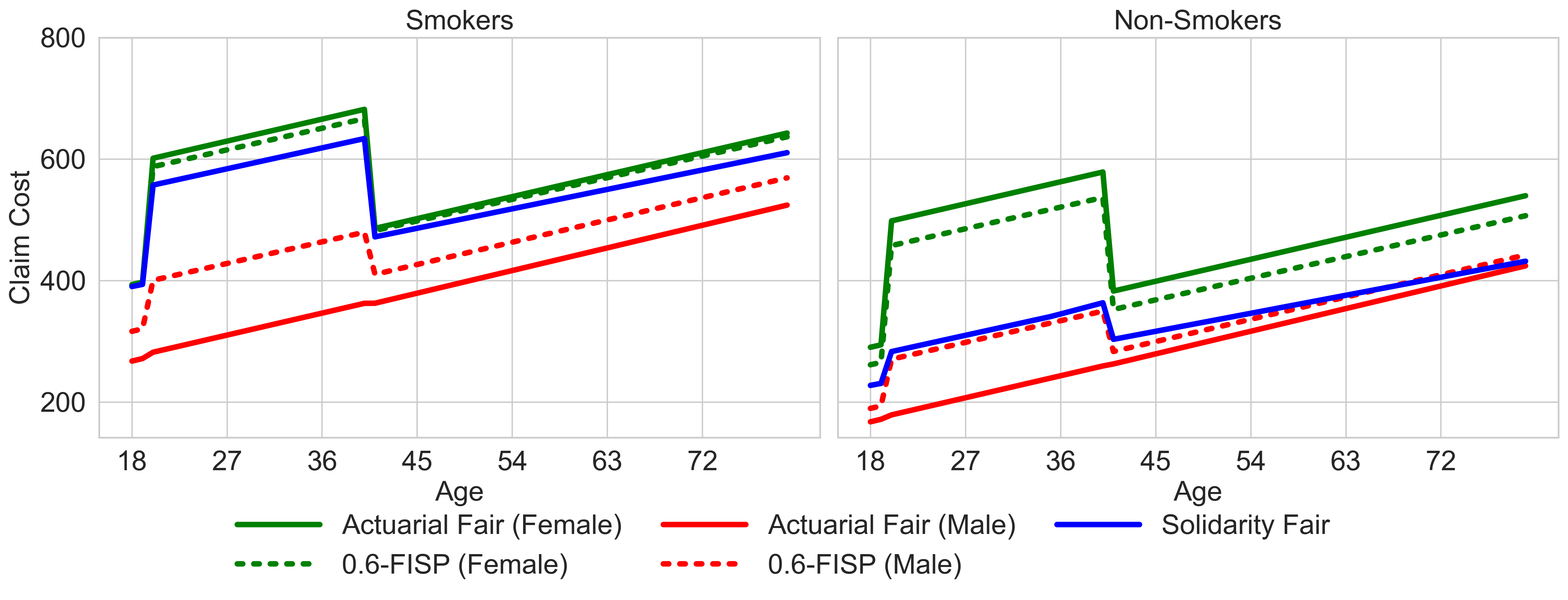}
        \caption{Condition violated ($P_L = 0.7 > \frac{1}{1 + \alpha}$)}
        \label{subfig:toy_alpha0.6_gt_v}
    \end{subfigure}
     \caption{Order preservation of $\alpha$-fair premiums for $|\cD| = 2$ and $\alpha = 0.6$}
    \label{fig:toy_gt_alpha_0.6}
\end{figure} 

\begin{figure}[H]
    \begin{subfigure}{0.5\textwidth}
        \centering
        \includegraphics[width=\textwidth]{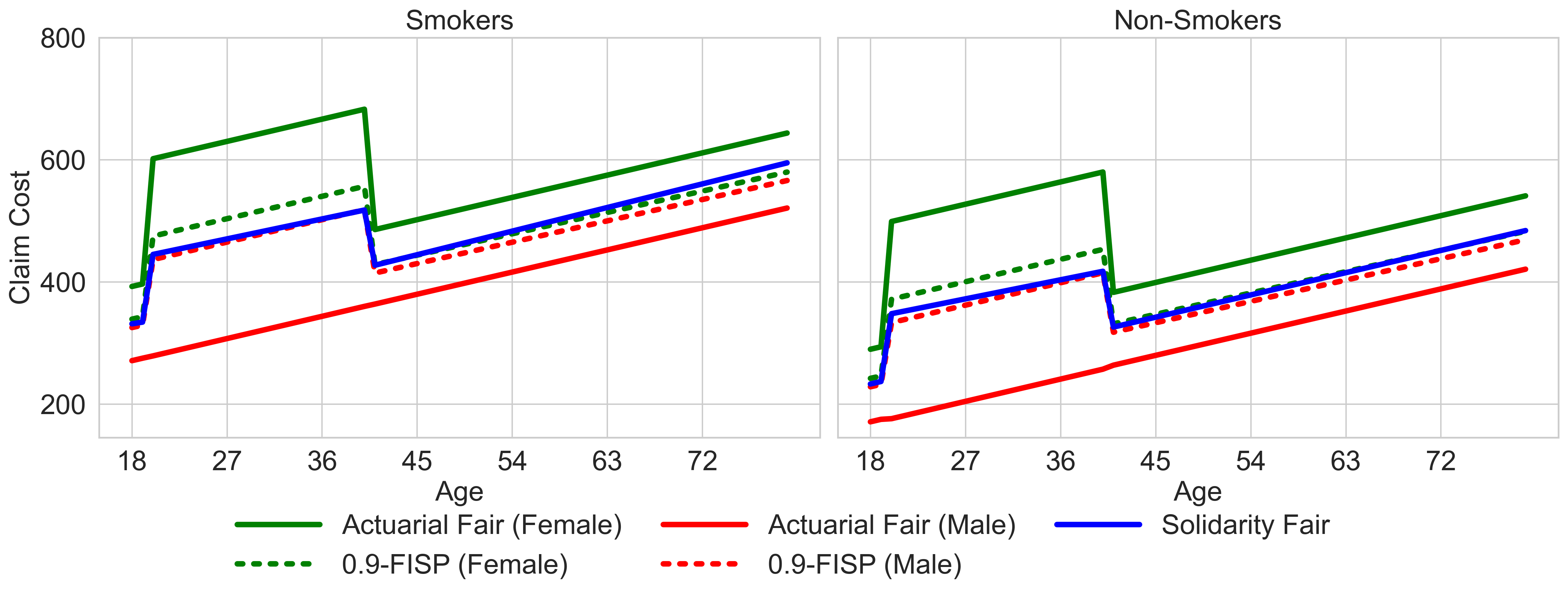}
        \caption{Condition satisfied ($P_L = 0.5 \le \frac{1}{1 + \alpha}$)}
        \label{subfig:toy_alpha0.9_gt_s}
    \end{subfigure}
    \begin{subfigure}{0.5\textwidth}
        \centering
        \includegraphics[width=\textwidth]{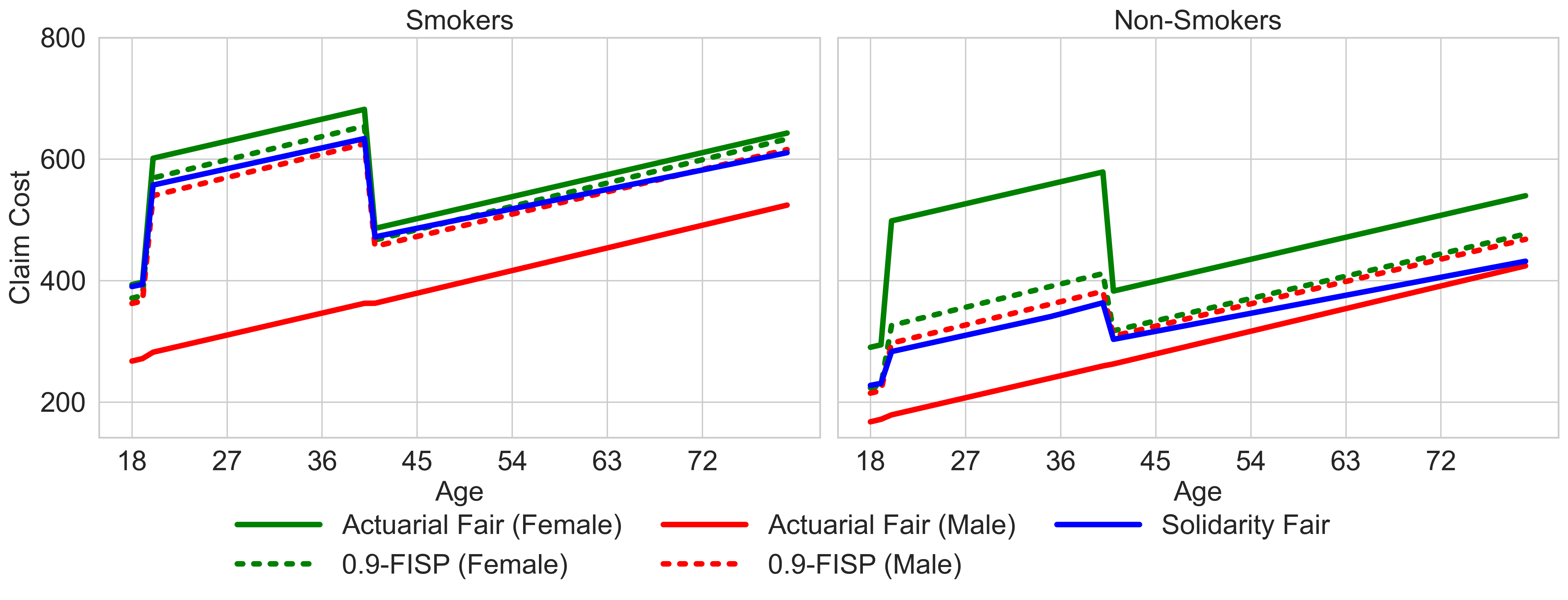}
        \caption{Condition violated ($P_L = 0.7 > \frac{1}{1 + \alpha}$)}
        \label{subfig:toy_alpha0.9_gt_v}
    \end{subfigure}
     \caption{Order preservation of $\alpha$-fair premiums for $|\cD| = 2$ and $\alpha = 0.9$}
    \label{fig:toy_gt_alpha_0.9}
\end{figure} 

\newpage
\subsection{Simulation Figures (D > 2)}
\label{appendix:deferred-simulation-figures_D>2}

We then present $\alpha$-FISP under varying fairness budget $\alpha$ for datasets B1 and B2 (i.e. $|\cD| > 2$) in the following:

\begin{figure}[H]
    \begin{subfigure}{0.5\textwidth}
        \centering
        \includegraphics[width=\textwidth]{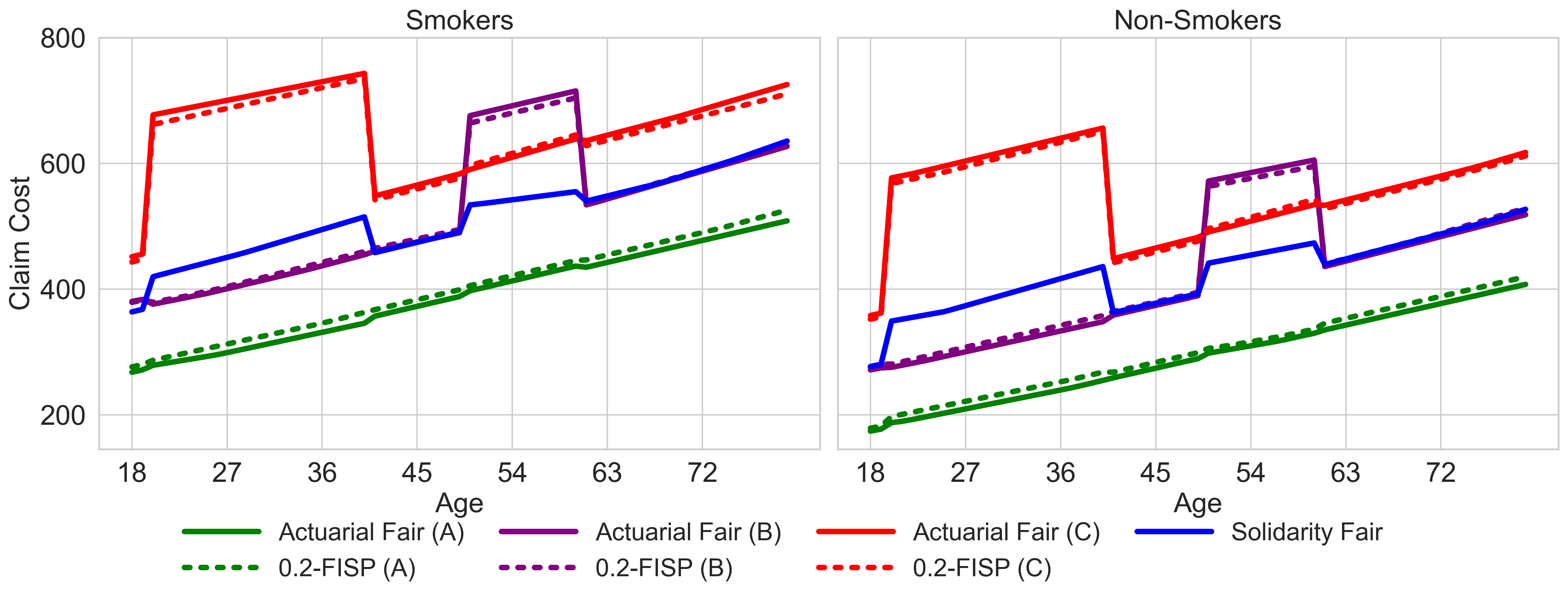}
        \caption{Condition satisfied ($P_L = \frac{1}{3} \le \frac{1}{1 + \alpha}$)}
        \label{subfig:toy3_alpha0.2_gt_s}
    \end{subfigure}
    \begin{subfigure}{0.5\textwidth}
        \centering
        \includegraphics[width=\textwidth]{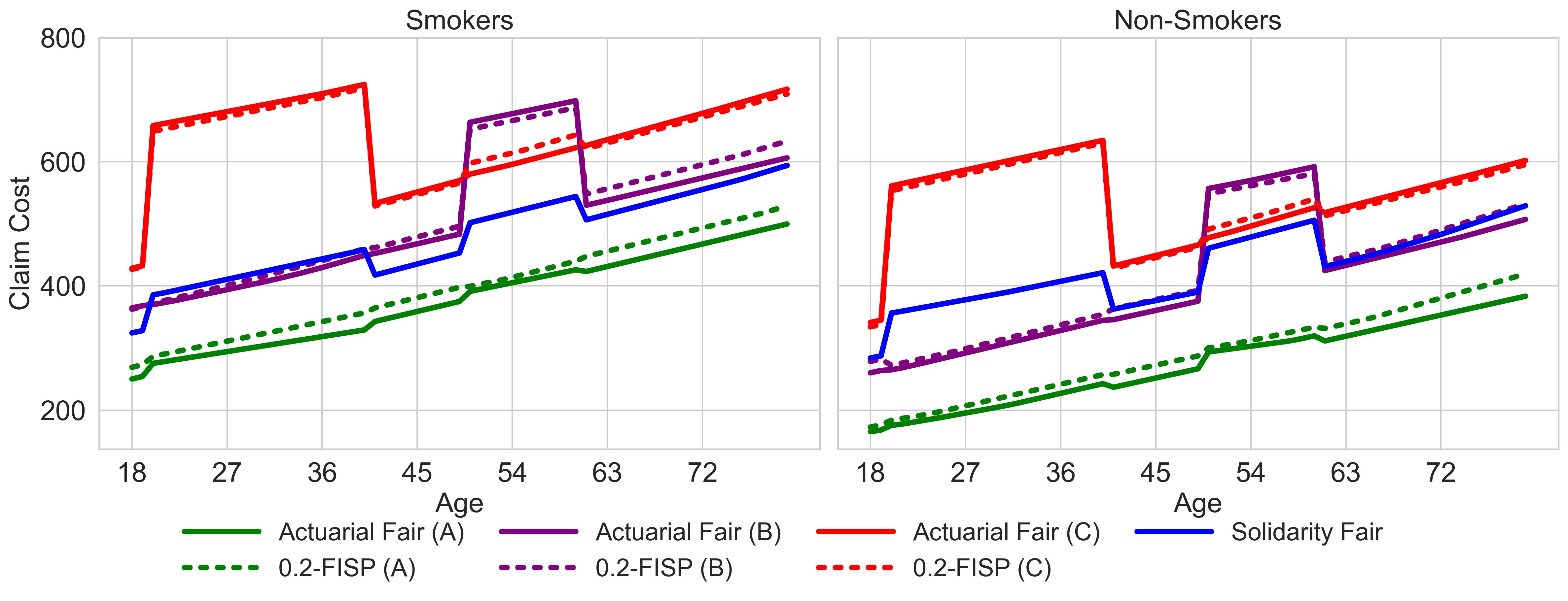}
        \caption{Condition satisfied ($P_L = 0.6 \le \frac{1}{1 + \alpha}$)}
        \label{subfig:toy3_alpha0.2_gt_v}
    \end{subfigure}
     \caption{Order preservation of $\alpha$-fair premiums for $|\cD| > 2$ and $\alpha = 0.2$}
    \label{fig:toy3_gt_alpha_0.2}
\end{figure} 

\begin{figure}[H]
    \begin{subfigure}{0.5\textwidth}
        \centering
        \includegraphics[width=\textwidth]{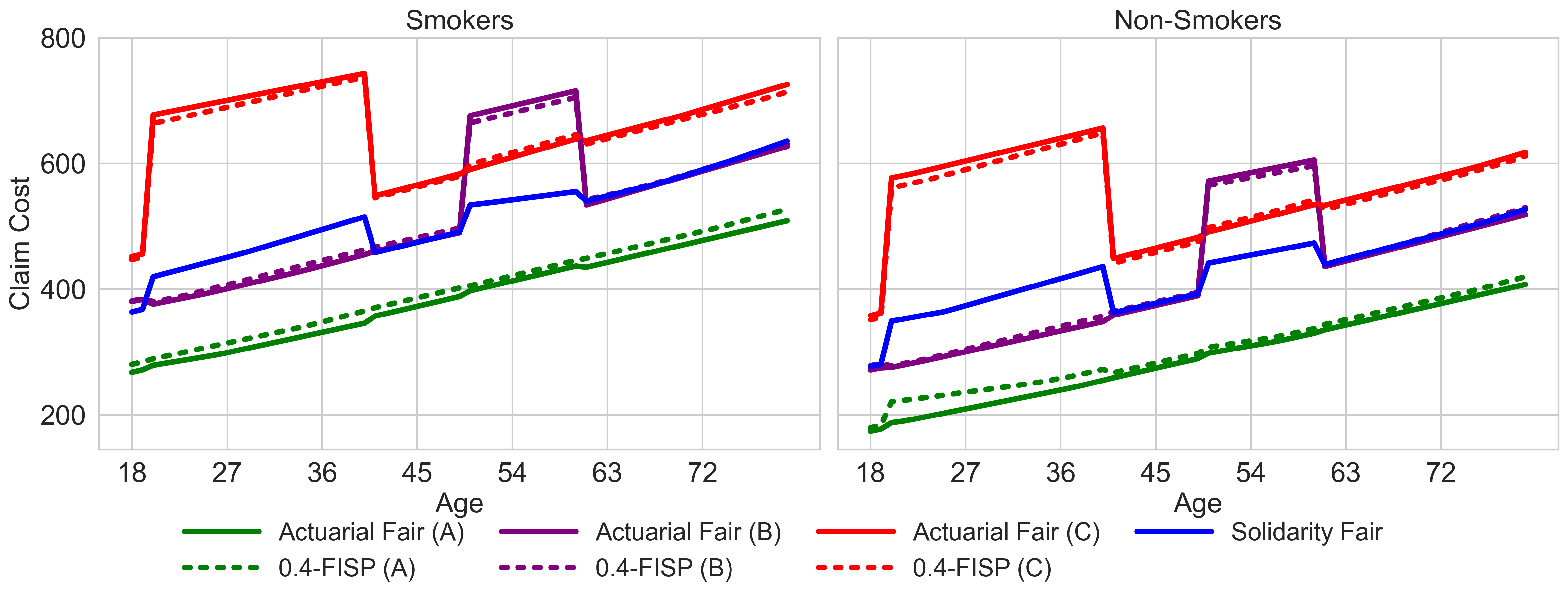}
        \caption{Condition satisfied ($P_L = \frac{1}{3} \le \frac{1}{1 + \alpha}$)}
        \label{subfig:toy3_alpha0.4_gt_s}
    \end{subfigure}
    \begin{subfigure}{0.5\textwidth}
        \centering
        \includegraphics[width=\textwidth]{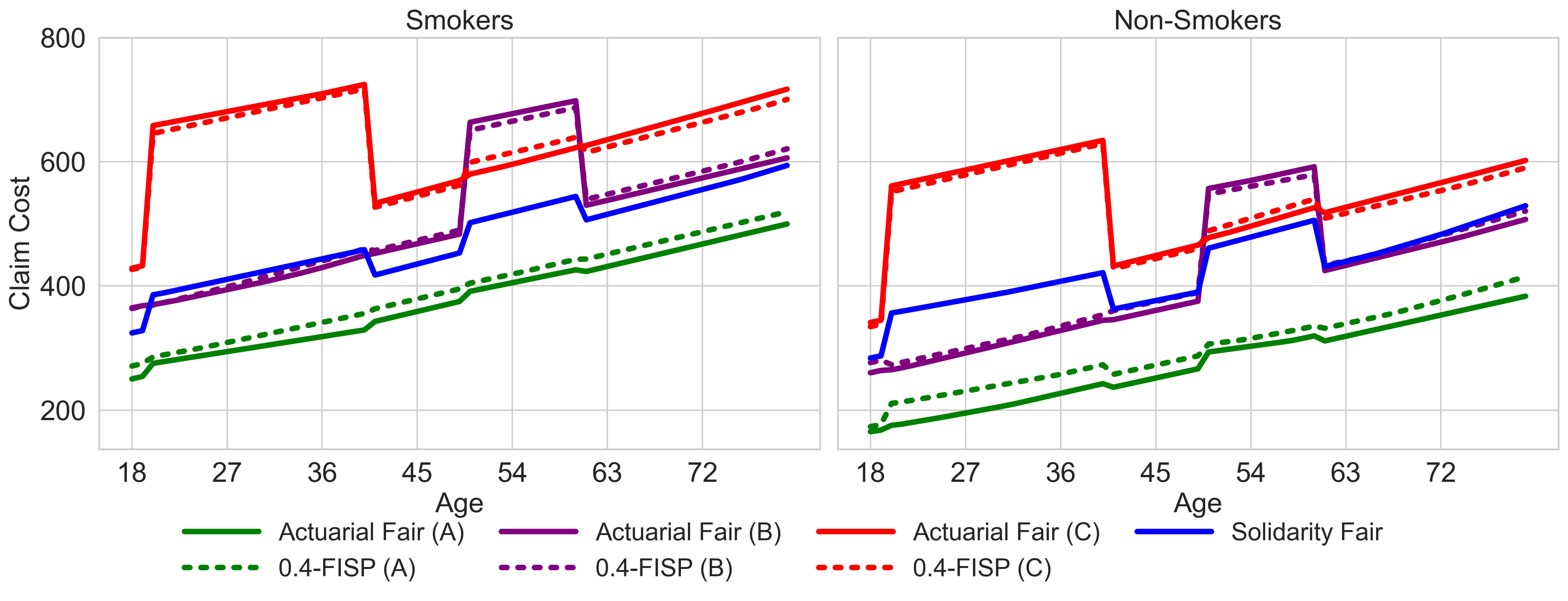}
        \caption{Condition satisfied ($P_L = 0.6 \le \frac{1}{1 + \alpha}$)}
        \label{subfig:toy3_alpha0.4_gt_v}
    \end{subfigure}
     \caption{Order preservation of $\alpha$-fair premiums for $|\cD| > 2$ and $\alpha = 0.4$}
    \label{fig:toy3_gt_alpha_0.4}
\end{figure} 

\begin{figure}[H]
    \begin{subfigure}{0.5\textwidth}
        \centering
        \includegraphics[width=\textwidth]{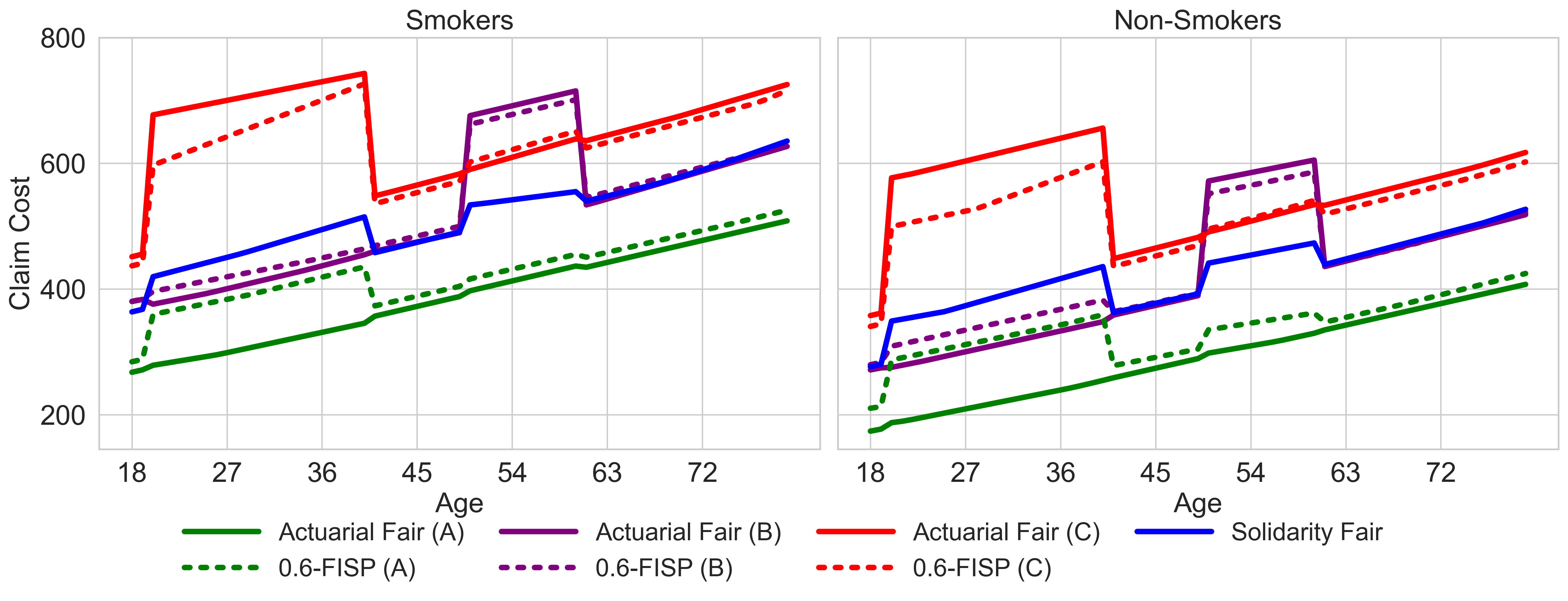}
        \caption{Condition satisfied ($P_L = \frac{1}{3} \le \frac{1}{1 + \alpha}$)}
        \label{subfig:toy3_alpha0.6_gt_s}
    \end{subfigure}
    \begin{subfigure}{0.5\textwidth}
        \centering
        \includegraphics[width=\textwidth]{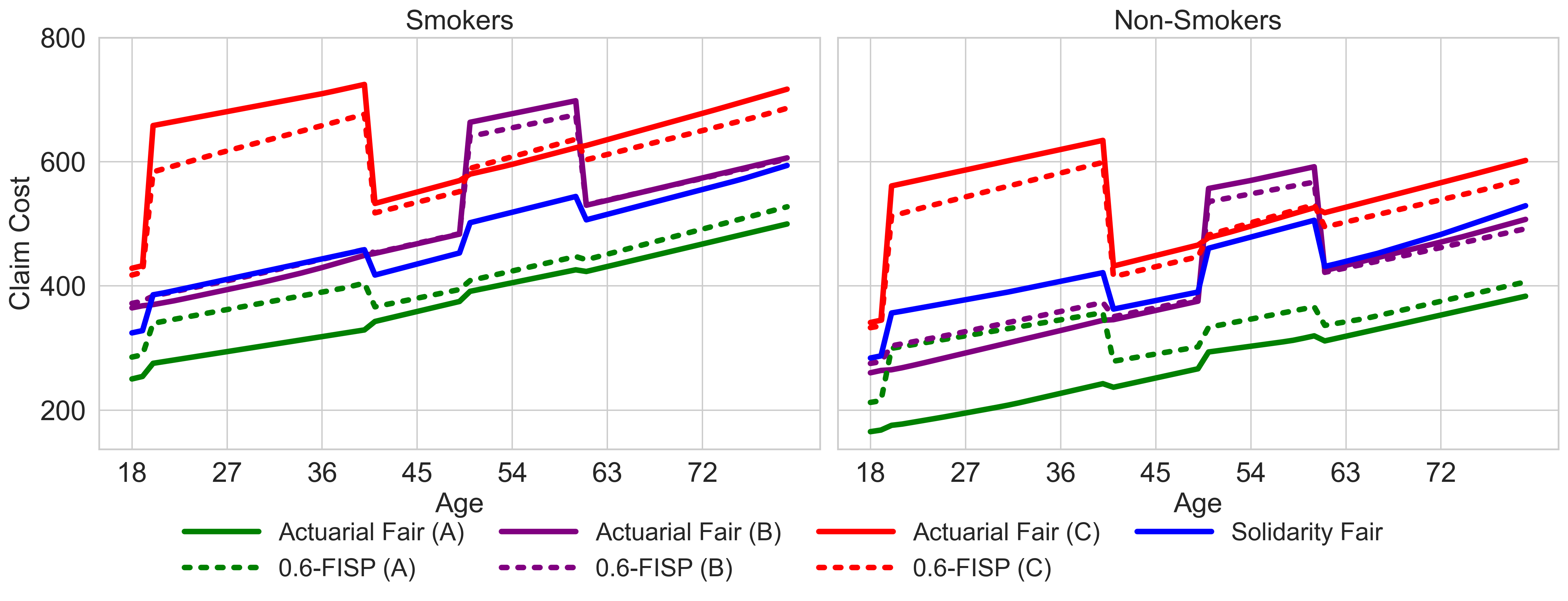}
        \caption{Condition satisfied ($P_L = 0.6 \le \frac{1}{1 + \alpha}$)}
        \label{subfig:toy3_alpha0.6_gt_v}
    \end{subfigure}
     \caption{Order preservation of $\alpha$-fair premiums for $|\cD| > 2$ and $\alpha = 0.6$}
    \label{fig:toy3_gt_alpha_0.6}
\end{figure} 

\begin{figure}[H]
    \begin{subfigure}{0.5\textwidth}
        \centering
        \includegraphics[width=\textwidth]{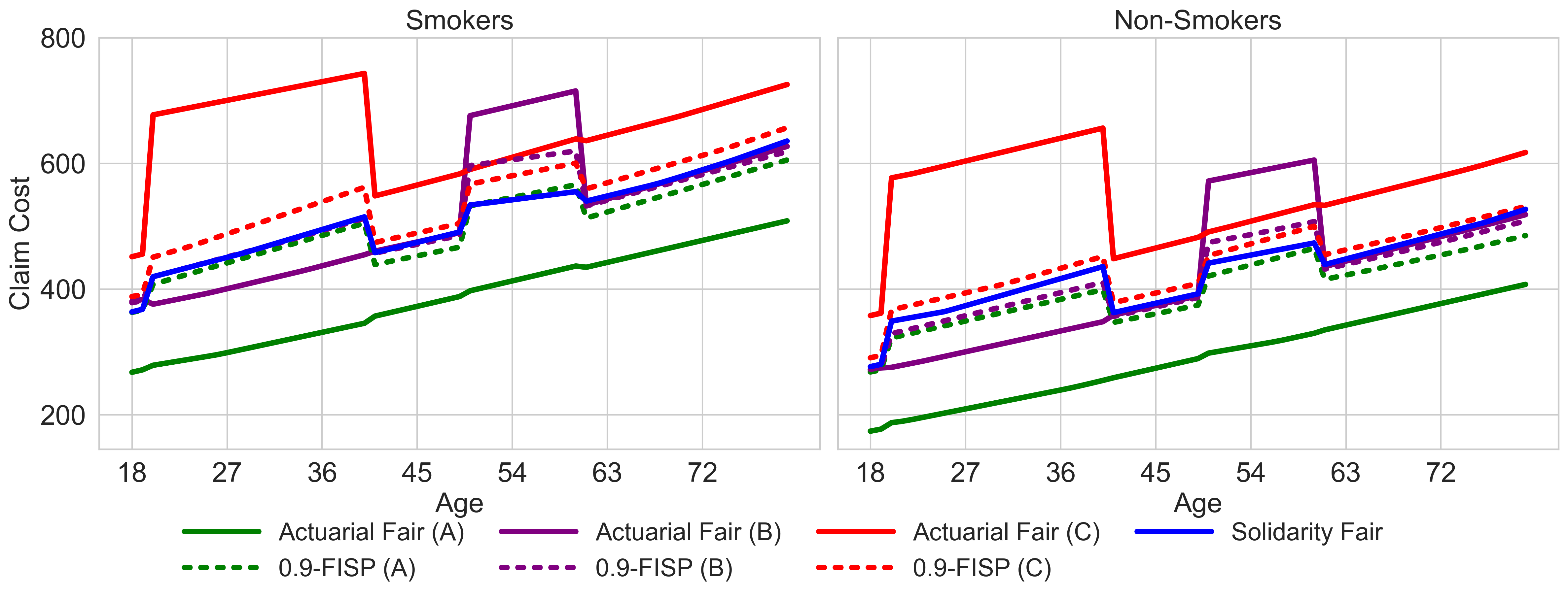}
        \caption{Condition satisfied ($P_L = \frac{1}{3} \le \frac{1}{1 + \alpha}$)}
        \label{subfig:toy3_alpha0.9_gt_s}
    \end{subfigure}
    \begin{subfigure}{0.5\textwidth}
        \centering
        \includegraphics[width=\textwidth]{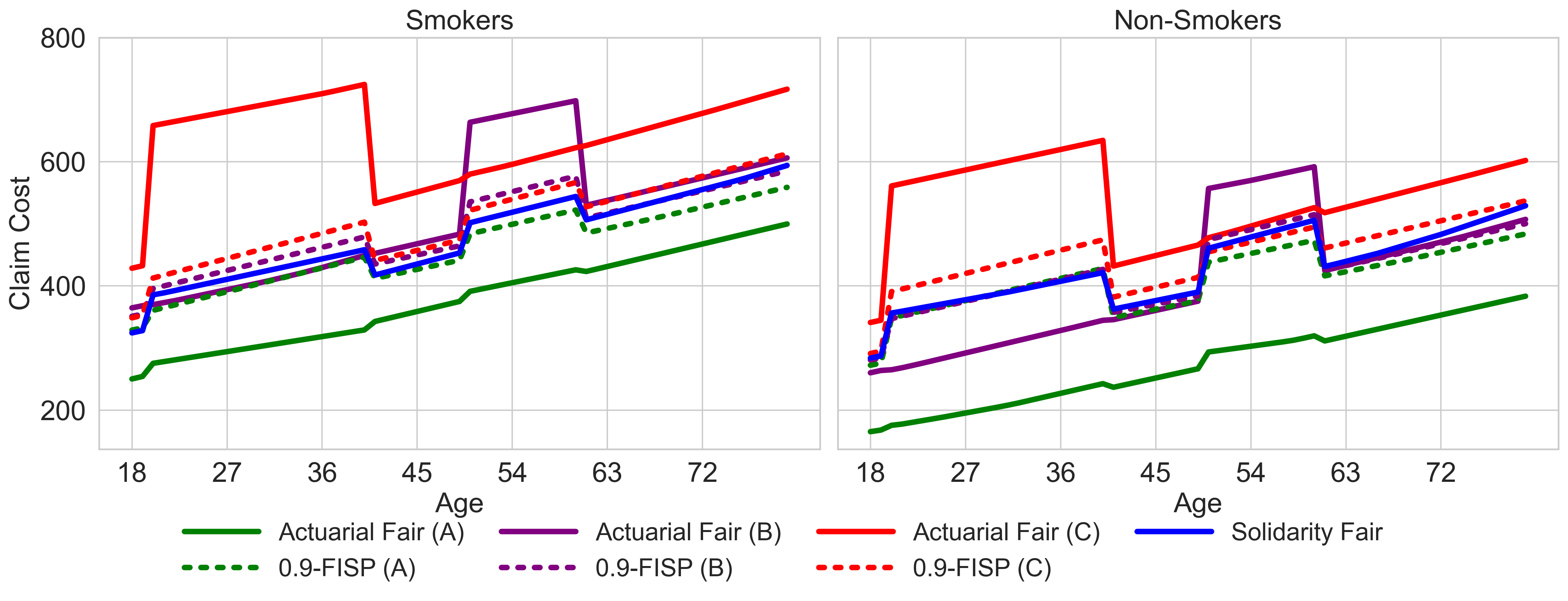}
        \caption{Condition violated ($P_L = 0.6 > \frac{1}{1 + \alpha}$)}
        \label{subfig:toy3_alpha0.9_gt_v}
    \end{subfigure}
     \caption{Order preservation of $\alpha$-fair premiums for $|\cD| > 2$ and $\alpha = 0.9$}
    \label{fig:toy3_gt_alpha_0.9}
\end{figure} 

\newpage
\subsection{Health Insurance Figures}
\label{appendix:deferred-health-insurance-figures}
\begin{figure}[H]
    \begin{subfigure}{0.50\textwidth}
        \centering
        \includegraphics[width=\textwidth]{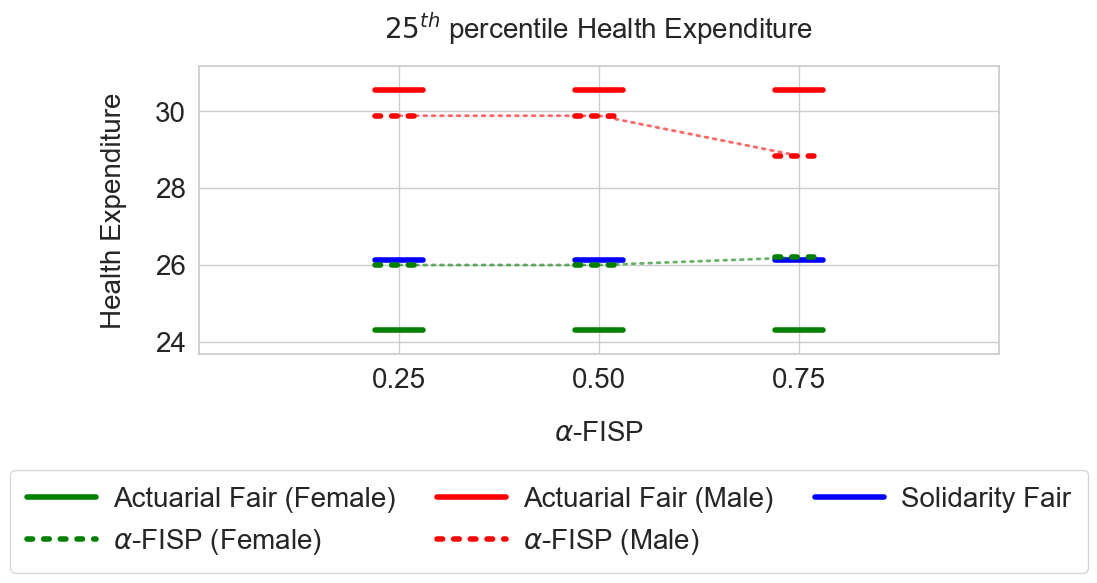}
        \label{subfig:Figures/Health_alpha_sensitivity_25th_segments}
    \end{subfigure}
    \begin{subfigure}{0.50\textwidth}
        \centering
        \includegraphics[width=\textwidth]{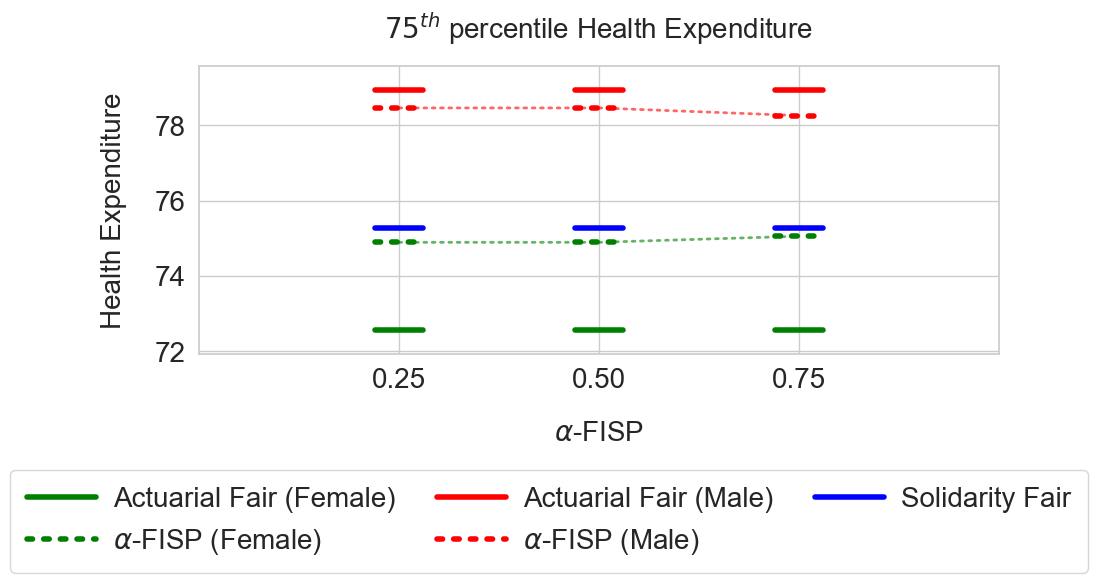}
        \label{subfig:Figures/Health_alpha_sensitivity_75th_segments}
    \end{subfigure}
     \caption{$\alpha$-fair premiums for health expenditure risk classes at the 25th and 75th percentiles. The panels illustrate the convergence of $\alpha$-FISP premiums (dashed lines) to the Solidarity Fair benchmark (blue solid line) as the fairness constraint tightens.}
    \label{fig:Health_alpha_sensitivity_panel}
\end{figure} 

\newpage
\section{Appendix C: Deferred Proofs}
\label{appendix:deferred-proofs}

\subsection{Proposition \ref{proposition:existence-of-solution} Proof}

\begin{proof}
\label{proof:existence-of-solution}

We first show that the feasible set is non-empty. For $x \in \cX$, define the candidate function:
\[
f_0(x,d_k) = \mu(x), \ k= 1, \ldots, |\cD|.
\]
Since $\mu(x) = \sum_{k=1}^{|\cD|} \mu_k(x) \cdot p_{k|x}$. Therefore, $f_0$ trivially satisfies IS and $\alpha$-fairness constraints. Thus, $f_0$ is feasible. Moreover, we know $f_0 \in \cF$, since
\begin{align*}
    \Ex \left[ (f_0(X,D)^2 \right] &= \Ex \left[ \mu(X)^2 \right] \\
    &= \Ex \left[ \left( \Ex \left[Y \mid X \right] \right)^2 \right] \\
    &\le \Ex \left[ \Ex \left[ Y^2 \mid X \right] \right] \\
    &= \Ex \left[ Y^2 \right] \\
    &< +\infty.
\end{align*}
Therefore, we have $f_0 \in \cF_{\alpha , IS}$ and $\cF_{\alpha, IS} \ne \emptyset$.

Next, we show that the feasible set $\cF_{\alpha,IS}$ is convex. Let $f,g \in \cF_{\alpha,IS}$ and $t \in [0,1]$. Define $h:=tf + (1-t)g$, then we have for $a.e.$ $x$:
\begin{align*}
    \sum_{k=1}^{|\cD|} \left( h(x,d_k) - \mu_k(x) \right) \cdot p_{k|x} &= t\sum_{k=1}^{|\cD|} \left( f(x,d_k) - \mu_k(x) \right) \cdot p_{k|x} + (1-t)\sum_{k=1}^{|\cD|} \left( g(x,d_k) - \mu_k(x) \right) \\
    & \ge 0.
\end{align*}
Also, for any $k<k'$ and $a.e.$ $x$, we have:
\begin{align*}
    h(x,d_k) - \alpha h(x,d_{k'}) &= t \left( f(x,d_k) - \alpha f(x,d_{k'}) \right) + (1-t) \left( g(x,d_k) - \alpha g(x,d_{k'}) \right) \\
    &\ge 0.
\end{align*}
Therefore, $h \in \cF_{\alpha, IS}$ meaning $\cF_{\alpha,IS}$ is convex.

Then, we show that $\cF_{\alpha,IS}$ is closed in $L^2(\bbP_{X,D})$. 

Let $\{f_n\}_{n \ge 1} \subset \cF_{\alpha,IS}$ s.t. $f_n \to f$ in $\cF = L^2(\bbP_{X,D})$. We know that $L^P$ convergence implies convergence in probability. Thus, we can find a subsequence $\{f_{n_j}\}_{j \ge 1}$ s.t. $f_{n_j} \overset{a.s.}{\to} f$ w.r.t. the probability measure $\bbP_{X,D}$. That is, there is some $\Omega_0 \subset \cX \times \cD$ and $\bbP_{X,D}(\Omega_0) = 1$ s.t.
\[
f_{n_j}(x,d) \to f(x,d), \forall (x,d) \in \Omega_0.
\]
But, note that IS and $\alpha$-fairness constraints are evaluated pointwise. We need to leverage the $a.s.$ convergence to ensure for $\bbP_X$-$a.e.$ $x$, the limit function $f$ also satisfies the constraints.

Now, fix $k \in [|\cD|]$. Define
\[
A_k := \left\{ x \in \cX \mid f_{n_j}(x,d_k) \to f(x,d_k) \right\}.
\]
Since we know $f_{n_j}(x,d) \overset{a.s.}{\to} f(x,d)$, therefore, we have
\[
\bbP_{X,D}\left( \left\{ (x,d) \mid x \in A_k^c \right\} \right) = 0.
\]
By conditioning on $X$, we get:
\[
\bbP_{X,D}\left( \left\{ (x,d) \mid x \in A_k^c \right\} \right) = \int \ones_{A_k^c}(x) \cdot p_{k|x} d\bbP_X(x) = 0.
\]
But, since $p_{k|x} > 0$, then this forces 
\[
\bbP_{X}(A_k^c) = 0.
\]
Therefore, for each $k \in [|\cD|]$, we have
\[
\bbP_X(A_k) = 1.
\]
Now, define
\[
A:= \bigcap_{k=1}^{|\cD|} A_k,
\]
then, this gives
\[
\bbP_X(A) = 1,
\]
as $A$ is a finite intersection of probability-1 sets.

Thus, we have shown that $\forall x \in A$,
\[
f_{n_j}(x,d_k) \to f(x,d_k), \forall k \in [|\cD|].
\]

Next, fix a pair $k < k'$. Then, for each subsequence index $j$, define
\[
B_j^{k,k'}:= \left\{ x \mid f_{n_j}(x,d_k) - \alpha f_{n_j}(x,d_{k'}) \ge 0 \ \text{and} \ f_{n_j}(x,d_{k'}) - \alpha f_{n_j}(x,d_k) \ge 0 \right\}.
\]
We know $f_{n_j}$ is feasible, therefore, we have
\[
\bbP_X \left( B_j^{k,k'} \right) = 1.
\]
Now, define 
\[
B^{k,k'} := \bigcap_{j=1}^{\infty} B_j^{k,k'},
\]
then, this gives
\[
\bbP_X \left( B^{k,k'} \right) = 1,
\]
as $B^{k,k'}$ is a countable intersection of probability-1 sets.

Now, we further define
\[
B^{\alpha-\text{fair}} := \bigcap_{k < k'} B^{k,k'},
\]
then, this again gives
\[
\bbP_X \left( B^{\alpha-\text{fair}} \right) = 1,
\]
since $|\cD| < \infty$, there are only finitely many pairs of $(k,k')$.

Similarly, for each $j$, define
\[
B_j^{IS} := \left\{ x \in \sum_{k=1}^{|\cD|} \left( f_{n_j}(x,d_k) - \mu_k(x) \right) \cdot p_{k|x} \ge 0 \right\},
\]
since $f_{n_j}$ is feasible, we have
\[
\bbP_X \left( B_j^{IS} \right) = 1.
\]
Define 
\[
B^{IS} := \bigcap_{j=1}^{\infty} B_j^{IS},
\]
we have 
\[
\bbP_X \left( B^{IS} \right) = 1.
\]
Finally, define
\[
B:= B^{\alpha-\text{fair}} \cap B^{IS},
\]
we get:
\[
\bbP_X(B) = 1.
\]
Thus, we have shown that $\forall x \in B$, all constraints hold for each $f_{n_j}$.

Note that since $\bbP_X(A) = \bbP_X(B) = 1$, thus, we have $\bbP_X(A \cap B) = 1$. Now, we are ready to take the limit.

Pick any $x \in A \cap B$. For each $j$, we have:
\[
f_{n_j}(x,d_k) - \alpha f_{n_j}(x,d_k) \ge 0,
\]
take $j \to \infty$, we have:
\[
f(x,d_k) - \alpha f(x,d_{k}) = \lim_{j \to \infty} \left( f_{n_j}(x,d_k) - \alpha f_{n_j}(x,d_k) \right) \ge 0.
\]
The same argument will show that the $\alpha$-fairness constraint is satisfied for the reverse way.

Pick any $x \in A \cap B$. For each $j$, we have:
\[
\sum_{k=1}^{|\cD|} \left( f_{n_j}(x,d_k) - \mu_k(x) \right) \cdot p_{k|x} \ge 0.
\]
Since $|\cD| < \infty$, and $f_{n_j}(x,d_k) \to f(x,d_k)$. By letting $j \to \infty$, we get:
\[
\lim_{j \to \infty} \left( \sum_{k=1}^{|\cD|} \left( f_{n_j}(x,d_k) - \mu_k(x) \right) \cdot p_{k|x} \right) = \sum_{k=1}^{|\cD|} \left( f(x,d_k) - \mu_k(x) \right) \cdot p_{k|x} \ge 0.
\]
Therefore, we have proved that $f \in \cF_{\alpha,IS}$ and this implies that $\cF_{\alpha, IS}$ is closed in $\cF$. $\cF_{\alpha,IS}$ is also weakly closed since it is also convex.

Next, define the objective
\begin{align*}
    \psi(f) &:= \Ex \left[ \left(f(X,D) - Y \right)^2 \right] \\
    &\ge \frac{1}{2} \Ex \left[ f(X,D)^2 \right] - \Ex \left[ Y^2 \right] \\
    &= \frac{1}{2} \|f\|_2^2 - \Ex \left[ Y^2 \right].
\end{align*}
Since $\Ex \left[ Y^2 \right] < \infty$, we have $\psi(f) \to \infty$ whenever $\|f\|_2^2 \to \infty$. Thus, $\psi$ is coercive.

To show $\psi$ is weakly LSC, expand the square to get:
\begin{align*}
    \psi(f) &= \Ex \left[ \left(f(X,D) - Y \right)^2 \right] \\
    &= \Ex \left[ f(X,D)^2 \right] - 2 \Ex \left[ f(X,D)Y \right] + \Ex \left[ Y^2 \right] \\
    &= \Ex \left[ f(X,D)^2 \right] - 2 \Ex \left[ f(X,D) \Ex \left[ Y \mid X,D \right] \right] + \Ex \left[ Y^2 \right] \\
    &= \Ex \left[ f(X,D)^2 \right] - 2 \Ex \left[ f(X,D)\mu_D(X) \right] + \Ex \left[ Y^2 \right].
\end{align*}
But, note that, by Jensen's inequality, we have
\[
\Ex \left[ \mu_D(X)^2 \right] \le \Ex \left[ Y^2 \right] < \infty,
\]
meaning $\mu_D(X) \in L^2(\bbP_{X,D})$. Thus,
\[
\psi(f) = \|f\|_2^2 - 2 \langle f, \mu_D \rangle + \Ex \left[ Y^2 \right].
\]
Since $\|\cdot\|_2^2$ is weakly LSC on $L^2(\bbP_{X,D})$ and $\langle \cdot, \mu_D \rangle$ is a bounded linear operator (inner product operator).
Therefore, $\psi$ is also weakly LSC.

Now, let 
\[
m:= \inf_{f \in \cF_{\alpha,IS}} \psi(f).
\]
Since $\cF_{\alpha, IS} \ne \emptyset$ and $m < \infty$. Take $\{f_n\}_{n \ge 1} \subset \cF_{\alpha,IS}$ with $\psi(f_n) \to m$. Then, by coercivity, $\{f_n\}_{n \ge 1}$ is bounded in $\cF$. Therefore, we can find a weakly convergent subsequence $f_{n_j} \rightharpoonup f^*$ in $\cF$. But, we already proved that $\cF_{\alpha,IS}$ is weakly closed. Thus, $f^* \in \cF_{\alpha, IS}$. By weakly LSC, we have:
\[
\psi(f^*) \le \liminf_{j \to \infty} \psi(f_{n_j}) = m.
\]
Therefore $\psi(f^*) = m$, hence $f^*$ is a minimizer.

To show $f^*$ is the unique minimizer, let $f,g \in \cF$ and $t \in (0,1)$, then apply the identity with $a,b \in \bbR$, $t \in (0,1)$:
\[
\left( ta + (1-t)b \right)^2 = ta^2 + (1-t)b^2 - t(1-t)(a-b)^2,
\]
we get:
\[
\psi(tf + (1-t)g) = t\psi(f) + (1-t)\psi(g) - t(1-t) \|f-g\|_2^2.
\]
If $f \ne g$ in $L^2(\bbP_{X,D})$, then $\|f - g\|_2^2 > 0$, resulting in
\[
\psi(tf + (1-t)g) < t\psi(f) + (1-t)\psi(g).
\]
This implies that $\psi$ is strictly convex.

Now, let $f^*, g^*$ be two distinct minimizers. Take $t = \frac{1}{2}$ and define
\[
h := \frac{1}{2} f^* + \frac{1}{2} g^* \in \cF_{\alpha, IS}.
\]
Since we have proven that $\cF_{\alpha,IS}$ is convex. Then, we have
\begin{align*}
    \psi(h) &= \frac{1}{2}\psi(f^*) + \frac{1}{2}\psi(g^*) - \frac{1}{4}\|f^* - g^*\|_2^2 \\
    &= m - \frac{1}{4}\|f^* - g^*\|_2^2 \\
    &< m,
\end{align*}
contradicting the optimality of $f^*, g^*$, and the infimum $m$ over $F_{\alpha,IS}$.

Therefore, there exists a minimizer $f^*$ up to $\bbP_{X,D}$-$a.s.$ equality.
\end{proof}

\newpage
\subsection{Theorem \ref{theorem:sample-complexity-upper-bound-fairness-and-IS} Proof}

\begin{proof}
\label{proof:sample-complexity-upper-bound-fairness-and-IS}
\textbf{Step 1: upper bounding the individual solvency (IS) deviation}

Note that we can write the population IS functional as follows
\begin{align*}
    G(f,t) &:= \Ex \left[ (f(X,D) - Y) t(X) \right] \\
    &= \Ex \left[ \left( \sum_{k=1}^{|\cD|} (f(X,d_k) - \mu_k(X)) \cdot p_{k|x} \right) t(X) \right].
\end{align*}
Consequently, we get the plug-in empirical IS functional as 
\[
\hat{G}_n(f,t) := \frac{1}{n} \sum_{i=1}^n \left( \sum_{k=1}^{|\cD|} (f(X_i,d_k) - \hat{\mu}_k(X_i)) \cdot \hat{p}_{k|X_i} \right) t(X_i),
\]
and the empirical IS functional with true $\mu_k$ and $p_{k|\cdot}$ as
\[
G_n(f,t) := \frac{1}{n} \sum_{i=1}^n \left( \sum_{k=1}^{|\cD|} (f(X_i,d_k) - \mu_k(X_i)) \cdot p_{k|X_i} \right) t(X_i),
\]

For any $f \in \cF_\alpha$, define
\begin{align*}
    IS_f(x) &:= \sum_{k=1}^{|\cD|} (f(x,d_k) - \mu_k(x)) \cdot p_{k|x}, \\
    \hat{IS}_f(x) &:= \sum_{k=1}^{|\cD|} (f(x,d_k) - \hat{\mu}_k(x)) \cdot \hat{p}_{k|x}.
\end{align*}
Then, we can decompose the error as follows:
\[
G(f,t) - \hat{G}_n(f,t) = \left( G(f,t) - G_n(f,t) \right) + \left( G_n(f,t) - \hat{G}_n(f,t) \right).
\]
Then, we will bound term by term. Starting with $|G(f,t) - G_n(f,t)|$.

Fix $x$. Then, we have
\[
\hat{IS}_f(x) - IS_f(x) = \sum_{k=1}^{|\cD|} \left( f(x,d_k) - \mu_k(x) \right) \left( \hat{p}_{k|x} - p_{k|x} \right) + \sum_{k=1}^{|\cD|} \left( \mu_k(x) - \hat{\mu}_k(x) \right) \cdot \hat{p}_{k|x}.
\]
Taking absolute value, we get:
\begin{align*}
    \left| \hat{IS}_f(x) - IS_f(x) \right| &\le (B_f + B) \sum_{k=1}^{|\cD|} \left| \hat{p}_{k|x} - p_{k|x} \right| + \sum_{k=1}^{|\cD|} \left| \mu_k(x) - \hat{\mu}_k(x) \right| \cdot \hat{p}_{k|x} \\
    &\le M \epsilon_p + \epsilon_\mu, \forall x.
\end{align*}
Since $|f(,d_k) - \mu_k(x)| \le B_f + B = M$ and $\sum_{k=1}^{|\cD|} \hat{p}_{k|x} = 1$. Therefore, with $t(x) \in [0,1]$, we have:
\begin{align*}
    \left| G(f,t) - G_n(f,t) \right| &= \left| \Ex \left[ \left( IS_f(X) - \hat{IS}_f(X) \right) t(X) \right] \right| \\
    &\le \Ex \left[ \left| IS_f(X) - \hat{IS}_f(X) \right| \right] \\
    &\le M \epsilon_p + \epsilon_\mu.
\end{align*}
Thus, we get:
\[
\sup_{f \in \cF_\alpha, t \in \cT} \left| G(f,t) - G_n(f,t) \right| \le M \epsilon_p + \epsilon_\mu.
\]

Next, we bound $\left| G_n(f,t) - \hat{G}_n(f,t) \right|$.

Define the plug-in constrained class:
\[
\hat{\cH} := \left\{ \hat{h}_{f,t}(x) := \hat{IS}_f(x) t(x) \mid f \in \cF_\alpha, t \in \cT \right\}.
\]
Then, we have:
\[
\sup_{f,t} \left| G_n(f,t) - \hat{G}_n(f,t) \right| = \sup_{\hat{h} \in \hat{\cH}} \left| P_X \hat{h} - P_{n,X} \hat{h} \right|,
\]
where $P_X$ is the expectation over $X$ and $P_{n,X}$ is the empirical average over $X_1, \ldots, X_n$.

For any $x$, we have
\begin{align*}
    \left| \hat{IS}_f(x) \right| &= \left| \sum_{k=1}^{|\cD|} \left( f(x,d_k) - \hat{\mu}_k(x) \right) \cdot \hat{p}_{k|x} \right| \\
    &\le \sum_{k=1}^{|\cD|} \left( |f(x,d_k) | + | \hat{\mu}_k(x) | \right) \cdot \hat{p}_{k|x} \\
    &\le B_f + B \\
    &= M.
\end{align*}
Therefore, 
\[
\left| \hat{h}_{f,t}(x) \right| \le M.
\]
Now, we bound the covering number of $\hat{\cH}$.

Let $\{f_j\}_{j=1}^{N_F}$ be an $\epsilon$-net of $\cF_\alpha$ in $L^2(\bbP_{X,D})$ and $\{t_\ell\}_{\ell = 1}^{N_T}$ be an $\epsilon$-net of $\cT$ in $L^2(\bbP_X)$. For any $(f,t)$, pick $f_j, t_\ell$ close. Then, we can write:
\[
\hat{h}_{f,t} - \hat{h}_{f_j, t_\ell} = \left( \hat{IS}_f - \hat{IS}_{f_j} \right) t + (t - t_\ell) \hat{IS}_{f_j}.
\]
Squaring, we get:
\[
\left( \hat{h}_{f,t} - \hat{h}_{f_j,t_\ell} \right)^2 \le 2 \left( \hat{IS}_f - \hat{IS}_{f_j} \right)^2 + 2(B_f + B)^2 (t - t_\ell)^2.
\]
Note that
\[
\hat{IS}_f(x) - \hat{IS}_{f_j} = \sum_{k=1}^{|\cD|} (f(x,d_k) - f_j(x,d_k)) \cdot \hat{p}_{k|x}.
\]

Let 
\[
0<p_{\min} := \inf_{x \in \mathcal{X}, k \in \mathcal{D}} p_{k|x},
\]
and then, taking expectation over $X$, we get:
\begin{align*}
    \left\| \hat{IS}_f - \hat{IS}_{f_j} \right\|_{L^2(\bbP_X)}^2 &\le \Ex_X \left[ \sum_{k=1}^{|\cD|} (f(X,d_k) - f_j(X,d_k))^2 \cdot \hat{p}_{k|x} \right] \\
    &\le \Ex_X \left[ \sum_{k=1}^{|\cD|} (f(X,d_k) - f_j(X,d_k))^2 \right] \\
    &\le \Ex_X \left[ \frac{1}{p_{\min}} \sum_{k=1}^{|\cD|} (f(X,d_k) - f_j(X,d_k))^2 \cdot p_{k|x} \right] \\
    &= \frac{1}{p_{\min}} \left\| f - f_j \right\|_{L^2(\bbP_{X,D})}^2.
\end{align*}
Therefore,
\[
\left( \hat{h}_{f,t} - \hat{h}_{f_j,t_\ell} \right)^2 \le \frac{2 \epsilon^2}{p_{\min}} + 2 M^2 \epsilon^2 = \epsilon^2 \left( \frac{2}{p_{\min}} + 2 M^2 \right),
\]
which immediately implies that
\[
\left\| \hat{h}_{f,t} - \hat{h}_{f_j, t_\ell} \right\|_{L^2(\bbP_X)} \le \epsilon \sqrt{\frac{2}{p_{\min}}+ 2M^2},
\]
meaning that the finite set
\[
\left\{ \hat{h}_{f_j, t_\ell} \mid j = 1, \ldots, N_F, \ell = 1, \ldots, N_T \right\}
\]
is an $\epsilon \sqrt{2/p_{\min} + 2M^2}$-net of $\cH$ w.r.t. $L^2(\bbP_X)$. Hence, we have:
\[
N \left(\epsilon \sqrt{\frac{2}{p_{\min}}+ 2M^2}, \cH, L^2(\bbP_X) \right) \le N \left(\epsilon, \cF_\alpha, \bbP_{X,D} \right) \times N \left( \epsilon, \cT, L^2(\bbP_X) \right).
\]

Let $\tilde{\epsilon} = \epsilon \sqrt{2/p_{\min} + 2M^2}$ and $C_0 = \sqrt{2/p_{\min} + 2M^2}$, we get:
\[
N \left( \tilde{\epsilon}, \cH, L^2(\bbP_X) \right) \le N \left(\frac{\tilde{\epsilon}}{C_0}, \cF_\alpha, L^2(\bbP_{X,D}) \right) \times N \left( \frac{\tilde{\epsilon}}{C_0}, \cT, L^2(\bbP_X) \right).
\]
Then, by Assumption \ref{assumption:covering-number}, we get:
\begin{align*}
    \log N(\tilde{\epsilon}, \hat{\cH}, L^2(\bbP_X)) &\le V_F \log \frac{A_F C_0}{\tilde{\epsilon}} + V_T \log \frac{A_T C_0}{\tilde{\epsilon}} \\
    &\le \left( V_F + V_T \right) \log \frac{A_H}{\tilde{\epsilon}},
\end{align*}
where $A_H > 0$, absorbing $A_F, A_T, C_0$.

Now, the standard symmetrization argument gives
\[
\Ex \left[ \sup_{\hat{h} \in \hat{\cH}} \left| P_X \hat{h} - P_{n,X} \hat{h} \right| \right] \le 2 \Ex \left[ \hat{\Re}_n (\hat{\cH}) \right]
\].
Then, by the Dudley's inequality, we have:
\[
\Ex \left[ \hat{\Re}_n(\hat{\cH}) \right] \le \frac{12}{\sqrt{n}} \int_0^M \sqrt{\log N(\epsilon, \hat{\cH}, L^2(\bbP_X))} \ d\epsilon.
\]
Plug in $ \log N(\tilde{\epsilon}, \hat{\cH}, L^2(\bbP_X)) \le \left( V_F + V_T \right) \log \frac{A_H}{\tilde{\epsilon}}$, we get
\begin{align*}
    \Ex \left[ \hat{\Re}_n(\hat{\cH}) \right] &\le \frac{12}{\sqrt{n}} \int_0^M \sqrt{(V_F + V_T) \log \frac{A_H}{\epsilon}} \ d \epsilon \\
    &= \frac{12 \sqrt{V_F + V_T}}{\sqrt{n}} \int_0^M \sqrt{\log \frac{A_H}{\epsilon}} \ d \epsilon.
\end{align*}
Let $I:= \int_0^M \sqrt{\log \frac{A_H}{\epsilon}} \ d \epsilon$ and substitute $\epsilon = Mu$, then
\begin{align*}
    I & = M \int_0^1 \sqrt{\log \frac{A_H}{Mu}} \ du \\
    &= M \int_0^1 \sqrt{\log \frac{A_H}{M} + \log \frac{1}{u}} \ du \\
    &\le M \int_0^1 \left( \sqrt{\log \frac{A_H}{M}} + \sqrt{\log\frac{1}{u}} \right) \ du \\
    &= M \left( \sqrt{\log \frac{A_H}{M}} + \frac{\sqrt{\pi}}{2} \right) \\
    &\le C_1 M.
\end{align*}
Therefore, with symmetrization, we get:
\[
\Ex \left[ \sup_{\hat{h} \in \hat{\cH}} \left| P_X \hat{h} - P_{n,X} \hat{h} \right| \right] \le CM \sqrt{\frac{V_F + V_T}{n}}.
\]
Define
\[
\Phi(X_1^n) = \sup_{\hat{h} \in \hat{\cH}} \left| P_X \hat{h} - P_{n,X} \hat{h} \right|.
\]
To apply McDiarmid's inequality, we replace $X_i$ with $X_i'$. Then, for any fixed $\hat{h}$, we have
\begin{align*}
    \left| \frac{1}{n} \sum_{j=1}^n \hat{h}(X_j) - \frac{1}{n} \sum_{j \ne i} \hat{h}(X_j) - \frac{1}{n} \hat{h}(X_i') \right| &= \frac{1}{n} \left| \hat{h}(X_i) - \hat{h}(X_i') \right| \\
    &\le \frac{2M}{n}.
\end{align*}
Hence, by McDiarmid's inequality, we get
\begin{align*}
    \bbP \left( \Phi(X_1^n) - \Ex [\Phi (X_1^n)] \ge t \right) &\le \exp \left( - \frac{2t^2}{\sum_{i=1}^n c_i^2} \right) \\
    &= \exp \left( - \frac{nt^2}{2M^2} \right).
\end{align*}
Since we want the event w.p. at least $1 - \frac{\delta}{4}$, thus, set $\exp \left( - \frac{nt^2}{2M^2} \right) = \frac{\delta}{4}$, and get $t = M \sqrt{\frac{2 \log (4/\delta)}{n}}$. Then, we have:
\[
\bbP \left( \Phi(X_1^n) - \Ex [\Phi(X_1^n)] \ge M \sqrt{\frac{2 \log (4/\delta)}{n}} \right) \le \frac{\delta}{4}.
\]
This implies that w.p. at least $1 - \frac{\delta}{4}$,
\begin{align*}
    \Phi(X_1^n) &\le \Ex [\Phi(X_1^n)] + M \sqrt{\frac{2 \log (4/\delta)}{n}} \\
    &\le CM \sqrt{\frac{V_F + V_T}{n}} + M \sqrt{\frac{2 \log (4/\delta)}{n}}.
\end{align*}
Therefore, we conclude that w.p. at least $1 - \frac{\delta}{4}$
\[
\sup_{f \in \cF_\alpha, t \in \cT} \left| G_n(f,t) - \hat{G}_n(f,t) \right| \le CM \left( \sqrt{\frac{V_F + V_T}{n}} + \sqrt{\frac{\log(4/\delta)}{n}} \right).
\]
Combining with the previous bound, we get:
\[
\sup_{f \in \cF, t \in \cT} \left| G(f,t) - \hat{G}_n(f,t) \right| \le M \epsilon_p + \epsilon\mu + CM \left( \sqrt{\frac{V_F + V_T}{n}} + \sqrt{\frac{\log(4/\delta)}{n}} \right).
\]
Now, choose
\[
\eta_n = M \epsilon_p + \epsilon\mu + CM \left( \sqrt{\frac{V_F + V_T}{n}} + \sqrt{\frac{\log(4/\delta)}{n}} \right),
\]
then on this good event, we have for every $f \in \cF_{\alpha, IS}$ and every $t \in \cT$,
\[
G(f,t) \ge 0 \implies \hat{G}_n(f,t) \ge G(f,t) - \left| G(f,t) - \hat{G}_n(f,t) \right| \ge - \eta_n.
\]
Therefore,
\[
\cF_{\alpha, IS} \subset \hat{\cF}_{\alpha, IS}^{(\eta_n)},
\]
in particular
\[
f_{\alpha, IS}^* \in \hat{\cF}_{\alpha, IS}^{(\eta_n)}.
\]


Now, we upper-bound the risk deviation.

Define the loss class 
\[
\cL_\alpha:= \{ \ell_f(x,d,y):= (f(x,d)-y)^2 \mid f \in \cF_\alpha \},
\]
then, we know that
\[
|\ell_f(x,d,y)| \le (|f(x,d)|+|y|)^2 \le (B_f + B)^2 = M^2.
\]

Let
\[
\Delta_n := \sup_{f \in \cF_\alpha} |\cR(f) - \cR_n(f)| = \sup_{\ell \in \cL_\alpha} |P\ell - P_n\ell|.
\]
Then, using the standard Rademacher generalization bound, we have
\[
\Delta_n \le 2 \Re_n(\cL_\alpha) + 3M^2 \sqrt{\frac{\log(4 / \delta)}{2n}},
\]
w.p. at least $1- \frac{\delta}{4}$. 

Then, we relate $\Re_n(\cL_\alpha)$ to $\Re_n(\cF_\alpha)$.

For $i = 1, \ldots, n$, define
\[
\phi_i(u) := (u - Y_i)^2, u \in [-B_f, B_f].
\]
Then, define
\[
\tilde{\phi}_i(u) := \phi_i(u) - \phi_i(0) = (u - Y_i)^2 - Y_i^2.
\]
Then, for $u, v \in [-B_f, B_f]$,
\begin{align*}
    |\tilde{\phi}(u) - \tilde{\phi}(v)| &= |\phi_i(u) - \phi_i(v)| \\
    &= |(u - Y_i)^2 - (v - Y_i)^2| \\
    &= |u - v| |u + v -2Y_i | \\
    &= |u - v| (|u| + |v| + 2|Y_i|) \\
    &\le 2(B_f + B)|u-v| \\
    &= 2M|u-v|.
\end{align*}
This implies that each $\tilde{\phi}_i$ is $2M$-Lipschitz and $\tilde{\phi}(0) = 0$. Then, by the Talagrand contraction lemma, we have
\begin{align*}
    \hat{\Re}_n(\cL_\alpha) &= \Ex_{\sigma} \left[ \sup_{f \in \cF_\alpha} \frac{1}{n} \sum_{i=1}^n \sigma_i \phi_i(f(X_i,D_i)) \right]\\
    &= \Ex_{\sigma} \left[ \sup_{f \in \cF_\alpha} \frac{1}{n} \sum_{i=1}^n \sigma_i \tilde{\phi}_i(f(X_i,D_i)) \right] \\
    &\le 2M \hat{\Re}_n(\cF_\alpha),
\end{align*}
where $\sigma_i$ denotes Rademacher random variable.
Therefore, taking expectation, we get
\[
\Re_n(\cL_\alpha) \le 2M \Re_n(\cF_\alpha).
\]
Thus, for any $\delta \in (0,1)$, we have:
\[
\Delta_n \le 4(B_f + B) \Re_n(\cF_\alpha) + 3M^2 \sqrt{\frac{\log(4/\delta)}{2n}},
\]
w.p. at least $1 - \frac{\delta}{4}$.

Then, on the event that $f_{\alpha,IS}^* \in \hat{\cF}_{\alpha, IS}^{(\eta_n)}$ and by definition of $\hat{f}_{\alpha, IS}$, we have:
\[
\cR_n(\hat{f}_{\alpha, IS}) \le \cR_n(f_{\alpha, IS}^*),
\]
Therefore, we have:
\begin{align*}
    &\cR(\hat{f}_{\alpha, IS}) - \cR(f_{\alpha, IS}^*) \\
    =& \left[ \cR(\hat{f}_{\alpha, IS}) - \cR_n(\hat{f}_{\alpha, IS}) \right] + \left[ \cR_n(\hat{f}_{\alpha, IS}) - \cR_n(f_{\alpha, IS}^*) \right] + \left[ \cR_n(f_{\alpha, IS}^*) - \cR(f_{\alpha, IS}^*) \right] \\
    \le& 2 \Delta_n.
\end{align*}
Therefore, w.p. at least $1 - \frac{\delta}{2}$:
\[
\cR(\hat{f}_{\alpha,IS}) - \cR(f_{\alpha,IS}^*) \le 8M \Re_n(\cF_\alpha) + 6M^2 \sqrt{\frac{\log(4/\delta)}{2n}}.
\]

\textbf{Step 2: upper bounding Fairness Bias}

For each $x \in \cX$, define
\begin{align*}
    m(x) &:= \min_k \mu_k(x), \\
    M(x) &:= \max_k \mu_k(x), \\
    R(x) &:= \frac{M(x)}{m(x)}.
\end{align*}
Then, since 
\[
\mu_{\min} \le m(x) \le M(x) \le \mu_{\max},
\]
therefore, we have:
\[
R(x) \le R_{\mu}.
\]
We then construct $g$ as follows:
\[
g_\alpha(x,d_k) := \min \left( \mu_k(x), \frac{m(x)}{\alpha} \right).
\]
We first verify that $g_\alpha(x, \cdot)$ is $\alpha$-fair for all $x \in \cX$.

Fix some $x \in \cX$, for every $k \in [|\cD|]$, $\mu_k(x) \ge m(x)$ and $\frac{m(x)}{\alpha} \ge m(x)$. This implies that
\[
g_\alpha(x,d_k) \ge m(x).
\]
On the other hand, for every $k \in [|\cD|]$,
\[
g_\alpha(x,d_k) \le \frac{m(x)}{\alpha}.
\]
Therefore, we have:
\[
\frac{\underset{k}{\max} \ g_\alpha(x,d_k)}{\underset{k}{\min} \ g_\alpha(x,d_k)} \le \frac{1}{\alpha}.
\]
This implies that $g_\alpha \in \cF_\alpha$.

Now, let
\[
\delta_k(x) := \mu_k(x) - g_\alpha(x,d_k) \ge 0.
\]
Then, if 
\begin{itemize}
    \item $\mu_k(x) \le \frac{m(x)}{\alpha}$, then $\delta_k(x) = 0$
    \item $\mu_k(x) > \frac{m(x)}{\alpha}$, then 
    \begin{align*}
        \delta_k(x) &= \mu_k(x) - \frac{m(x)}{\alpha} \\
        &\le M(x) - \frac{m(x)}{\alpha} \\
        &= m(x) \left( R(x) - \frac{1}{\alpha} \right).
    \end{align*}
\end{itemize}
Therefore, we have:
\[
\delta_k(x)^2 \le m(x)^2 \left( R(x) - \frac{1}{\alpha} \right)_+^2.
\]
Now, define
\[
\lambda(x) := \sum_{k=1}^{|\cD|} \delta_k(x) \cdot p_{k|x} \ge 0,
\]
and 
\[
\tilde{g}_\alpha(x,d_k) := g_\alpha(x,d_k) + \lambda(x).
\]
Then, we know $\tilde{g}_\alpha$ is $\alpha$-fair and satisfies IS, since adding the same constant will not violate the fairness ratio constraint. Moreover:
\begin{align*}
    &\Ex \left[ \tilde{g}_\alpha(X,D) - Y \mid X = x \right] \\
    =& \Ex \left[ g_\alpha(X,D) - Y \mid X= x \right] + \lambda(x) \\
    =& \sum_{k=1}^{|\cD|} \Ex \left[ g_\alpha(x,d_k) - Y \mid X=x,D=k \right] \cdot p_{k|x} + \lambda_(x) \\
    =& \sum_{k=1}^{|\cD|} \left( g_\alpha(x,d_k) - Y \mid X=x,D=k \right) \cdot p_{k|x} + \lambda_(x) \\
    =& \sum_{k=1}^{|\cD|} \left( g_\alpha(x,d_k) - \mu_k(x) \right) \cdot p_{k|x} + \lambda(x) \\
    =& \sum_{k=1}^{|\cD|}-\delta_k(x) \cdot p_{k|x} + \lambda(x) \\
    =& -\lambda(x) + \lambda(x) \\
    =& 0.
\end{align*}
Therefore, $\tilde{g}_\alpha \in \cF_{\alpha,IS}$.

Note that
\[
\tilde{g}_\alpha(x,d_k) - \mu_k(x) = \lambda(x) - \delta_k(x),
\]
then, let $\delta(X,D) := \delta_k(X)$, if $X = x, D = k$, then condition on $X$, $\delta$ is a discrete random variable taking values $\delta_1(x), \ldots, \delta_{|\cD|}(x)$ w.p. $p_{1|x}, \ldots, p_{|\cD||x}$.

Now, by definition of $\lambda(x)$, we have
\[
\lambda(x) := \sum_{k=1}^{|\cD|} \delta_k(x) \cdot p_{k|x} =\Ex[\delta \mid X = x],
\]
Again, knowing that:
\[
\tilde{g}_\alpha(x,d_k) - \mu_k(x) = \lambda(x) - \delta_k(x).
\]
then, therefore,
\[
\tilde{g}_\alpha(X,D) - \mu_D(X) = \lambda(x) - \delta(X,D),
\]
whenever $X = x$.

Thus, we have
\[
\Ex \left[ \left( \tilde{g}_\alpha - \mu_D \right)^2 \mid X = x \right] = \Ex \left[ (\lambda(x) - \delta)^2 \mid X = x \right],
\]
then, the conditional variance of $\delta$ given $X = x$ is:
\begin{align*}
    Var(\delta \mid X = x) &= \Ex \left[ \left( \delta - \Ex[\delta \mid X = x] \right)^2 \mid X = x \right] \\
    &= \Ex \left[ (\delta - \lambda(x))^2 \mid X = x \right] \\
    &= \Ex \left[ \left( \tilde{g}_\alpha - \mu_D \right)^2 \mid X = x \right] \\
    &= \Ex \left[ \delta^2 \mid X = x \right] - \left( \Ex[\delta \mid X = x] \right)^2 \\
    &\le \Ex[\delta^2 \mid X = x].
\end{align*}
But, we have shown that
\[
\delta_k(x)^2 \le m(x)^2 \left( R(x) - \frac{1}{\alpha} \right)_+^2.
\]
Therefore,
\begin{align*}
    Bias(\alpha) &\le \Ex \left[ \left( \tilde{g}_{\alpha} - \mu_D \right)_+^2 \right] \\
    &\le \mu_{\max}^2 \left( R_{\mu} - \frac{1}{\alpha} \right)_+^2
\end{align*}

\textbf{Step 3: combining results}

Therefore, we conclude that for any fixed $\alpha \in (0,1]$ and $\delta \in (0,1)$ w.p. at least $1 - \delta$, the following upper bound on the excessive risk holds:
\begin{align*}
    \cR(\hat{f}_{\alpha,IS}) - \cR(f^*) &\le 2 \Delta_n + Bias(\alpha) \\
    &\le 8M \Re_n(\cF_\alpha) + 6M^2 \sqrt{\frac{\log(4/\delta)}{2n}} + \mu_{\max}^2 \left( R_{\mu} - \frac{1}{\alpha} \right)_+^2
\end{align*}

\end{proof}

\newpage
\subsection{Corollary \ref{corollary:constraint-violation-rate} Proof}

\begin{proof}

Since $\hat{f}_{\alpha,IS} \in \cF_\alpha$, then by definition of $\cF_\alpha$, we know that $\hat{f}_{\alpha, IS}$ admits $0$ out-of-sample fairness violation rates.

For solvency, recall the deviation between the population solvency functional and its empirical counterpart from Theorem \ref{theorem:sample-complexity-upper-bound-fairness-and-IS} is:
\[
\sup_{f \in \cF_\alpha, IS} \left| G(f,t) - \hat{G}_n(f,t) \right| \le \eta_n,
\]
w.p. at least $1 - \delta$.

Now, on this h.p. event, applying the bound to $f = \hat{f}_{\alpha, IS}$, we have that for every $t \in \cT$,
\[
G(\hat{f}_{\alpha, IS},t) \ge \hat{G}_n(\hat{f}_{\alpha, IS},t) - \eta_n.
\]
Since the empirical estimator satisfies:
\[
\hat{G}_n(\hat{f}_{\alpha, IS},t) \ge 0, \forall t \in \cT,
\]
we get:
\[
G(\hat{f}_{\alpha, IS},t) \ge -\eta_n, \forall t \in \cT.
\]
Now, suppose that
\[
t(x) = \ones\{ IS_{\hat{f}_{\alpha, IS}}(x) < 0 \} \in \cT.
\]
Then, plug in the solvency functional, we can write:
\[
G(\hat{f}_{\alpha, IS},t) = \Ex \left[ IS_{\hat{f}_{\alpha, IS}}(X) \cdot \ones \{ IS_{\hat{f}_{\alpha, IS}}(X) < 0 \} \right].
\]
Then, on the event $\{IS_{\hat{f}_{\alpha, IS}}(X) < 0\}$, we have:
\[
IS_{\hat{f}_{\alpha, IS}}(X) = - \left( -IS_{\hat{f}_{\alpha, IS}}(X) \right)_+.
\]
This implies that:
\[
G(\hat{f}_{\alpha, IS},t) = - \Ex \left[ (-IS_{\hat{f}_{\alpha, IS}}(X))_+ \right].
\]
Then, we get:
\[
\Ex \left[ (-IS_{\hat{f}_{\alpha, IS}})_+ \right] \le \eta_n.
\]
Now, by Markov's inequality, for any $\tau > 0$, we get:
\begin{align*}
    \bbP \left( IS_{\hat{f}_{\alpha, IS}}(X) < -\tau \right) &= \bbP \left( (-IS_{\hat{f}_{\alpha, IS}}(X))_+ > \tau \right) \\
    &\le \frac{\Ex \left[ (-IS_{\hat{f}_{\alpha, IS}}(X))_+ \right]}{\tau}.
\end{align*}
Then, it follows that:
\[
\bbP \left( IS_{\hat{f}_{\alpha, IS}}(X) < -\tau \right) \le \frac{\eta_n}{\tau}.
\]
If the slack parameter $\eta_n$ is imposed, then
\[
G(\hat{f}_{\alpha, IS},t) \ge \hat{G}_n(\hat{f}_{\alpha, IS},t) - \eta_n \ge - 2 \eta_n.
\]
Then, by Markov's inequality again, we have:
\[
\Ex \left[ (-IS_{\hat{f}_{\alpha, IS}}(X))_+ \right] \le 2 \eta_n.
\]
Thus, we conclude
\[
\bbP \left( IS_{\hat{f}_{\alpha, IS}}(X) < -\tau \right) \le \frac{2 \eta_n}{\tau}
\]
    
\end{proof}

\newpage
\subsection{Proposition \ref{proposition:bounded-solution} Proof}

\begin{proof}
\label{proof:bounded solution}

Let
\[
m(x):= \min_k \mu_k(x), \ \text{and}\ M(x):= \max_k \mu_k(x).
\]
We will first show that
\[
\min_k \mu_k(x) \le f^*(x,\cdot),
\]
for $\bbP_X$-$a.e.$ $x$.

Similar to the proof of Proposition \ref{proposition:order-preserving-upper-bound}. Define
\[
S_L := \left\{ x \mid \exists k \ \text{s.t.}\ f^*(x,d_k) < m(x) \right\},
\]
and show that $\bbP_X(S_L) = 0$.

For each $k$, define 
\[
S_{L,k} := \left\{ x \mid f^*(x,d_k) < m(x) \right\}.
\]
Then, we know that
\[
S_L = \bigcup_{k=1}^{|\cD|} S_{L,k},
\]
then given $|\cD| < \infty$, if $\bbP_X(S_L) > 0$, there must be some $k_0$ s.t. $\bbP_X(S_{L,k_0}) > 0$.

Assume for contradiction that such $k_0$ exists.

Define 
\[
F(x) := \max_k f^*(x,d_k),
\]
then, under IS, we have:
\[
\sum_{k=1}^{|\cD|} f^*(x,d_k) \cdot p_{k|x} \ge \sum_{k=1}^{|\cD|} \mu_k(x) \cdot p_{k|x} = \mu(x) \ge m(x).
\]
Thus, we have $F(x) \ge m(x)$ for $\bbP_X$-$a.e.$ $x$.

Now, define
\[
\epsilon(x) := \ones_{S_{L,k_0}}(x) \cdot \min \left\{ m(x) - f^*(x,d_{k_0}), F(x) - f^*(x,d_{k_0}) \right\}.
\]
$\epsilon(x) > 0$, since we assumed that $f^*(x,d_{k_0}) < m(x)$, for $x \in S_{L,k_0}$.

We now construct $\tilde{f}$ as follows:
\[
\tilde{f}(x,d_k) := \begin{cases}
    f^*(x,d_{k_0}) + \epsilon(x),\ \text{if}\ k = k_0, \\
    f^*(x,d_k) \ \ \ \ \ \ \ \ \ \ \ \ \ , \ \text{if}\ k \ne k_0.
\end{cases}
\]
Then, we show that $\tilde{f} \in \cF$ and that $\tilde{f}$ is feasible.

Since $f^* \in \cF$ and $f^*$ is feasible. Then, enough to bound $\epsilon(x)$.
\begin{align*}
    \Ex \left[ \left( \tilde{f}(X,D) - f^*(X,D) \right)^2 \right] &= \Ex \left[ \epsilon(X)^2 \cdot \ones_{\{D = k_0\}} \right] \\
    &\le \Ex \left[ \left( \mu_{k_0}(X) - f^*(X,d_{k_0}) \right)^2 \cdot \ones_{\{D = k_0\}} \right] \\
    &\le 2 \Ex \left[ \mu_{k_0}(X)^2 \cdot \ones_{\{D = k_0\}} \right] + 2 \Ex \left[ f^*(X,d_{k_0})^2 \cdot \ones_{\{D = k_0\}} \right] \\
    &< \infty,
\end{align*}
since we have shown that $\Ex \left[ \mu_D(X)^2 \right] \le \Ex \left[ Y^2 \right] < \infty$ and $f^* \in \cF$.

Since $\epsilon(x) > 0$, IS is trivially preserved.

For $\alpha$-fairness:

If $x \notin S_{L,k_0}$, $\tilde{f}(x, \cdot) = f^*(x, \cdot)$.

If $x \in S_{L,k_0}$, then we increased the $k_0$ coordinate and capped it at $F(x)$. Thus, we know the maximum of $\tilde{f}(x, \cdot)$ is also $F(x)$. Since $f^*$ is already $\alpha$-fair, then we must have $f^*(x,d_k) \ge \alpha F(x)$. Thus, with $\epsilon(x) > 0$, we have for all $k$,
\[
\alpha F(x) \le \tilde{f}(x,d_k) \le F(x).
\]
Therefore, $\tilde{f} \in \cF$ and is feasible.

Since $f^*$ and $\tilde{f}$ differ only on $\{X \in S_{L,k_0}\, D = k_0\}$, thus, we have:
\[
\psi(\tilde{f}) - \psi(f^*) = \Ex \left[ \ones_{\{X \in S_{L,k_0}, D = k_0\}} \cdot \left( \left( \tilde{f}(X,d_{k_0} - \mu_{k_0}(X) \right)^2 - \left( f^*(X,d_{k_0}) - \mu_{k_0}(X) \right)^2 \right) \right],
\]
where $\psi(f):= \Ex \left[ (f(X,D) - Y)^2 \right]$ (same as defined in previous proofs).

For $x \in S_{L,k_0}$, define
\[
g(x) := \mu_{k_0}(x) - f^*(x,d_{k_0}) > 0.
\]
Also, we have 
\[
\tilde{f}(x,d_{k_0}) = f^*(x,d_{k_0}) + \epsilon(x) \le m(x) \le \mu_{k_0}(x).
\]
Thus, we get
\[
0 < \epsilon(x) \le g(x).
\]
Point-wise, we have:
\begin{align*}
    \left( \tilde{f}(x,d_{k_0} - \mu_{k_0}(x) \right)^2 - \left( f^*(x,d_{k_0} - \mu_{k_0}(x) \right)^2 &= (-g(x) + \epsilon(x))^2 - (-g(x))^2 \\
    &= \epsilon(x)^2 - 2\epsilon(x)g(x) \\
    &< 0.
\end{align*}
Therefore, we get:
\[
\bbP(\{X \in S_{L,k_0},D=k_0\}) = \int_{x \in S_{L,k_0}} p_{k_0 |x} d \bbP_X(x) > 0.
\]
But, this immediately implies that 
\[
\psi(\tilde{f}) < \psi(f^*),
\]
contradicting the optimality of $f^*$, as desired.

Now, we show that 
\[
f^*(x,\cdot) \le \max_k \mu_k(x),
\]
for $\bbP_X$-$a.e.$ $x$.

Define
\[
S_H := \left\{ x \mid \exists k \ \text{s.t.}\ f^*(x,d_k) > M(x) \right\},
\]
similarly, we show $\bbP_X(S_H) = 0$.

Assume for contradiction that $\bbP_X(S_H) > 0$.

Define
\[
g(x,d_k) := \min \left\{ f^*(x,d_k), M(x) \right\},
\]
then, we know $g \in \cF$.

We now argue that $g$ is $\alpha$-fair. Fix $x$.

If $F(x) \le M(x)$, then $g(x,\cdot) = f^*(x,\cdot)$, $\alpha$-fairness is preserved.

If $F(x) > M(x)$, then at least one coordinate is capped. Thus, we have:
\[
\max_k g(x,d_k) = M(x).
\]
Since $f^*$ is $\alpha$-fair, thus
\[
f^*(x,d_k) \ge \alpha F(x) \ge \alpha M(x).
\]
If $f^*(x,d_k) > M(x)$, then $g(x,d_k) = M(x) \ge \alpha M(x)$. 

If $f^*(x,d_k) \le M(x)$, then $g(x,d_k) = f^*(x,d_k) \ge \alpha F(x) > \alpha M(x)$. Hence, $\underset{k}{\min} g(x,d_k) \ge \alpha M(x) = \alpha \underset{k}{\max} g(x,d_k)$.

Therefore, $g$ is also $\alpha$-fair.

For IS, define
\[
s_g(x) := \sum_{k=1}^{|\cD|}\left( g(x,d_k) - \mu_k(x) \right) \cdot p_{k|x}.
\]
If $s_g(x) \ge 0$, then $g$ is feasible, and it is enough to get a contradiction by showing that $\psi(g) < \psi(f^*)$.

However, recall that in the construction of $g$, the cap brought down the premium for some coordinates. Thus, IS is not guaranteed. Thus, $s_g < 0$ may happen on some set.

Let
\[
T:= \{x \mid s_g < 0 \},
\]
then, on $T$, define the deficit
\[
\text{def}(x) := -s_g(x) = \sum_{k=1}^{|\cD|}\left(\mu_k(x) - g(x,d_k) \right) \cdot p_{k|x} > 0,
\]
and the corresponding positive correction as:
\[
\text{pos}(x) := \sum_{k=1}^{|\cD|}\left(\mu_k(x) - g(x,d_k) \right)_+ \cdot p_{k|x}.
\]
Note that if $\text{def}(x) > 0$, then at least one of the $\mu_k(x) - g(x,d_k), k = 1, \ldots, |\cD|$ is strictly positive. Thus, $\text{pos}(x) > 0$. Also, note that $\text{def}(x) \le \text{pos}(x)$ by the presence of $(\cdot)_+$ in $\text{pos}(\cdot)$. Therefore, we have the ratio:
\[
t(x):= 
\begin{cases}
    \frac{\text{def}(x)}{\text{pos}(x)}, x \in T, \\
    0 \ \ \ \ \ \ \ , x \notin T.
\end{cases}
\]
Now, define
\[
h(x,d_k) := g(x,d_k) + t(x)(\mu_k(x) - g(x,d_k))_+.
\]
We will verify that $h \in \cF$ and is feasible.

For each $(x,d_k)$: 

if $\mu_k(x) \le g(x,d_k)$ then $h = g$.

If $\mu_k(x) > g(x,d_k)$, then
\[
h(x,d_k) = (1 - t(x))g(x,d_k) + t(x) \mu_k(x)
\]
is a convex combination of $g(x,d_k)$ and $\mu_k(x)$. Hence, point-wise, we have:
\[
|h(x,d_k)| \le |g(x,d_k)| + |\mu_k(x)|.
\]
Then, we have:
\[
h(X,D)^2 \le 2 g(X,D)^2 + 2\mu_D(X)^2,
\]
since $(a+b)^2 \le 2a^2 + 2b^2$. But, we know that $g \in \cF$ and $\mu_D \in \cF$. Therefore, $h \in \cF$.

For IS, we have:
\begin{align*}
    \sum_{k=1}^{|\cD|} (h(x,d_k) - \mu_k(x)) \cdot p_{k|x} &= \sum_{k=1}^{|\cD|} (g(x,d_k) - \mu_k(x)) \cdot p_{k|x} + t(x) \sum_{k=1}^{|\cD|} (\mu_k(x) - g(x,d_k))_+ \cdot p_{k|x} \\
    &= s_g(x) + t(x) \text{pos}(x).
\end{align*}
If $x \notin T$, then $t(x) = 0$, $s_g(x) \ge 0$.

If $x \in T$, then
\[
s_g(x) + t(x) \text{pos}(x) = -\text{def}(x) + \frac{\text{def}(x)}{\text{pos}(x)} \cdot \text{pos}(x) = 0.
\]
For $\alpha$-fairness:

On $S_H^c$, we have $g = f^*$ and $T = \emptyset$. Thus $h = f^*$, so $h$ is $\alpha$-fair.

On $S_H$, we know there is some $k$ with $f^*(x,d_k) > M(x)$. Capping at $M(x)$, we get $\underset{k}{\max} g(x,d_k) = M(x)$. As argued previously, we have $\underset{k}{\min} g(x,d_k) \ge \alpha M(x)$. But $h$ is constructed by increasing some coordinate to values no larger than $\mu_k(x)$. Therefore, $h$ is also $\alpha$-fair.

Thus, $h \in \cF$ and is feasible.

Now, we show that $h$ strictly decreases the objective.

Fix $(x,d_k)$:

If $g(x,d_k) \ge \mu_k(x)$, then $h(x,d_k) = g(x,d_k)$.

If $g(x,d_k) < \mu_k(x)$, then 
\[
h(x,d_k) = (1 - t(x)) g(x,d_k) + t(x) \mu_k(x).
\]
Then, we have:
\[
\mu_k(x) - h(x,d_k) = (1 - t(x)) (\mu_k(x) - g(x,d_k)),
\]
which implies 
\begin{align*}
    (h(x,d_k) - \mu_k(x))^2 &= (1 - t(x))^2 (g(x,d_k) - \mu_k(x))^2 \\
    &\le (g(x,d_k) - \mu_k(x))^2.
\end{align*}
Thus, we have
\[
\psi(h) \le \psi(g).
\]
Then, enough to show that $g$ strictly decreases the objective.

Define the event
\[
A := \{f^*(X,D) > M(X)\}.
\]
Then, for any $x \in S_H$, there exists at least one $k$ with $f^*(x,d_k) > M(x)$. Then, we have:
\[
\sum_{k=1}^{|\cD|} \ones_{\{f^*(x,d_k) > M(x)\}} \cdot p_{k|x} > 0,
\]
since $p_{k|x} > 0$. Then
\begin{align*}
    \bbP(A) &= \int_{x \in S_H} \sum_{k=1}^{|\cD|} \ones_{\{f^*(x,d_k) > M(x)\}} \cdot p_{k|x} \cdot d\bbP_X(x) \\
    &> 0.
\end{align*}
On $A$, we have $g(X,D) = M(X) < f^*(X,D)$ and $\mu_D(X) \le M(X)$. Then, point-wise we have:
\begin{align*}
    (g(x,d_k) - \mu_k(x))^2 - (f^*(x,d_k) - \mu_k(x))^2 &= (M(x) - \mu_k(x))^2 - (f^*(x,d_k) - \mu_k(x))^2 \\
    &< 0.
\end{align*}
This immediately gives
\[
\psi(h) \le \psi(g) < \psi(f^*),
\]
contradicting the optimality of $f^*$, as desired.

This completes the proof.
\end{proof}

\newpage
\subsection{Proposition \ref{proposition:order-preserving} Proof}

\subsubsection{Proposition \ref{proposition:order-preserving} (i) Proof}

\begin{proof}
\label{proof:order-preserving-lower-bound}

We will prove by contradiction.

Fix a group index $k$ and define
\[
S:= \left\{ \mu_k(x) \ge \mu(x) \ \text{and} \ f^*(x,d_k) < \mu(x) \right\}.
\]
Then, enough to show that $\bbP_X(S) = 0$.

Assume for contradiction that $\bbP_X(S) > 0$.

Note that for $\bbP_X$-$a.e.$ $x$, by IS, we have:
\[
\sum_{k=1}^{|\cD|} \left( f^*(x,d_k) - \mu_k(x) \right) \cdot p_{k|x} \ge 0 \iff \sum_{k=1}^{|\cD|} f^*(x,d_k) \cdot p_{k|x} \ge \sum_{k=1}^{|\cD|} \mu_k(x) \cdot p_{k|x}.
\]
This implies 
\[
\sum_{k=1}^{|\cD|} f^*(x,d_k) \cdot p_{k|x} \ge \mu(x).
\]

Let $M(x):= \underset{k}{\max} f^*(x,d_k)$. Then, we know that 
\[
\sum_{k=1}^{|\cD|} f^*(x,d_k) \cdot p_{k|x} \le M(x),
\]
since the average is always upper-bounded by the maximum. Therefore, for any $x \in S$,
\[
M(x) \ge \mu(x) > f^*(x,d_k)
\]
with the strict inequalities:
\[
M(x) - f^*(x,d_k) > 0,
\]
and 
\[
\mu_k(x) - f^*(x,d_k) > 0.
\]
Define 
\begin{align*}
    \Delta_M(x) &:= M(x) - f^*(x,d_k) \\
    \Delta_\mu(x) &:= \mu_k(x) - f^*(x,d_k), 
\end{align*}
and the perturbation
\[
\epsilon(x) := \ones_S(x) \min \left\{ \Delta_M(x), \Delta_\mu(x) \right\}.
\]
Then, we construct an alternative solution $\tilde{f}$ as follows:
\[
\tilde{f}(x,d_j) := \begin{cases}
    f^*(x,d_k) + \epsilon(x),\ \text{if}\ j = k \\
    f^*(x,d_j) \ \ \ \ \ \ \ \ \ \ \ \ , \ \text{if}\ j \ne k.
\end{cases}
\]

Then, we verify that $\tilde{f} \in \cF$.

Note that $\tilde{f}$ and $f^*$ only differ at coordinate $k$. Therefore, enough to show the differed part is square-integrable. Concretely, the differed part is:
\[
\Ex \left[ \left( \tilde{f}(X,D) - f^*(X,D) \right)^2 \right] = \Ex \left[ \epsilon(X)^2 \ones_{\{D = k\}} \right].
\]
Note that 
\[
\epsilon(x) \le \ones_S(x) \Delta_\mu(x) = \ones_S(x) \left( \mu_k(x) - f^*(x,d_k) \right),
\]
then, we have:
\begin{align*}
    \Ex \left[ \epsilon(X)^2 \ones_{\{D = k\}} \right] &\le \Ex \left[ \left( \mu_k(X) - f^*(X,d_k) \right)^2 \cdot \ones_{\{D = k\}} \right] \\
    &\le 2 \Ex \left[ \mu_k(X)^2 \cdot \ones_{\{D = k\}} \right] + 2 \Ex \left[ f^*(X,d_k)^2 \cdot \ones_{\{D = k\}} \right] \\
    &< \infty,
\end{align*}
since in Proposition \ref{proof:existence-of-solution}, we have shown that $\Ex \left[ \mu_D(X)^2 \right] \le \Ex \left[ Y^2 \right] < \infty$, and $f \in \cF$. Therefore, we have $\tilde{f} \in \cF$.

Next, we verify that $\tilde{f}$ is feasible.

For IS, we have:
\begin{align*}
    \sum_{k=1}^{|\cD|} \tilde{f}(x,d_k) \cdot p_{k|x} &= \sum_{k=1}^{|\cD|} f^*(x,d_k) \cdot p_{k|x}  + \epsilon(x) \cdot p_{k|x} \\
    &\ge \mu(x) + \epsilon(x) \cdot p_{k|x} \\
    &\ge \mu(x).
\end{align*}

For $\alpha$0fairness.

If $x \notin S$, then $\tilde{f}(x, \cdot) = f^*(x, \cdot)$, the constraints are satisfied since $f^*$ is feasible.

Now, fix $x \in S$, then,
\begin{align*}
    \tilde{f}(x,d_k) &= f^*(x,d_k) + \epsilon(x) \\
    &\le f^*(x,d_k) + \Delta_M(x) \\
    &= M(x).
\end{align*}
This implies that the maximum of $\tilde{f}(x,\cdot)$ is also upper-bounded by $M(x)$.

Let $j^*$ be the index s.t. $M(x) = f^*(x,d_j)$, then, for every $j$, we have
\[
f^*(x,d_j) \ge \alpha f^*(x,d_{j^*}) = \alpha M(x).
\]
Therefore, for $j \ne k$, we have:
\[
\tilde{f}(x,d_j) = f^*(x,d_k) \ge \alpha M(x).
\]
Therefore, for all $j$, we have:
\[
\alpha M(x) \le \tilde{f}(x,d_j) \le M(x).
\]
Thus, for any pair of $j,l$,
\[
\tilde{f}(x,d_j) - \alpha \tilde{f}(x,d_l) \ge \alpha M(x) - \alpha M(x) = 0,
\]
and the same holds for the reverse way as one can easily verify.

Now, it remains to show that $\tilde{f}$ strictly decreases the objective.

Let 
\[
\psi(f) := \Ex \left[ (f(X,D) - Y)^2 \right].
\]
Note that by conditioning, we have:
\[
\Ex \left[ (f(X,D) - Y)^2 \mid X,D \right] = \left( f(X,D) - \mu_D(X) \right)^2 + Var(Y \mid X,D),
\]
but the variance term does not depend on $f$. Thus, enough to show that the squared bias term is decreased with $\tilde{f}$.

Since $\tilde{f}$ and $f^*$ differ only on $\{D = k\} \cap \{X \in S\}$, so, we get:
\[
\psi(\tilde{f}) - \psi(f^*) = \Ex \left[ \ones_{\{X \in S, D = k\}} \left( \left( \tilde{f}(X,d_k) - \mu_k(X) \right)^2 - \left( f^*(X,d_k) - \mu_k(X) \right)^2 \right) \right].
\]
Recall that on $\{X \in S\}$, we have $\tilde{f}(X,d_k) = f^*(X,d_k) + \epsilon(X)$.

Define
\[
g(X) := \mu_k(X) - f^*(X,d_k).
\]
Then, we have $g(X) > 0$ and $\epsilon(X) \le g(X)$ on S.

Point-wise, we have:
\begin{align*}
    \left( \tilde{f}(x,d_k) - \mu_k(x) \right)^2 - \left( f^*(x,d_k) - \mu_k(x) \right)^2 &= \left( f^*(x,d_k) +\epsilon(x) - \mu_k(x) \right)^2 - \left( f^*(x,d_k) - \mu_k(x) \right)^2 \\
    &= (-g(x) + \epsilon(x))^2 - (-g(x))^2 \\
    &= \epsilon(x)^2 - 2\epsilon(x)g(x).
\end{align*}
But, since $0 < \epsilon(x) \le g(x)$ for every $x \in S$. Thus, we have $\epsilon(x) - 2g(x) \le -g(x) < 0$. Hence, for $x \in S$, we have
\[
\epsilon(x)^2 - 2 \epsilon(x)g(x) < 0.
\]
Therefore, since we assumed $\bbP_X(S) > 0$. But, we also know that $p_{k|x} > 0$ for $\bbP_X$-$a.e.$ $x$. Thus,
\begin{align*}
    \bbP \left( \{X \in S,D = k\} \right) &= \int_{x \in S} p_{k|x} \cdot d \bbP_X(x) \\
    &> 0.
\end{align*}
Then, this implies
\begin{align*}
    \psi(\tilde{f}) - \psi(f^*) &= \int \ones_{\{x \in S, D = k\}} \left( \epsilon(x) - 2 \epsilon(x)g(x) \right) d \bbP_X(x) \\
    < 0,
\end{align*}
contradicting the optimality of $f^*$. Therefore, $\bbP_X(S) = 0$, as desired. 

This completes the proof.
\end{proof}

\newpage
\subsubsection{Proposition \ref{proposition:order-preserving} (ii) Proof}

\begin{proof}
\label{proof:order-preserving-upper-bound}

We first consider when $|\cD| = 2$.

Let one level in $\cD$ correspond to low risk (L) and the other to high risk (H).

Fix $x \in \cX$. If $\mu_L(x) \ge \alpha \mu_H(x)$ and $\mu_H(x) \ge \alpha \mu_L(x)$, then, we know $\left( f^*(x,d_L),f^*(x,d_H) \right) = \left( \mu_L(x), \mu_H(x) \right)$ is feasible (IS hits equality by definition of $\mu(x)$) and optimal with $\mu_L(x) \le \mu(x) \implies f^*(x,d_L) = \mu_L(x) \le \mu(x)$, as desired.

Now, consider the case when $\mu_L(x) < \alpha \mu_H(x)$. Then $\left( \mu_L(x), \mu_H(x) \right)$ is no longer feasible. Observe that for any fixed $f(x,d_H)$, the objective
\[
\psi(f) = \left( f(x,d_L) - \mu_L(x) \right)^2 \cdot p_{L|x} + \left( f(x,d_H) - \mu_H(x) \right)^2 \cdot p_{H|x},
\]
is only a function of $f(x, d_L)$ and is minimized at the smallest feasible $f(x,d_L)$. Since $\mu_L(x) < \alpha \mu_H(x)$, any feasible solution that keeps $f(x,d_H)$ close to $\mu_H(x)$ must have $f(x,d_L) > \mu_L(x)$. The objective is strictly increasing in $f(x,d_L)$ for $f(x,d_L) > \mu_L(x)$. Therefore, we only need to consider the scenario where $f(x,d_L) = \alpha f(x,d_H)$.

Now, let $f(x,d_L) = \alpha f(x,d_H)$, we have:
\[
\psi(f) = \left( \alpha f(x,d_H) - \mu_L(x) \right)^2 \cdot p_{L|x} + \left( f(x,d_) - \mu_H(x) \right)^2 \cdot p_{H|x}.
\]
Fairness constraint is automatically satisfied by the setup of $f(x,d_L) = \alpha f(x,d_H)$. For IS, we have:
\[
\left( p_{H|x} + \alpha p_{L|X} \right) f(x,d_H) \ge f(x,d_L) \cdot p_{L|x} + f(x,d_H) \cdot p_{H|X} = \mu(x).
\]
This implies
\[
f(x,d_H) \ge \frac{\mu(x)}{p_{H|x} + \alpha p_{L|x}}.
\]
Now, if we minimize $\psi$ with $f(x,d_L) = \alpha f(x,d_H)$, we would obtain
\[
\frac{\partial}{\partial f(x,d_H)} \psi(f) = 2 \alpha \left( \alpha f(x,d_H) - \mu_L(x) \right) \cdot p_{L|x} + 2\left( f(x,d_H) - \mu_H(x) \right) \cdot p_{H|x} =0.
\]
This gives 
\[
\tilde{f}(x,d_H) = \frac{\alpha \mu_L(x) \cdot p_{L|x} + \mu_H(x) \cdot p_{H|x}}{\alpha^2 p_{L|x} + p_{H|x}},
\]
consequently,
\[
\tilde{f}(x,d_L) = \alpha \tilde{f}(x,d_H) = \frac{\alpha^2 \mu_L(x) \cdot p_{L|x} + \alpha \mu_H(x) \cdot p_{H|x}}{\alpha^2 p_{L|x} + p_{H|x}}.
\]
Therefore, we have that
\[
f^*(x,d_H) = \max \left\{ \tilde{f}(x,d_H), \frac{\mu(x)}{p_{H|x} + \alpha p_{L|x}} \right\},
\]
and
\[
f^*(x,d_L) = \alpha f^*(x,d_H).
\]
If the maximum is attained in order to satisfy IS, then 
\begin{align*}
    f^*(x,d_L) &= \alpha f^*(x,d_H) \\
    &= \frac{\alpha \mu(x)}{p_{H|x} + \alpha p_{L|x}}.
\end{align*}
Then, enough to verify that if 
\begin{align*}
    \frac{\alpha \mu(x)}{p_{H|x} + \alpha p_{L|x}} \le \mu(x) \iff \alpha &\le p_{H|x} + \alpha p_{L|x} \\
    &=1 - p_{L|x} + \alpha p_{L|x} \\
    &= \alpha + (1 - \alpha)(1 - p_{L|x}),
\end{align*}
as desired.

Now, suppose IS is slack, that is $f^*(x,d_H) = \tilde{f}(x,d_H)$, then, we have:
\begin{align*}
    f^*(x,d_L) &= \alpha f^*(x,d_H) \\
    &= \frac{\alpha^2 \mu_L(x) \cdot p_{L|x} + \alpha \mu_H(x) \cdot p_{H|x}}{\alpha^2 p_{L|x} + p_{H|x}}.
\end{align*}
Then, enough to verify if 
\begin{align*}
    \frac{\alpha^2 \mu_L(x) \cdot p_{L|x} + \alpha \mu_H(x) \cdot p_{H|x}}{\alpha^2 p_{L|x} + p_{H|x}} \le \mu(x) \iff \frac{\alpha^2 \mu_L(x) \cdot p_{L|x} + \alpha \mu_H(x) \cdot p_{H|x}}{\alpha^2 p_{L|x} + p_{H|x}} \le\mu_L(x) \cdot p_{L|x} + \mu_H(x) \cdot p_{H|x}.
\end{align*}
Then, simplify to get:
\[
\alpha \tilde{f}(x,d_H) - \mu(x) = \frac{(1 - \alpha)(1 - p_{L|x})\left( (1 + \alpha)(\mu_H(x) - \mu_L(x)) \cdot p_{L|x} - \mu_H(x) \right)}{\alpha^2 p_{L|x} + p_{H|x}},
\]
since $(1 + \alpha)(1 - p_{L|x}) \ge 0$ and $\alpha^2 p_{L|x}+p_{H|x} > 0$, we get:
\[
\alpha \tilde{f}(x,d_H) \le \mu \iff p_{L|x} \le \frac{\mu_H(x)}{(1 + \alpha)\left( \mu_H(x) - \mu_L(x) \right)}.
\]
Take $\mu_L(x) \to 0$, we get the worst-case condition
\[
p_{L|x} \le \frac{1}{1 + \alpha},
\]
to have 
\[
f^*(x,d_L) \le \mu(x),
\]
as desired.

To extend the result to $|\cD| > 2$, we need to impose an additional restriction on the optimal solution.

Suppose there exists a low-risk cluster $C_L \subset \cD$ and a high-risk cluster $C_H = \cD \setminus C_L$ s.t. for some $M(x) > 0$
\[
\begin{cases}
    f^*(x,d_k) = \alpha M(x), \ \forall k \in C_L \\
    f^*(x,d_k) = M(x) \ \ \ , \ \forall k \in C_H.
\end{cases}
\]
Then, let 
\begin{align*}
    p_{L|x}:= \sum_{k \in C_L} p_{k|x}, \quad \mu_L(x):= \frac{1}{p_{L|x}} \sum_{k \in C_L} \mu_k(x) \cdot p_{k|x}, \\
    p_{H|x}:= \sum_{k \in C_H} p_{k|x}, \quad \mu_H(x):= \frac{1}{p_{H|x}} \sum_{k \in C_H} \mu_k(x) \cdot p_{k|x}.
\end{align*}
Then, with the same argument used for $|\cD| = 2$, we have the worst-case sufficient condition
\[
\sum_{k \in C_L} p_{k|x} \le \frac{1}{1 + \alpha},
\]
then, for every $k \in C_L$,
\[
f^*(x,d_k) \le \mu(x), \forall \alpha \in (0,1].
\]
This completes the proof.
\end{proof}

\newpage
\subsubsection{Proposition \ref{proposition:order-preserving} (iii) Proof}

\begin{proof}
\label{proof:order-preserving-strict}

Fix $x$, then we have:
\begin{align*}
    \Ex \left[ (f(X,D) - Y)^2 \mid X = x \right] &= \sum_{k=1}^{|\cD|} \Ex \left[ (f(x,d_k) - Y)^2 \mid X = x, D = k \right] \cdot p_{k|x} \\
    &= \sum_{k=1}^{|\cD|} \Ex \left[ \left( (f(x,d_k) - \mu_k(x)) + (\mu_k(x) - Y) \right)^2 \mid X = x, D = k \right] \cdot p_{k|x}.
\end{align*}
Expand the square, we have that for each fixed $k$,
\[
\Ex \left[(f(x,d_k) - Y)^2 \mid X = x, D = k \right] = (f(x,d_k) - \mu_k(x))^2 + Var(Y \mid X =x, D = k).
\]
Then, since the variance term does not depend on $f$, we can drop it when optimizing over $f$.

Let 
\[
M(x) := \max_k f(x,d_k),
\]
then, feasibility implies that 
\[
\alpha M(x) \le f(x,\cdot) \le M(x).
\]
On the other hand, if there is some $M(x)$ with
\[
\alpha M(x) \le f(x,\cdot) \le M(x),
\]
then
\[
\alpha \le \frac{f(x,d_k)}{f(x,d_j)} \le \frac{1}{\alpha},
\]
for all pairs $(k,j)$. In other words, for fixed $x$, $\alpha$-fairness constraints are precisely the existence of a scalar $M(x)$ satisfying the above inequalities.

Thus, for fixed $x$, the optimization becomes:
\[
    \min_{f \in \cF, M(x) \in \bbR} \sum_{k=1}^{|\cD|} (f(x,d_k) - \mu_k(x))^2 \cdot p_{k|x},
\]
subject to
\begin{align*}
    \sum_{k=1}^{|\cD|} (f(x,d_k) - \mu_k(x)) \cdot p_{k|x} &\ge 0 \\
    \alpha(M(x) \le f(x,d_k) &\le M(x), \forall k.
\end{align*}
Note that this is convex with a strictly convex objective. Thus, we can leverage the KKT condition.

Now, define the Lagrange multipliers:
\begin{itemize}
    \item $\lambda(x) \ge 0$, for $\mu(x) - \sum_{k=1}^{|\cD|} f(x,d_k) \cdot p_{k|x} \le 0$.
    \item $u_k(x) \ge 0$, for $f(x,d_k) - M(x) \le 0$.
    \item $v_k(x) \ge 0$, for $\alpha M(x) - f(x,d_k) \le 0$.
\end{itemize}
Then, by stationarity w.r.t. $f$, we get:
\[
2 (f(x,d_k) - \mu_k(x)) \cdot p_{k|x} - \lambda(x)\cdot p_{k|x} + u_k(x) - v_k(x) = 0.
\]

Now, consider the following cases:
\begin{enumerate}
    \item If $\alpha M(x) < f(x,d_k) < M(x)$, then $u_k(x) = v_k(x) = 0$, we get:
    \[
    f(x,d_k) = \mu_k(x) + \frac{\lambda(x)}{2}.
    \]
    \item If $f(x,d_k) = M(x)$, then for $\alpha < 1$, the lower bound is slack, meaning $v_k(x) = 0$, we get:
    \[
    u_k(x) = \left( \lambda(x) - 2(M(x) - \mu_k(x)) \right) \cdot p_{k|x} \ge 0 \implies \mu_k(x) + \frac{\lambda(x)}{2} \ge M(x).
    \]
    \item If $f(x,d_k) = \alpha M(x)$, then the upper bound is slack, meaning $u_k(x) = 0$, we get:
    \[
    v_k(x) = \left( 2( \alpha M(x) - \mu_k(x)) -\lambda(x) \right) \ge 0 \implies \mu_k(x) + \frac{\lambda(x)}{2} \le \alpha M(x).
    \]
\end{enumerate}
Then, for every $k$, we have:
\[
f(x,d_k) = \Pi_{[\alpha M(x),M(x)]} \left( \mu_k(x) + \frac{\lambda(x)}{2} \right).
\]
Note that if $\mu_k(x) \ge \mu_j(x)$, then
\[
\mu_k(x) + \frac{\lambda(x)}{2} \ge \mu_j(x) + \frac{\lambda(x)}{2} \implies f(x,d_k) = \Pi \left( \mu_k(x) + \frac{\lambda(x)}{2} \right) \ge \Pi \left( \mu_j(x) + \frac{\lambda(x)}{2} \right) = f(x,d_j).
\]
Hence, we have
\[
\mu_k(x) \ge \mu_j(x) \implies f^*(x,d_k) \ge f^*(x,d_j),
\]
for $\bbP_X$-$a.e.$ $x$.
\end{proof}


\newpage

\end{document}